\documentclass[letterpaper, english]{article}

\usepackage[T1]{fontenc}
\usepackage{babel,csquotes}
\usepackage{microtype}
\usepackage{hyperref}
\hypersetup{
	unicode = true,
	colorlinks=true,
	linkcolor=blue,
	citecolor=blue,
	urlcolor=blue,
    pdftitle = {Characterization of the Basin of Convexity for Multi-Snapshot Spike Deconvolution via Variable Projection},
	pdfauthor = {Meghna Kalra, Maxime Ferreira Da Costa, and Kiryung Lee}
	}
\usepackage[margin=1in]{geometry}
\usepackage{lmodern}
\usepackage{enumitem}

\usepackage{biblatex}
\usepackage{graphicx}%
\usepackage{multirow}%
\usepackage{amsmath,amssymb,amsfonts}%
\usepackage{amsthm}%
\usepackage{mathrsfs}%
\usepackage{xcolor}%
\usepackage{booktabs}%
\usepackage{algorithm}%
\usepackage{algorithmicx}%
\usepackage{algpseudocode}%
\usepackage[svgnames]{xcolor}
\usepackage{subcaption}
\usepackage{bm,bbm}   
\usepackage{amsmath}
\usepackage{mathtools}

\newtheorem{theorem}{Theorem}%
\usepackage{algorithm}%
\usepackage{algorithmicx}%
\usepackage{algpseudocode}%
\newtheorem{remark}{Remark}%

\newcommand\norm[1]{\lVert#1\rVert}
\newtheorem*{theorem*}{Theorem}

\newtheorem{lemma}[theorem]{Lemma}

\usepackage{comment}

\let\hat\widehat

\title{Characterization of the Basin of Convexity for\\ Multi-Snapshot Spike Deconvolution via Variable Projection}

\author{
Meghna Kalra$^{1}$,
Maxime Ferreira Da Costa$^{2}$, and
Kiryung Lee$^{1}$
\\[1em]
$^{1}$Department of Electrical and Computer Engineering\\
The Ohio State University\\
2015 Neil Ave., Columbus, OH 43210, USA
\\[1em]
$^{2}$Laboratory of Signals and Systems\\
CentraleSupélec, Université Paris--Saclay\\
3 rue Joliot Curie, Gif-sur-Yvette, 91190, France\thanks{This work was supported in part by the National Science Foundation under 
Award CCF-1943201 (M.K.\ and K.L.) and by the Agence Nationale de la 
Recherche under Award ANR-24-CE48-3094-01
 (M.F.D.C.).
}
}

\date{\today}

\begin{document}

\maketitle

\abstract{The problem of multi-snapshot spike deconvolution is studied, where the goal is to recover the locations of sparse impulses from their noisy convolution with a known point spread function (PSF) across multiple snapshots. A variable-projection formulation is adopted, in which the amplitudes are eliminated in closed form, thereby reducing the task to a nonconvex least-squares problem over the spike locations alone. This formulation is referred to as the variable-projection formulation of spike deconvolution (VarProSD). An explicit characterization of the basin of convexity of the VarProSD objective is provided in terms of key PSF properties, including its power spectral density and smoothness, revealing how sampling bandwidth and spike separation affect the local geometry. Within this basin, consistency of the estimator in the number of snapshots is established under stochastic noise, and a complementary, sharper error bound is derived under adversarial noise through the local Lipschitz property of the inverse map. Local convergence guarantees for gradient descent are further established when initialization is performed within the basin. A central role throughout the analysis is played by Beurling--Selberg extremal approximations, which enable sharp, PSF-agnostic bounds on the conditioning of the structured matrices arising in the optimization landscape. Numerical experiments are presented to corroborate the theoretical findings and demonstrate the effectiveness of modified ESPRIT initialization followed by gradient-based refinement.}



\maketitle

\newpage

\tableofcontents

\clearpage

\section{Introduction}\label{sec1 intro}

\subsection{Problem Formulation}

Consider the problem of estimating pulse parameters from multiple
snapshots. Each snapshot is composed of a superposition of $K$ pulses,
each shaped by a known point spread function (PSF)
$g: \mathbb{R} \to \mathbb{R}$ centered at common locations
$\{\tau_k\}_{k=1}^K$, with varying amplitudes across snapshots.
The observed signal in the $\ell$th snapshot is given by
\begin{equation*}
y_\ell(t) = \sum_{k=1}^K x_{k,\ell}\, g(t-\tau_k) + \mathrm{noise},
\quad \ell \in [L] := \{1,\dots,L\}.
\end{equation*}
The objective is to estimate the unknown locations
$\{\tau_k\}_{k=1}^K$ from the Fourier transform coefficients of
$y_\ell$, sampled at frequencies in the set
$\Omega := \bigl\{f_i = \frac{2i-N-1}{2T} : i \in [N]\bigr\}$
across $L$ snapshots.
Let $\bm{Y} \in \mathbb{C}^{N \times L}$ denote the matrix collecting
the measurements, where each entry $[\bm{Y}]_{i,\ell}$ corresponds to
the Fourier transform of $y_{\ell}$ evaluated at frequency $f_i$,
\emph{i.e.}
$\widehat{y_\ell}(f) = \int_{-\infty}^\infty y_\ell(t)
e^{-\mathsf{j} 2\pi f t}\,dt$, corrupted with additive noise.
Define the mapping $\bm{\gamma} \mapsto \bm{\Phi}_{\bm{\gamma}}$
from $\bm{\gamma} = [\gamma_1,\dots,\gamma_K]^\mathsf{T} \in
\mathbb{R}^K$ to a Vandermonde matrix
$\bm{\Phi}_{\bm{\gamma}} \in \mathbb{C}^{N \times K}$ given by
\begin{equation}
\label{eq:Phi_tau_def}
[\bm{\Phi}_{\bm{\gamma}}]_{i,k} = e^{-\mathsf{j} 2\pi \gamma_k f_i},
\end{equation}
and let $\bm{\tau} = [\tau_1,\dots,\tau_K]^\mathsf{T}$ denote the
ground-truth parameter vector.
Then the matrix $\bm{Y}$ can be compactly expressed as
\begin{align*}
\bm{Y} = \bm{G}\bm{\Phi}_{\bm{\tau}}\bm{X} + \bm{Z},
\end{align*}
where $\bm{Z} \in \mathbb{C}^{N \times L}$ denotes the noise matrix,
and $\bm{X} \in \mathbb{C}^{K \times L}$ is defined elementwise by
$[\bm{X}]_{k,\ell} = x_{k,\ell}$.
The diagonal matrix $\bm{G} \in \mathbb{C}^{N \times N}$ is given by
\begin{equation}
\label{eq:G_def}
\bm{G} = \mathrm{diag}\!\left(\widehat{g}(f_1),\dots,
\widehat{g}(f_N)\right) \in \mathbb{C}^{N \times N},
\end{equation}
where $\widehat{g}$ denotes the Fourier transform of $g$.
Given the measurement matrix $\bm{Y}$ and the known point spread function (PSF) $\bm{G}$,
the goal is to recover the spike locations
$\{\tau_k\}_{k=1}^K$ without knowledge of the amplitude matrix $\bm{X}$.

This problem formulation aligns with the well-known \emph{spike
deconvolution} problem in the literature.
Such a model arises in many real-world signal processing applications
like array processing, medical imaging, radar, and wireless
communications. For instance, in radar sensing, the received signal is
modeled as the convolution of a known, well-calibrated transmitted
waveform with a sparse set of target echoes, corrupted by noise.
Recovering the precise delays of these echoes reduces to a spike
deconvolution problem, which is central to high-resolution target
detection and localization~\cite{bajwa2011identification}.
Similarly, in seismic imaging, the received signal is modeled as the
convolution of a known seismic wavelet with a sparse reflectivity
profile (series of spikes), corrupted by noise~\cite{velis2008stochastic}.
The goal is to recover the precise locations and amplitudes of
subsurface reflectors from convolved and noisy measurements.

We formulate the estimation task as an optimization problem that seeks
the spike locations best explaining the observed data. Concretely,
this amounts to minimizing the overall discrepancy across the Fourier
measurements:
\begin{equation}\label{eq:least_squares1}
\mathop{\mathrm{minimize}}_{\bm{\gamma}\in \mathbb{R}^K,\,
\bm{\Upsilon} \in \mathbb{C}^{K \times L}}
~\frac{1}{2L}
\norm{\bm{Y}-\bm{G}\bm{\Phi}_{\bm{\gamma}}\bm{\Upsilon}}_{\mathrm{F}}^2.
\end{equation}
Substituting $\bm \Upsilon$ in \eqref{eq:least_squares1} with its optimal solution conditioned on $\bm \gamma$, we arrive at an equivalent optimization formulation, which we refer to as the Variable Projection formulation of Spike Deconvolution (\emph{VarProSD}):
\begin{align*}
 \mathop{\rm minimize}_{\bm{\gamma} \in \mathbb{R}^K}\, \, 
\ell(\bm{\gamma})
\end{align*}
where $\ell(\bm{\gamma})$ is defined as
\begin{equation}
\label{eq:loss_function}
\ell(\bm{\gamma}) := 
~\frac{1}{2L}\norm{\bm{P}_{\bm{\gamma}}^\perp \bm{Y}}_\mathrm{F}^2,
\end{equation}
and
\begin{equation}
\label{eq:def:orth_proj}
\bm{P}_{\bm{\gamma}}^\perp :=
\bm{I}_N - \bm{G}\bm{\Phi}_{\bm{\gamma}}
(\bm{G}\bm{\Phi}_{\bm{\gamma}})^\dagger 
\end{equation} denotes the projector onto the orthogonal complement of the column
space of $\bm{G}\bm{\Phi}_{\bm{\gamma}}$.
This reformulation strategy, known as \emph{variable projection} (VP)
in the general literature, was introduced by
\cite{golub1973differentiation}.
Golub and Pereyra further derived lemmas for the derivatives of the
pseudoinverse and the projection operator, yielding a closed-form
expression for the gradient of the cost function
in~\eqref{eq:loss_function}.
We prefer the VarProSD formulation~\eqref{eq:loss_function} over jointly optimizing~\eqref{eq:least_squares1} for two reasons.
First, while the number of spike locations is fixed at $K$, the
number of amplitude variables in $\bm{\Upsilon}$ scales with the
number of snapshots $L$, leading to a high-dimensional optimization
problem in the large-$L$ regime. 
Second, and more fundamentally, the spike locations and the amplitudes
behave asymmetrically: the locations $\bm{\tau}$ form a parameter of
fixed dimension $K$, independent of $L$, and can be consistently
estimated as $L \to \infty$ even under noisy conditions. In contrast,
the amplitude estimates do not converge to the ground-truth amplitudes,
even with oracle knowledge of $\bm{\tau}$. This asymmetry highlights
the advantage of reducing the optimization to the spike locations,
which are the parameters of primary interest in spike deconvolution.

\subsection{Related Work}
In the classical noise-free setting, equalization—implemented as pointwise division by the Fourier transform of the PSF—can simplify spike deconvolution to a harmonic retrieval problem. This reduced problem can be solved exactly using well-established algorithms such as MUSIC \cite{schmidt1986multiple}, ESPRIT (Estimation of Signal Parameters via Rotational Invariance Techniques) \cite{roy1989esprit}, and Prony’s method \cite{prony1795essai}. However, in the presence of noise, equalization can significantly amplify estimation errors. Although the original ESPRIT algorithm has been shown to tolerate the effects of the PSF to some extent—even without explicit knowledge of its shape \cite{bresler1989resolution,kalra2024stable}—the reconstruction still suffers from persistent distortion due to the PSF, even when a large number of snapshots are available.
In addition, a variant of ESPRIT has been proposed to estimate spike locations by leveraging the known PSF without relying on equalization \cite{swindlehurst1999methods}. While this approach was also extended to handle unknown PSFs, its theoretical guarantees are limited to the noise-free setting. A key limitation of ESPRIT and its variants is their inability to fully exploit the structural assumptions inherent in the problem formulation. In the special case of a trivial PSF (Dirac), TLS-ESPRIT has been shown to be equivalent to a nonconvex optimization formulation that factorizes the split data into subarrays using rotation invariance \cite{ottersten2002performance}. However, this analysis reveals that TLS-ESPRIT overlooks the Vandermonde structure embedded in the measurement model.

In recent years, convex optimization methods—such as the atomic norm minimization framework—have been developed to fully exploit the problem structure without discretizing the parameter space \cite{decastro2012exact,candes2013super,candes2014towards,chi2020harnessing}. This framework has been extended to a broad class of measurement models, including time-domain sampling and scenarios with missing data \cite{schiebinger2018superresolution,tang2013compressed}.

Non-convex optimization approaches, including simple first-order methods, have shown strong performance in signal estimation and machine learning tasks~\cite{traonmilin2020basins, traonmilin2024strong}. In the context of spike deconvolution, preconditioned gradient methods have been studied to exploit the specific geometry of the optimization landscape~\cite{ferreira2023local,gabet2024preconditioned}. These methods offer improved convergence rates toward the ground truth in the absence of noise, along with provable statistical estimation error bounds. However, existing theoretical analyses assume a fixed and specific PSF and are limited to the single-snapshot observation setting. Yet, those previous analysis studies the contraction of the gradient descent iterates in $\ell_\infty$-norm, which lacks scaling when the model order is high, and utilizes loose numerical approximations of the operator norms.

Another line of work focuses on the super-resolution regime (where the minimum distance between the spikes falls below the Rayleigh limit) for reconstruction of Diracs. Their reconstruction typically require extra assumptions, such as the spikes forming specific patterns like clusters and pairs \cite{li2020super,li2022stability,batenkov2021super,kunis2021condition}. 
While this line of research has addressed challenging and important scenarios in practical applications, the present work focuses on a different difficulty—signal recovery in the presence of blurring caused by a non-trivial PSF.

\subsection{Contributions}

We study spike deconvolution under a non-trivial PSF from multiple
measurement snapshots. Since the resulting estimation problem is
non-convex, a rigorous characterization of when and how
optimization-based recovery succeeds is essential.
As illustrated in Figure~\ref{fig:motivation_snr1}, modified ESPRIT remains stable under noise but exhibits a persistent 
gap to the $\sqrt{\mathrm{CRB}}$ (square root of the Cram\'{e}r--Rao 
bound, a fundamental lower bound on estimation variance) ~\cite{ferreira2023conditionNumber,scharf1993geometry}. Initializing gradient-based iterative
refinement from modified ESPRIT substantially reduces this gap,
approaching the $\sqrt{\mathrm{CRB}}$ at moderate-to-high SNR. 
This motivates the following contributions.

\begin{figure}[t]
  \centering
\includegraphics[width=0.55\textwidth]{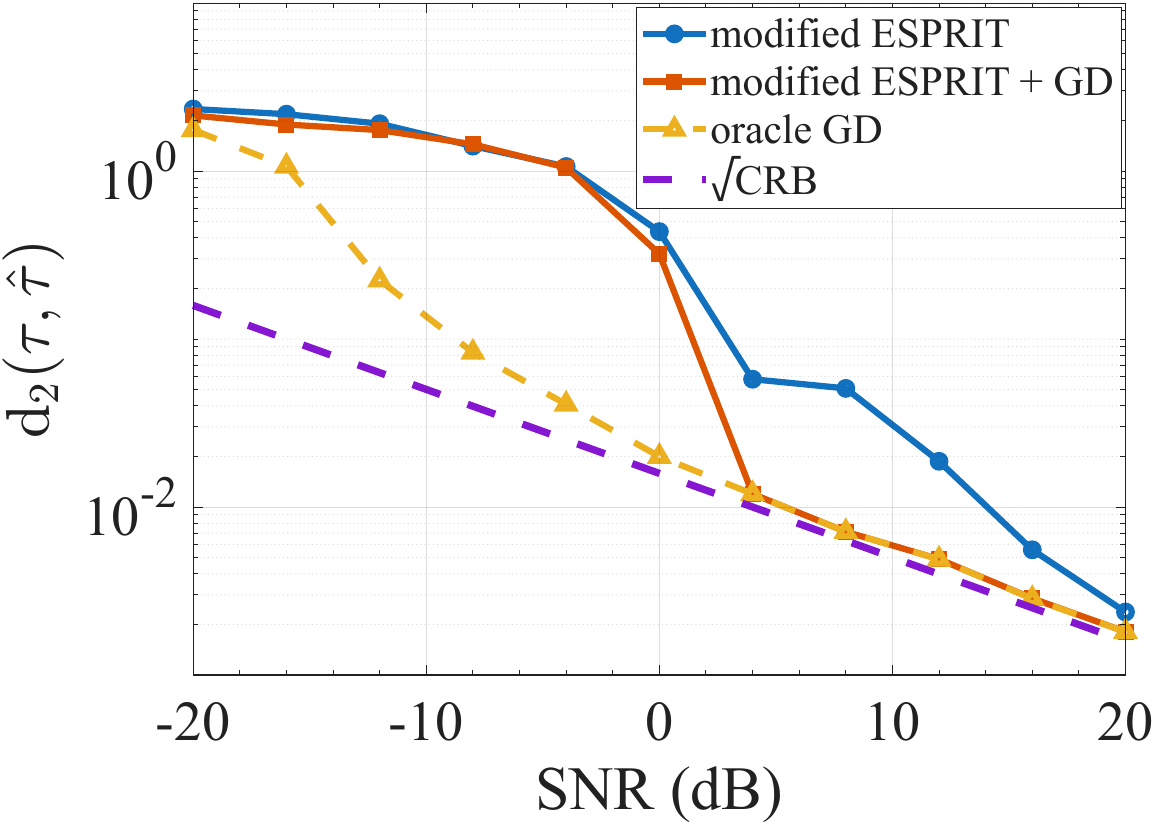}
  \caption{Estimation error vs.\ SNR under a Gaussian PSF for
    different initialization strategies ($\sigma = 0.1$, $B = 11$,
    $N = 51$, $L = 20$). Modified ESPRIT exhibits a persistent gap
    to the $\sqrt{\mathrm{CRB}}$, while initializing gradient-based
    optimization from modified ESPRIT significantly reduces this gap,
    approaching the $\sqrt{\mathrm{CRB}}$ for moderate to high SNR
    (approximately $\geq 5$~dB).}
  \label{fig:motivation_snr1}
\end{figure}

\begin{enumerate}[wide]

\item \textbf{Characterization of the Size of the Basin of
Convexity:}
We provide the first explicit characterization of the basin of
convexity for the VarProSD objective in the spike deconvolution
setting.
Specifically, we characterize a neighborhood
$\mathcal{N}(\bm{\tau},\varrho)$ of radius $\varrho$ around the
ground truth $\bm{\tau}$ within which the objective is strongly convex and admits a unique local minimizer $\bm{\gamma}_\star$.
The radius $\varrho$ is expressed explicitly in terms of interpretable
problem parameters: the spectral properties of the PSF, the minimum
separation between spikes, the dynamic range of the amplitudes, and
the sampling bandwidth $B = N/T$. No prior result quantifies the basin size for spike deconvolution with an arbitrary PSF under multiple snapshots. The radius $\varrho$ is given 
explicitly in~\eqref{eq:basin_radius_cond} in Section~\ref{subsec:basin}.

To state this precisely and build intuition, we introduce the following
definitions, and then visualize the geometry of the cost
function~\eqref{eq:loss_function} under different separation regimes
in Figure~\ref{fig:basin_vs_min_sep11}.
Since sampling in the Fourier domain over the uniform grid $\Omega$,
spaced by intervals of $1/T$, introduces modulo-$T$ ambiguities in
the time domain, we recover the locations $\{\tau_k\}_{k=1}^K$ on
the length-$T$ torus $\mathbb{T} \simeq \mathbb{R}/T\mathbb{Z}$,
with the wrap-around distance
\begin{align}
\label{eq:torus_def}
d_{\mathbb{T}}(\tau,\gamma) = \min_{p \in \mathbb{Z}}
|\tau - \gamma + pT|.
\end{align}
The minimum separation between the ground-truth locations is defined
as
\begin{equation}
\label{eq:min_sep}
\Delta := \begin{cases}
\displaystyle\min_{j \neq k} d_{\mathbb{T}}(\tau_j,\tau_k)
& \text{if } K \geq 2,\\
T & \text{if } K = 1.
\end{cases}
\end{equation}
We use the two-spike case for ease of illustration, as it allows the
cost function to be visualized as a two-dimensional surface.
Specifically, we evaluate the Hessian of the cost
function~\eqref{eq:loss_function} at each point $(\gamma_1,\gamma_2)$
and compute its smallest eigenvalue.
A binary mask is then constructed by labeling a point as inside
(white) the region of convergence if the minimum eigenvalue is
strictly positive, and outside (black) otherwise.
In other words, the white region corresponds to points where the
Hessian is positive definite, indicating that the cost function is
locally strongly convex in a neighborhood of that point.
In our work, we focus on the region of convergence that contains the
ground truth and mark the corresponding local minimizer as an orange
square.
The blue cross indicates the ground-truth location, while the magenta
circle represents the neighborhood $\mathcal{N}(\bm{\tau},\varrho)$
as predicted by our theoretical analysis, lying entirely within the
white region of convergence.
The radius of this magenta circle is defined as the largest value for
which it remains fully contained within the region of convergence
(note that this region need not be convex as a set).

\begin{figure}[t]
  \centering
  \begin{subfigure}[b]{0.40\textwidth}
    \includegraphics[width=\columnwidth]{%
      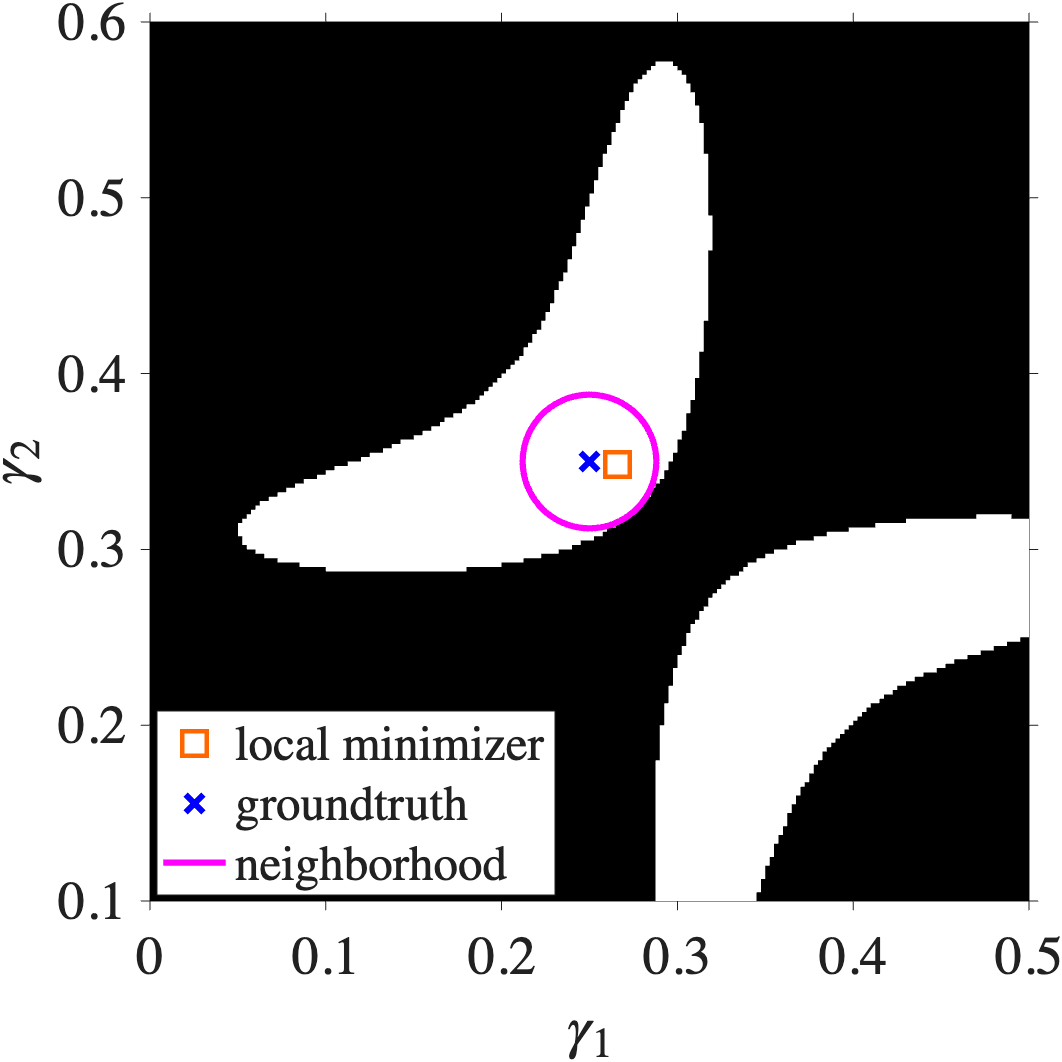}
    \caption{$\Delta = 0.1$}
  \end{subfigure}
  \hfill
  \begin{subfigure}[b]{0.40\textwidth}
    \includegraphics[width=\columnwidth]{%
      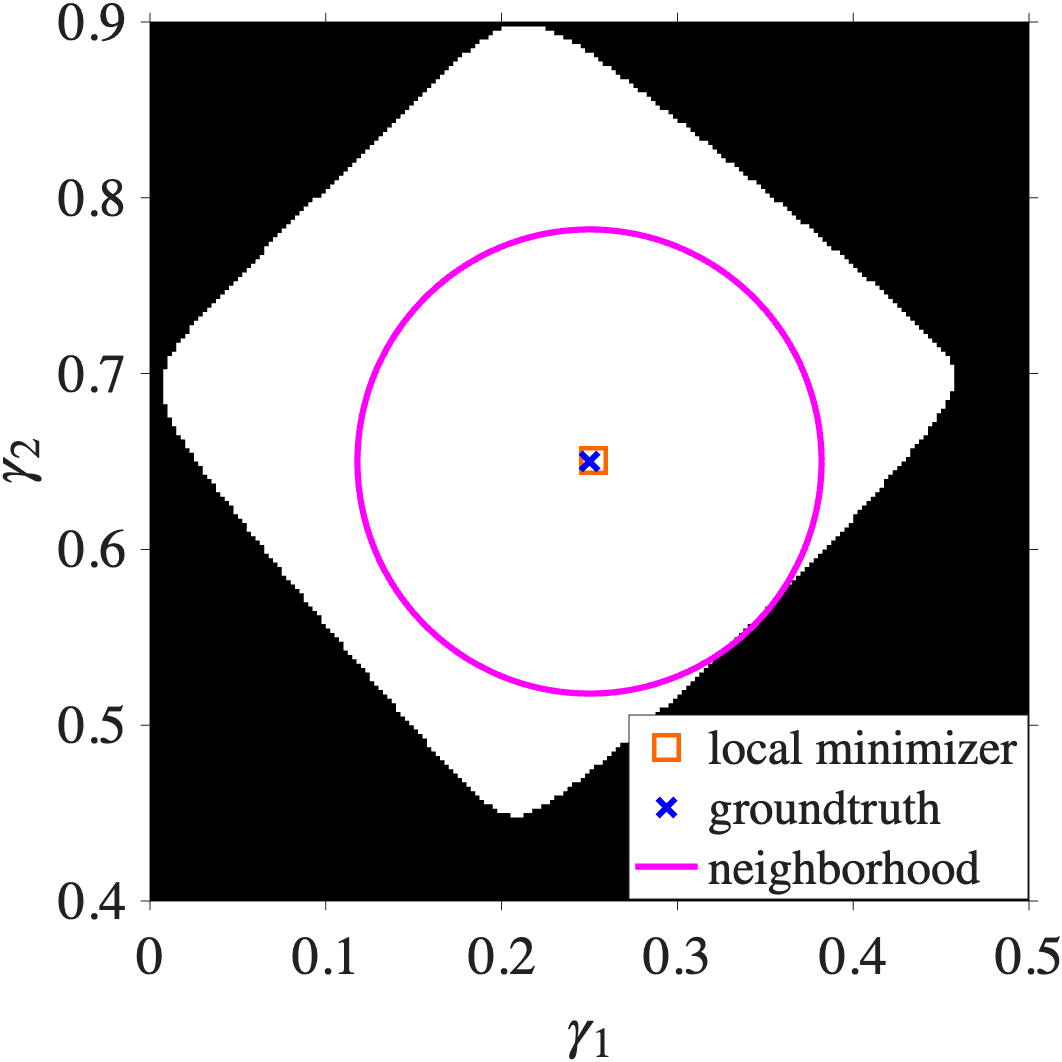}
    \caption{$\Delta = 0.4$}
  \end{subfigure}
  \caption{Visualization of the region of convergence of the VarProSD
    objective~\eqref{eq:loss_function} with varying minimum separation
    $\Delta$ under a Gaussian PSF ($\sigma = 0.3$, $K=2$,
    SNR $= 20$~dB, $L=20$, $N=11$, $T=1$).
    The white region is the region of convergence (positive-definite
    Hessian); the magenta circle is the theoretically characterized
    neighborhood $\mathcal{N}(\bm{\tau},\varrho)$ around the ground
    truth (blue cross); the orange square marks the local minimizer.}
  \label{fig:basin_vs_min_sep11}
\end{figure}
The plots in Figure~\ref{fig:basin_vs_min_sep11} illustrate how the
region of convergence and the neighborhood
$\mathcal{N}(\bm{\tau},\varrho)$ expand as the minimum separation
$\Delta$ between spikes increases, shown here under the Gaussian PSF
$g(t) = \frac{1}{\sqrt{2\pi}\sigma}
\exp\!\left(-\frac{t^2}{2\sigma^2}\right)$ with $\sigma=0.3$.
As $\Delta$ increases from $0.1$ to $0.4$, both the white region of
convergence and the magenta neighborhood grow, confirming that greater
source separation yields a larger basin within which recovery is
guaranteed. A rigorous statement of this characterization appears in Section~\ref{subsec:basin}.

\item \textbf{Estimation Error Analysis:}
Having established the basin of convexity and the existence 
of a unique local minimizer $\bm{\gamma}_\star$ within it, 
we also characterize its accuracy relative to the ground-truth 
spike locations $\bm{\tau}$. Prior non-asymptotic analyses 
of ESPRIT~\cite{li2022stability, yang2022nonasymptotic} 
establish consistency and $\mathcal{O}(1/\sqrt{L})$ error 
decay in the number of snapshots, but only for the trivial 
PSF. Prior work with a non-trivial PSF~\cite{kalra2024stable} 
establishes a similar non-asymptotic bound for ESPRIT, but 
with an extra term due to the PSF that persists even as 
$L \to \infty$. Under stochastic i.i.d.\ Gaussian noise, we show that the local minimizer $\bm{\gamma}_\star$ achieves consistency and the same $\mathcal{O}(1/\sqrt{L})$ decay rate for an arbitrary PSF, with the estimation error depending explicitly on the spectral properties of the PSF, the noise level, the number of spikes, and the dynamic range of the 
amplitudes. A rigorous statement appears in 
Theorem~\ref{thm:minimizer_error} in Section~\ref{subsec:basin}.

We also characterize the estimation error under adversarial 
perturbations. While the stochastic analysis above assumes 
i.i.d.\ Gaussian noise, perturbations in many real-world 
applications are structured rather than random. For example, 
in a spike-injection attack on localization, an adversary 
concentrates all noise energy into fake 
pulses~\cite{li2023channel}, yielding significantly larger 
estimation errors than stochastic models predict. Similarly, in surface EMG, cross-talk spikes from neighboring muscles~\cite{de2012inter} can be misinterpreted as legitimate motor-unit firings. These structured perturbations may be significantly underestimated by stochastic noise models.
We construct a worst-case adversarial perturbation as a
rank-one noise matrix that maximally degrades the estimator's
performance, and establish a deterministic error bound via a
local Lipschitz characterization of the inverse map from noise
realizations to the least-squares estimate. This viewpoint is inspired by the work in~\cite{basu2000stability}, 
where the authors proved the existence of a smooth inverse map for general nonlinear least-squares problems. Our contribution 
is to explicitly compute the Jacobian of this map in the spike 
deconvolution setting and characterize the worst-case adversarial 
noise, with the explicit construction given in Section~\ref{sec:adv_noise}. Fixing all problem parameters and varying only the noise, 
the adversarial bound is sharper than the stochastic bound by a 
factor of $\sqrt{K}$ in the high-SNR, large-snapshot regime. 
A rigorous statement appears in Theorem~\ref{thm:mainalt} 
in Section~\ref{sec:adv_noise}.
\item \textbf{Convergence of Gradient Descent:}
The existence and accuracy of the local minimizer 
$\bm{\gamma}_\star$ established in previous contributions hold independently of the algorithm used to find it.
We now analyze gradient descent for reaching $\bm{\gamma}_\star$ 
from within the basin $\mathcal{N}(\bm{\tau},\varrho)$.

While the variable projection framework provides a useful
reformulation by eliminating the amplitudes, its convergence analysis
under gradient descent in the spike deconvolution setting, to the
best of our knowledge, has not been previously addressed in the
literature.
Convergence of VP has been analyzed in non-smooth optimization via
proximal-gradient methods~\cite{van2021variable}, and for regularized VP
in ill-posed inverse problems, where local convergence is
established~\cite{espanol2025local}. Neither framework addresses plain gradient descent with the structure
inherent in spike deconvolution.
While these analyses establish local convergence of Gauss--Newton under
a sufficiently accurate initialization \cite{ruhe1980algorithms}, they do
not characterize the explicit size or radius of such a neighborhood.
In contrast, our contribution lies not in proposing a new optimization method but in providing the first explicit characterization of the
basin of convexity for the VP objective in the spike deconvolution
setting, in terms of interpretable problem parameters such as the
PSF properties, the sampling bandwidth, the spike separation, and
the dynamic range of the amplitudes.
Specifically, we establish that gradient descent initialized within
$\mathcal{N}(\bm{\tau},\varrho)$ converges linearly to
$\bm{\gamma}_\star$.
Rigorous convergence guarantees for gradient descent are established in 
Theorem~\ref{thm:gd_convergence} (Section~\ref{subsec:gd_convergence}).

\item \textbf{Beurling--Selberg Approximation Theory as a Key Theoretical Tool:}
Our theoretical analysis leverages bounds on the extreme singular
values of generalized Vandermonde matrices, also known as the large
sieve inequality, through the Beurling--Selberg approximation theory
of functions of bounded variation~\cite{vaaler1985ExtremalFunctions,
vaaler2023NumberLattice}. These functions are used in~\cite{Ferreira2022Stability} to
characterize a phase transition on the minimal separation needed to
stably estimate spike locations in the presence of noise given an
arbitrary PSF, and we build on this foundation throughout our analysis.
This approach enables our results to accommodate an arbitrary PSF,
assumed only to have a power spectral density of bounded variation. The resulting bounds are independent of the number of spikes $K$, providing guarantees that scale cleanly with the model order.
\item \textbf{Utility of PSF Parameter Characterization for Sampling 
Bandwidth Selection:}
A practical consequence of our theoretical analysis is a principled 
strategy for selecting the sampling bandwidth $B = N/T$ given a known 
PSF. The spectral characteristics of the PSF that govern the size of 
the basin of convexity also depend on $B$, and our analysis reveals 
that choosing a large $B$ is not always optimal. Hence, we develop a 
strategy to empirically choose the optimal $B$ based on the PSF 
characteristics and how it enlarge the basin and improve recovery. 
Our contribution is not to provide a theoretically or mathematically 
optimal $B$. Instead, we show empirically in 
Section~\ref{sec3 numerical result} that the bandwidth predicted by 
our PSF parameter characterization closely tracks the empirically 
optimal bandwidth that minimizes the empirical error across a range 
of PSFs (Gaussian, Morlet) and pulse widths. This provides a practical, PSF-driven method for sampling bandwidth selection.

\end{enumerate}

\subsection{Mathematical Notation}
Vectors and matrices are written in small boldface $\bm{a}$ and
capital boldface $\bm{A}$, respectively.
The notations $\bm{A}^\mathsf{T}$, $\bm{A}^\mathsf{H}$, and
$\bar{\bm{A}}$ denote the transpose, the conjugate transpose, and
the entry-wise conjugate of matrix $\bm{A}$, respectively.
The notations $\norm{\bm{A}}$, $\norm{\bm{A}}_{\mathrm{F}}$,
$\norm{\bm{A}}_{\infty \rightarrow 2}$ refer to the spectral norm,
Frobenius norm, and the largest $\ell_2$ norm of the columns of a
matrix $\bm{A}$, respectively.
The maximum (resp.\ minimum) eigenvalue and singular value of
$\bm{A}$ are denoted $\lambda_{\max}(\bm{A})$ and
$\sigma_{\max}(\bm{A})$ (resp.\ $\lambda_{\min}(\bm{A})$ and
$\sigma_{\min}(\bm{A})$).
The shorthand notation $[n]$ denotes the set of integers
$\{1,\ldots,n\}$.
The number $j$ is the basis of imaginary numbers, $j^2 = -1$. The notation $a\lesssim b$ denotes the relationship between two scalars $a$ and $b$ such that there exists an absolute constant $C>0$ for which $a \leq Cb$.
The symbols $\odot$ and $*$ denote the Hadamard and Khatri-Rao
product (also known as the column-wise Kronecker product) respectively.
The symbol $\hat{h}$ denotes the Fourier transform of $h$.
The total variation of a measure $\mu$ is written $V(\mu)$ and
defined by
\begin{align*}
V(\mu) := \sup \left\{ \int_{-\infty}^{+\infty} h(f)\,
\mathrm{d}\mu(f)\;;\quad h \in \mathcal{C}_0,\;
\norm{h}_{L_\infty} \leq 1 \right\},
\end{align*}
where $\mathcal{C}_0$ denotes the space of continuous functions of
the real variable.

\subsection{Organization of the paper}
The rest of the paper is organized as follows.
Section~\ref{sec:main_results} presents the main results, including the
local geometry of the basin of convexity, the estimation error of the
local minimizer under stochastic noise, an alternative stability analysis
under adversarial noise via the local Lipschitz characterization of the
inverse map, and local convergence guarantees for gradient descent.
Section~\ref{sec3 numerical result} provides numerical results to
validate the theoretical findings, including performance under different
initializations, resolution limits, the role of bandwidth selection, and
error scaling with key parameters.
Section~\ref{sec4 key lemma} introduces the key lemmas on the conditioning
of structured matrices that support the main theorems, while
Section~\ref{sec5 derivation proof} details the derivation of the main
results.
Finally, Section~\ref{sec6 discussion} presents the discussion and open
questions to conclude the paper.
The appendices provide supplementary material, including prior art on the
conditioning of structured matrices, properties of the Khatri--Rao and
Hadamard products, the closed-form gradient, Jacobian, and Hessian
identities for the VarProSD objective, and the proofs of the supporting
lemmas.

\section{Main Results}
\label{sec:main_results}

In this section, we present theoretical guarantees for the VarProSD estimator. We first characterize the local geometry of the VarProSD objective in a neighborhood of the ground truth spike locations. In particular, we define a quantifiable basin of convexity within which the objective function is strongly convex and has Lipschitz continuous gradient, ensuring the existence and uniqueness of a local minimizer.

We then analyze the estimation accuracy of this minimizer and derive error bounds that characterize how the estimation error depends on the residual covariance. In addition, we study the effect of structured or adversarial perturbations and establish an alternative stability bound that captures worst-case perturbations.

Finally, we analyze the behavior of practical optimization algorithms used to minimize the VarProSD objective and establish local convergence guarantees for gradient descent when initialized within the basin of attraction.

To state our results, we rely on the spectral properties of the PSF which is captured through the following quantities.
Let $P_g: \mathbb{R} \to \mathbb{R}$ denote the \textit{truncated} power spectral density (PSD) of $g$ restricted to the interval $J_N:= [-\frac{N-1}{2T}, \frac{N-1}{2T}]$ defined by
\begin{equation}
\label{eq:def_P_g}
P_g(f) := |\widehat{g}(f)|^2 \mathbbm{1}_{J_N}(f), \quad \forall f \in \mathbb{R}, 
\end{equation}
where $\mathbbm{1}_{J_N}: \mathbb{R} \to \mathbb{R}$ denotes the indicator function of the interval $J_N$. 
The truncated PSDs $P_{g'}$ and $P_{g''}$ of the first and second derivatives of $g$ are defined similarly when $g$ in \eqref{eq:def_P_g} is substituted by $g'$ and $g''$ respectively. 
The $L_1$-norm of $P_g$ is denoted by $E_{g}$ as
\begin{equation}
\label{eq:def_energy}
E_g := \norm{P_g}_{L_1}
\end{equation}
The parameter $E_{g'}$ (resp. $E_{g''}$) is defined similarly when $g$ in \eqref{eq:def_energy} is substituted by the first derivative $g'$ (resp. the second derivative $g''$). 
\begin{remark}
As an illustrative example, consider when the PSF is an ideal low-pass filter in the time domain. In this case, the Fourier transform is constant over $[-N/2T, N/2T]$, allowing explicit evaluation of the following parameters: \[ E_g = \frac{N}{T}, \, \,E_{g'} = \int_{-N/2T}^{N/2T} 4 \pi^2 u^2 du = \frac{\pi^2}{3} \left(\frac{N}{T} \right)^3 , \, \, E_{g''} = \int_{-N/2T}^{N/2T} 16 \pi^4 u^4 du = \frac{\pi^4}{5} \left(\frac{N}{T} \right)^5, \] 
\end{remark}
The ratio of the total variation of $P_g$ over its $L_1$ norm is denoted by $\rho_g$, \emph{i.e.}
\begin{equation}
\label{eq:def_rho_g}
\rho_g := \frac{V(P_g)}{E_{g}}. 
\end{equation}
The parameter $\rho_{g'}$ (resp. $\rho_{g''}$) is defined similarly when $g$ in \eqref{eq:def_rho_g} is substituted by the first derivative $g'$ (resp. the second derivative $g''$). 
Then $\rho$ is defined as the maximum normalized total variations in the bandwidth $J_N$, \emph{i.e.}
\begin{equation}\label{eq:rho_def_max}
   \rho :=  \max \{\rho_g,\rho_{g'},\rho_{g''} \}
\end{equation} 
These quantities will be used to characterize the size of the neighborhood and the local curvature of the VarProSD objective in the results that follow. 
For notational convenience, we furthermore also define
\begin{align}
\label{eq:def_Agamma}
\bm A_\gamma := \bm G \bm \Phi_\gamma,
\end{align}
which will be used in the analysis throughout.
We also define the $\ell_2$ and $\ell_\infty$ torus distances between
$\bm{\tau}$ and its estimate $\bm{\gamma}$ as
\begin{align}
\label{eq:d2_torus}
\mathrm{d}_2(\bm{\tau}, \bm{\gamma})
&:= \sqrt{\sum_{j=1}^K d_{\mathbb{T}}(\tau_j,\gamma_j)^2}, 
\end{align}
\begin{align*}
\mathrm{d}_\infty(\bm{\tau}, \bm{\gamma})
&:= \max_{j \in [K]} d_{\mathbb{T}}(\tau_j, \gamma_j), 
\end{align*}
where $d_{\mathbb{T}}$ is defined in~\eqref{eq:torus_def},
and the dynamic range of the amplitudes as
\begin{align*}
\kappa = \frac{\norm{\bm{X}}}{r_{\min}(\bm{X})},
\end{align*}
where $r_{\min}(\bm{X}) := \min_{j \in [K]}
\norm{\bm{e}_j^\mathsf{H}\bm{X}}_2$ denotes the minimum row norm
of $\bm{X}$.

\subsection{Local Neighborhood and Geometry}
\label{subsec:basin}

We begin by defining a neighborhood around the ground truth spike
locations within which the VarProSD objective \eqref{eq:loss_function}
exhibits well-behaved local
geometry. Then, we characterize the behavior of the objective within this region and establish conditions under which it admits a unique local minimizer.

\begin{theorem}[Local regularity and stability of the minimizer within the neighborhood]
\label{thm:local_geometry}

Let $\bm{\tau} \in \mathbb{R}^K$ denote the ground truth spike locations and assume the minimum separation between the spikes satisfies $\Delta > \frac{2}{3}\rho \kappa^2$. 
We define the following neighborhood
\begin{equation}
\label{eq:neighborohood}
\mathcal{N}(\bm{\tau},\varrho)
:=
\left\{
\bm{\gamma} \in \mathbb{R}^K :
\mathrm{d}_2(\bm{\gamma},\bm{\tau}) \le \varrho
\right\},
\end{equation}
where the radius $\varrho$ is given by
\begin{equation}
\label{eq:basin_radius_cond}
\varrho
=
\frac{1}{2}\!\left(\Delta - \frac{2}{3}\rho \kappa^2\right)
\;\wedge\;
c_1 \kappa^{-2}
\sqrt{\frac{E_g}{E_{g'}} \wedge \frac{E_{g'}}{E_{g''}}}.
\end{equation}
Furthermore, also suppose the residual covariance $\norm{\bm{R}}  = \min_{c\in \mathbb{R}}\norm{\frac{1}{L}\bm{Y}\bm{Y}^\mathsf{H}- \frac{1}{L}\bm{Y}_0\bm{Y}_0^\mathsf{H} - c\bm{I}_N}$ satisfies
\begin{equation}
\label{eq:cond_on_R1}
\frac{\|\bm R\|}{T E_g \, r_{\min}^2(\bm X)}
\le
c_1 
\sqrt{\frac{E_{g'}}{E_g}}
\left(
\sqrt{\frac{E_g}{E_{g'}} \wedge \frac{E_{g'}}{E_{g''}}}
\;\wedge\;
\frac{\varrho}{\sqrt{K}}
\frac{r_{\max}(\bm X)}{r_{\min}(\bm X)}
\right).
\end{equation}
for some absolute constant $c_1 >0$.
Then, over $\mathcal N(\bm{\tau},\varrho)$, 
the objective $\ell(\bm{\gamma})$ is strongly convex and has Lipschitz
continuous gradient. In particular, there exists a unique local minimizer
$\bm{\gamma}_\star \in \mathcal N(\bm{\tau},\varrho)$.
\end{theorem}
The proof of Theorem~\ref{thm:local_geometry} is provided in Section~\ref{subsec:proof_local_geometry}.
\noindent
\newline
\textbf{Interpretation:}
The quantity $\varrho$ in \eqref{eq:basin_radius_cond} characterizes the size of the neighborhood around the ground truth within which the objective admits a well-behaved local geometry. For this neighborhood to exist, the separation condition $\Delta > \frac{2}{3}\rho \kappa^2$ must hold. In particular, when the dynamic range $\kappa$ is large, a larger separation is required to ensure the existence of such a region.

The expression for $\varrho$ consists of two terms. The first term arises from a fundamental minimum separation requirement, ensuring that the spikes are sufficiently well-resolved. The second term is introduced to simplify our theoretical analysis in interpretable form. It depends explicitly on the spectral properties of the PSF through $E_g, E_{g'}$, and $E_{g''}$.
For instance, in the case of an ideal low-pass PSF, the second term scales as $\kappa^{-2} \frac{T}{N}$, up to some constants. This allows for a more explicit interpretation in terms of the bandwidth $B = N/T$. When $B$ is small, the first term in \eqref{eq:basin_radius_cond} dominates. In this regime, $\Delta$ must scale with $\rho$ to keep the radius positive, thereby expanding the neighborhood. As $B$ grows, the second term dominates, and the radius correspondingly shrinks. This illustrates how the geometry of the basin radius depends on the sampling bandwidth.

The quantity $\frac{\|\bm R\|}{T E_g \, r_{\min}^2(\bm X)}$ in \eqref{eq:cond_on_R1} can be related to a normalized covariance estimation error of the form $\frac{\|\bm Y \bm Y^\mathsf{H} - \bm Y_0 \bm Y_0^\mathsf{H}\|}{\|\bm Y_0 \bm Y_0^\mathsf{H}\|}$ up to conditioning of $\bm X$. In particular, $T E_g \, r_{\min}^2(\bm X)$ can be used as a lower bound on $\|\bm Y_0 \bm Y_0^\mathsf{H}\|$.
This provides an interpretation of the residual covariance condition as requiring the covariance estimation error to be sufficiently small relative to the signal energy.

After establishing that the VarProSD objective exhibits
favorable local geometry within the basin $\mathcal N(\bm{\tau},\varrho)$ and therefore, admits a unique local minimizer $\bm{\gamma}_\star$ in this region, we now characterize the accuracy of this minimizer relative to the ground truth
spike locations $\bm{\tau}$.

\begin{theorem}[Estimation error of the local minimizer]
\label{thm:minimizer_error}
Under the assumptions of Theorem~\ref{thm:local_geometry}, the unique local
minimizer $\bm{\gamma}_\star \in \mathcal N(\bm{\tau},\varrho)$ satisfies the following estimation error bound
\begin{equation}
\mathrm{d}_2(\bm{\gamma}_\star,\bm{\tau})
\le   
c_2 \sqrt{K}  \, \sqrt{\frac{E_g}{E_{g'}}} \frac{\|\bm{R}\|}{T E_g \,  r_{\min}(\bm{X})^2},
\label{eq:thm:error_bonnd}
\end{equation}
for some absolute constant $c_2>0$,
\end{theorem}

The proof of Theorem~\ref{thm:minimizer_error} is provided in Section~\ref{subsec:proof_gd}.
\newline
The bound in \eqref{eq:thm:error_bonnd} characterizes how the estimation error
of the local minimizer depends on the residual covariance matrix $\bm{R}$.
In particular, under stochastic noise, suppose the entries of $\bm{Z}$ are i.i.d. $\mathcal{N}(0,\sigma_{\bm Z}^2)$.
Then, choosing $c=\sigma_{\bm Z}^2$ in the definition of $\bm{R}$, we obtain
\begin{align*}
L\norm{\bm R}
&=
\norm{
\bm Y \bm Y^\mathsf H
-
\bm Y_0 \bm Y_0^\mathsf H
-
L\sigma_{\bm Z}^2 \bm I
} \\
&=
\norm{
\bm Y_0 \bm Z^\mathsf H
+
\bm Z \bm Y_0^\mathsf H
+
\bm Z\bm Z^\mathsf H
-
L\sigma_{\bm Z}^2 \bm I
} \\
&\le
2\norm{\bm Y_0 \bm Z^\mathsf H}
+
\norm{\bm Z\bm Z^\mathsf H - L\sigma_{\bm Z}^2 \bm I }.
\end{align*}
By Gaussian concentration \cite{wainwright2019high,vershynin2018high} (also see \cite[Eq.~(67)]{yang2022nonasymptotic}), with probability at least $1 - 2e^{-c_2 N}$,
\begin{align*}
\norm{\bm{Z}\bm{Z}^\mathsf{H} - L\sigma_{\bm{Z}}^2 \bm{I}_N}
\le 8\sigma_{\bm{Z}}^2 \max\, \!\bigl(\sqrt{NL},\,N\bigr).
\end{align*}
Since we can write $\bm{Y}_0 \bm{Z}^\mathsf{H} = \bm{C}^{1/2}\bm{\Upsilon}$, where $\bm{C} = \sigma_{\bm{Z}}^2 \bm{Y}_0 \bm{Y}_0^\mathsf{H}$ and $\bm{\Upsilon} \in \mathbb{R}^{N \times N}$ is a random matrix with i.i.d.\ $\mathcal{N}(0,1)$ entries, it follows that
\begin{align*}
\norm{\bm{Y}_0 \bm{Z}^\mathsf{H}}
&\le \sigma_{\bm{Z}} \norm{(\bm{Y}_0 \bm{Y}_0^\mathsf{H})^{1/2}} \norm{\bm{\Upsilon}} \\
&\le 3\sigma_{\bm{Z}} \norm{\bm{Y}_0} \sqrt{N}
\end{align*}
with probability at least $1 - e^{-c_2 N}$.
Using $\norm{\bm Y_0} \leq \frac{2}{\sqrt{3}}\sqrt{T E_g}\,\norm{\bm X}$ (Lemma~\ref{lem:ferreira_thm1}), we further obtain
\begin{align*}
\norm{\bm{Y}_0 \bm{Z}^\mathsf{H}}
\le 2\sqrt{3}\,\sigma_{\bm{Z}} \sqrt{T E_g}\, \norm{\bm{X}} \sqrt{N}.
\end{align*}
Combining the bounds, with probability at least $1 - 3e^{-c_2 N}$,
\begin{align*}
L\norm{\bm{R}}
&\le 6\sqrt{3}\,\sigma_{\bm{Z}} \sqrt{T E_g}\, \norm{\bm{X}} \sqrt{N}
+ 8\sigma_{\bm{Z}}^2 \max\!\bigl(\sqrt{NL},\,N\bigr).
\end{align*}
For additional simplification, assume that the amplitude matrix $\bm{X}$ is deterministic and well-conditioned, with $\|\bm{X}\| = \mathcal{O}(\sqrt{L}), \quad
r_{\max}(\bm{X}) = \mathcal{O}(\sqrt{L}), \quad
r_{\min}(\bm{X}) = \Theta(\sqrt{L})$.
Then, in the high-SNR regime (noise variance is small relative to signal energy), the dominant terms yield
\begin{align*}
\|\bm{R}\|
\lesssim
\sigma_{\bm Z}\sqrt{\frac{N T E_g}{L}}.
\end{align*}
Substituting this into \eqref{eq:thm:error_bonnd} yields that the estimation error scales as
\begin{align*}
\mathrm{d}_2(\bm{\gamma}_\star,\bm{\tau})
\;\lesssim\;
\frac{ \sigma_{\bm Z} \sqrt{K\,N} }
{ \sqrt{L \,T\, E_{g'}}}.
\end{align*}
We can further specialize the bound to the ideal low-pass (LPF) case to make the dependence on the bandwidth more explicit. For an LPF kernel with bandwidth $B = \frac{N}{T}$, it holds that $E_{g'}$ scales as $B^3$. Substituting this into the above bound yields
\begin{align*}
\mathrm{d}_2(\bm{\gamma}_\star,\bm{\tau})
\;\lesssim\;
\sigma_{\bm Z} \sqrt{\frac{K}{L}} \;
\frac{T}{N}.
\end{align*}
This shows that the estimation error decreases with the number of snapshots $L$ as $L^{-1/2}$, improves with increasing bandwidth $B$, and grows linearly with the noise level $\sigma_{\bm Z}$. 

However, the above argument relies on stochastic assumptions on the noise and, in particular, on concentration properties of the empirical covariance matrix. Such assumptions may be violated in the presence of structured or adversarial perturbations, where the covariance error does not decay with the number of snapshots. This motivates a complementary, deterministic stability analysis of the least-squares estimator that does not rely on stochastic assumptions. We develop such an analysis in the next subsection.

\subsection{Adversarial Noise Analysis}
\label{sec:adv_noise}

The previous subsection characterized the accuracy of the least-squares
estimator $\bm{\gamma}_\star$ through the residual covariance bound.
We now study a complementary perturbation model in which the noise may be
structured or adversarial rather than stochastic. As discussed in the
introduction, such perturbations arise in applications including
spike-injection attacks~\cite{li2023channel} and surface
EMG~\cite{de2012inter}, where they can yield estimation errors far larger
than stochastic models predict.

Motivated by these examples, we now present an alternative stability
analysis of the constrained least-squares estimator that
yields a sharper error bound in the presence of adversarial noise.

\subsubsection{Alternative Stability Bound under Adversarial Noise}

Recall that under the assumptions introduced in Section~\ref{subsec:basin}, 
the VarProSD objective is strongly convex and has Lipschitz continuous gradient 
within the basin $\mathcal N(\bm{\tau},\varrho)$ and therefore admits a unique 
local minimizer in this neighborhood. Let $\bm{\gamma}_\star$ denote the 
constrained least-squares estimator, i.e., the unique local minimizer of 
$\ell(\bm{\gamma})$ over the basin $\mathcal N(\bm{\tau},\varrho)$.

By substituting the observation model
$\bm{Y} = \bm{A}_{\bm{\tau}} \bm{X} + \bm{Z}$ into the cost function
\eqref{eq:loss_function}, the objective can be parameterized explicitly by the
noise matrix $\bm{Z}$ as
\begin{align*}
\ell(\bm{\gamma};\bm{Z}) = 
\frac{1}{2L} \norm{\bm{P}_{\bm{\gamma}}^\perp (\bm{A}_{\bm{\tau}} \bm{X}+ \bm{Z})}_\mathrm{F}^2.
\end{align*}
We now define an inverse map that takes the 
vectorization of the noise 
matrix $\bm{Z}$ as input and outputs the constrained least-squares estimator 
$\bm{\gamma}_\star$, i.e.
\begin{equation}
\label{eq:Psi_def}
 \bm{\psi} : \bm{z} \to \bm{\gamma}_\star := \mathop{\mathrm{argmin}}_{\bm{\gamma} \in \mathcal{N}(\bm{\tau},\varrho)}  \, \ell(\bm{\gamma};\bm{Z}),
\end{equation} 
where $\bm{z} = \mathrm{vec}(\bm{Z})$.
By analyzing the local Lipschitz continuity of $\bm{\psi}$, we establish the
following error bound on $\bm{\gamma}_\star$.

\begin{theorem}
\label{thm:mainalt}
Suppose the assumptions of Theorem~\ref{thm:local_geometry} hold and that the
spectral norm of the noise matrix satisfies
\begin{align}
\label{eq:cond_on_Z_in_lemma}
\norm{\bm{Z}} \le c_5 \sqrt{T\,E_g}\,\norm{\bm{X}}
\end{align}
for some constant $c_5>0$. Then there exists a numerical constant $c_6>0$ 
such
that the estimation error of the constrained least-squares estimator
$\bm{\gamma}_\star$ satisfies
\begin{equation}
\label{eq:theorem_main_result}
\mathrm{d}_2(\bm{\gamma}_\star,\bm{\tau})
\le
\frac{c_6\,\kappa^2 \,\norm{\bm{Z}}_\mathrm{F}}
{\sqrt{T\,E_{g'}}\,\norm{\bm{X}}}.
\end{equation}
\end{theorem}

The proof of Theorem~\ref{thm:mainalt} is presented in Section~\ref{subsec:proof_mainalt}.
\newline
Theorem~\ref{thm:mainalt} characterizes the sensitivity of the constrained 
least-squares estimator to structured perturbations in the observations. In 
contrast to the stochastic error bound in Theorem~\ref{thm:minimizer_error}, 
the estimation error now depends directly on the Frobenious norm of the 
perturbation matrix $\bm Z$.

To compare this bound with the one established in 
Theorem~\ref{thm:minimizer_error}, we construct an explicit worst-case noise 
instance in the next subsection that demonstrates the sharpness of 
Theorem~\ref{thm:mainalt} and clarifies when the inverse-map bound offers 
a tighter guarantee.

\subsubsection{Construction of Worst-Case Adversarial Noise}

We next demonstrate the sharpness of Theorem~\ref{thm:mainalt}'s error bound 
by constructing an adversarial noise instance. Specifically, we align all 
perturbation energy with the top singular vector of the Jacobian of $\bm{\psi}$ 
evaluated at $\bm{z}=\bm0$.
We begin by applying a Taylor expansion around $\bm{z}=\bm0$ to locally 
linearize the inverse map in \eqref{eq:Psi_def} as
\[
\bm{\psi}(\bm{z})\approx \bm{\psi}(\bm{0}) + \bm{z}^\mathsf{T} \, 
\nabla_{\bm{z}}\bm{\psi}(\bm{0}) 
\]
for sufficiently small $\bm{z}$, where $\nabla_{\bm{z}}\bm{\psi}(\bm{0})$ 
denotes the Jacobian of the inverse map at $\bm{z}=\bm{0}$.
By the implicit function theorem \cite{rudin1976principles} (see also 
\cite[Theorem~6]{basu2000stability}), the Jacobian of the inverse map $\bm{\psi}$ 
can be expressed as
\begin{equation}
\label{eq:jacobian}
\nabla_{\bm{z}} \bm{\psi}(\bm{z}) 
= - \left( \nabla^2_{\bm{\gamma}} \bm{\ell}({\bm{\gamma}}) \right)^{-1} 
\left( \nabla_{\bm{z}}\nabla_{\bm{\gamma}} \bm{\ell}({\bm{\gamma}} ;\bm{Z}) \right). 
\end{equation}
where $\bm{\ell}({\bm{\gamma}})$ is the cost function is defined in \eqref{eq:loss_function} and we identify that the first term on the right-hand side of \eqref{eq:jacobian} 
as the Hessian of that cost function. The detailed expression for the Hessian is given in Lemma~\ref{lem:expression_for_hess_alg} in Appendix~\ref{secA3}.
Next, we present 
the expression for the cross-derivative term 
$\nabla_{\bm{z}}\nabla_{\bm{\gamma}} \bm{\ell}({\bm{\gamma}};\bm{Z})$ on the 
right-hand side of \eqref{eq:jacobian} in the following lemma.
\begin{lemma}
\label{lem:jacobian_iota}
The cross derivative of the cost function with respect to $\bm{z}$ is
\begin{align*}
\nabla_{\bm{z}}\nabla_{\bm{\gamma}} \bm{\ell}({\bm{\gamma}};\bm{Z})   = 
& - \frac{1}{2L} \begin{bmatrix}
 \bm{Y}^\mathsf{T} \bm{A}_{\bm{\gamma}}^{\dagger\mathsf{T}} \ast 
 {\bm{P_{\bm{\gamma}}}^\perp \bm{\Lambda} \bm{A}_{\bm{\gamma}}}  + 
 \bm{Y}^\mathsf{H} {\bm{P_{\bm{\gamma}}}^\perp \bm{\Lambda} \bm{A}_{\bm{\gamma}}} 
 \ast \bm{A}_{\bm{\gamma}}^{\dagger\mathsf{T}}\end{bmatrix}^\mathsf{H} 
\end{align*}
where $\ast$ denotes the Khatri-Rao product. 
\end{lemma}
The proof of Lemma~\ref{lem:jacobian_iota} is given in Appendix~\ref{subsecA3:crossderiv}.

\noindent To construct the adversarial noise, we evaluate the Jacobian of the inverse 
map in \eqref{eq:jacobian} at the ground truth $\bm{\tau}$ with $\bm{z}=\bm0$, 
which yields
\begin{equation}
\label{eq:jacobian_at_0}
\nabla_{\bm{z}} \bm{\psi}(\bm{0}) 
= - \left( \nabla^2_{\bm{\gamma}} \bm{\ell}({\bm{\tau}}) \right)^{-1} 
\left( \nabla_{\bm{z}}\nabla_{\bm{\gamma}} \bm{\ell}({\bm{\tau}};\bm{0}) \right),
\end{equation}
where
\begin{subequations}
\label{eq:jacobian_at_zero}
\begin{equation}
\label{eq:jacobian_at_zero_gamma}
\nabla^2_{\bm{\gamma}} \bm{\ell}({\bm{\tau}}) 
=   \frac{1}{L}\mathrm{Re} \left( 
\overline{\bm{X}} \bm{X}^\mathsf{T}
\odot
\bm{A}_{\bm{\tau}}^\mathsf{H} \bm{\Lambda}^\mathsf{H}  \bm{P}_{\bm{\tau}}^\perp  
\bm{\Lambda} \bm{A}_{\bm{\tau}}  \right)
\end{equation}
and
\begin{align*}
\nabla_{\bm{z}}\nabla_{\bm{\gamma}} \bm{\ell}({\bm{\tau}};\bm{0}) = 
-\frac{1}{2L}\begin{bmatrix}
\bm{X}^\mathsf{T} \ast \bm{P}_{\bm{\tau}}^\perp \bm{\Lambda} \bm{A}_{\bm{\tau}}
\end{bmatrix}^\mathsf{H}. 
\end{align*}
\end{subequations}
Taking the singular value decomposition of \eqref{eq:jacobian_at_0}, we select 
the top right singular vector corresponding to the largest singular value. 
Reshaping this length-$NL$ vector into an $N \times L$ complex matrix produces 
the adversarial noise $\bm{Z}_{\mathrm{adv}}$.

\begin{remark}
If the entries of $\bm{X}$ are zero mean and uncorrelated across sources, then 
one can show that the Gramian of the Jacobian matrix in \eqref{eq:jacobian_at_0} 
converges to a Kronecker product of structured matrices as $L \to \infty$. 
First, note that the matrix $\overline{\bm{X}} \bm{X}^\mathsf{T}$ in 
\eqref{eq:jacobian_at_zero_gamma} converges to the complex conjugate of the 
autocorrelation matrix $\bm{R}_{\bm X} = \frac{1}{L} \bm X \bm X^\mathsf{H}$, 
which is a diagonal matrix in this scenario, whereby the right-hand side of 
\eqref{eq:jacobian_at_zero_gamma} also becomes a diagonal matrix after the 
Hadamard product. The Gramian $\bm{C}$ of the Jacobian matrix in 
\eqref{eq:jacobian_at_0} is then written as
\[
\bm{C} = \left[\bm{A} \ast \bm{B} \right] \bm{D} \left[\bm{A} \ast \bm{B} \right]^\mathsf{H},
\]
with $\bm{A} = \bm{X}^\mathsf{T}$ and $\bm{B} = \bm{P}_{\bm{\tau}}^\perp 
\bm{\Lambda} \bm{A}_{\bm{\tau}}$, where $\bm{D} = \frac{1}{4L^2} \left( 
\nabla^2_{\bm{\gamma}} \bm{\ell}({\bm{\tau}}) \right)^{-1}$ is a diagonal matrix. 
Let $\bm{a}_k$ and $\bm{b}_k$ denote the $k$th column of $\bm{A}$ and $\bm{B}$, 
respectively, and $d_k$ denote the $k$th diagonal entry of $\bm{D}$, for 
$k = 1,\dots,K$. Then
\[
\bm{C} = \sum_{k=1}^K (\bm{a}_k \otimes \bm{b}_k) d_k (\bm{a}_k \otimes \bm{b}_k)^\mathsf{H} 
= \sum_{k=1}^K d_k \bm{a}_k \bm{a}_k^\mathsf{H} \otimes \bm{b}_k \bm{b}_k^\mathsf{H}
= \bm{A} \bm{D} \bm{A}^\mathsf{H} \otimes \bm{B} \bm{B}^\mathsf{H}. 
\]
Since $\bm{C}$ is written as the Kronecker product of two matrices, by 
\cite[Proposition 7.1.10]{bernstein2009matrix}, each eigenvector of $\bm{C}$ is 
given as the Kronecker product of eigenvectors of the constituent factors, 
$\bm{A} \bm{D} \bm{A}^\mathsf{H}$ and $ \bm{B} \bm{B}^\mathsf{H}$.
Therefore, the adversarial noise matrix constructed from the dominant eigenvector 
of $\bm{C}$ has rank one since the Kronecker product is reshaped into the outer 
product. 
\end{remark}

\begin{figure}[t]
  \centering
  \begin{subfigure}[b]{0.49\textwidth}
    \includegraphics[width=\textwidth]{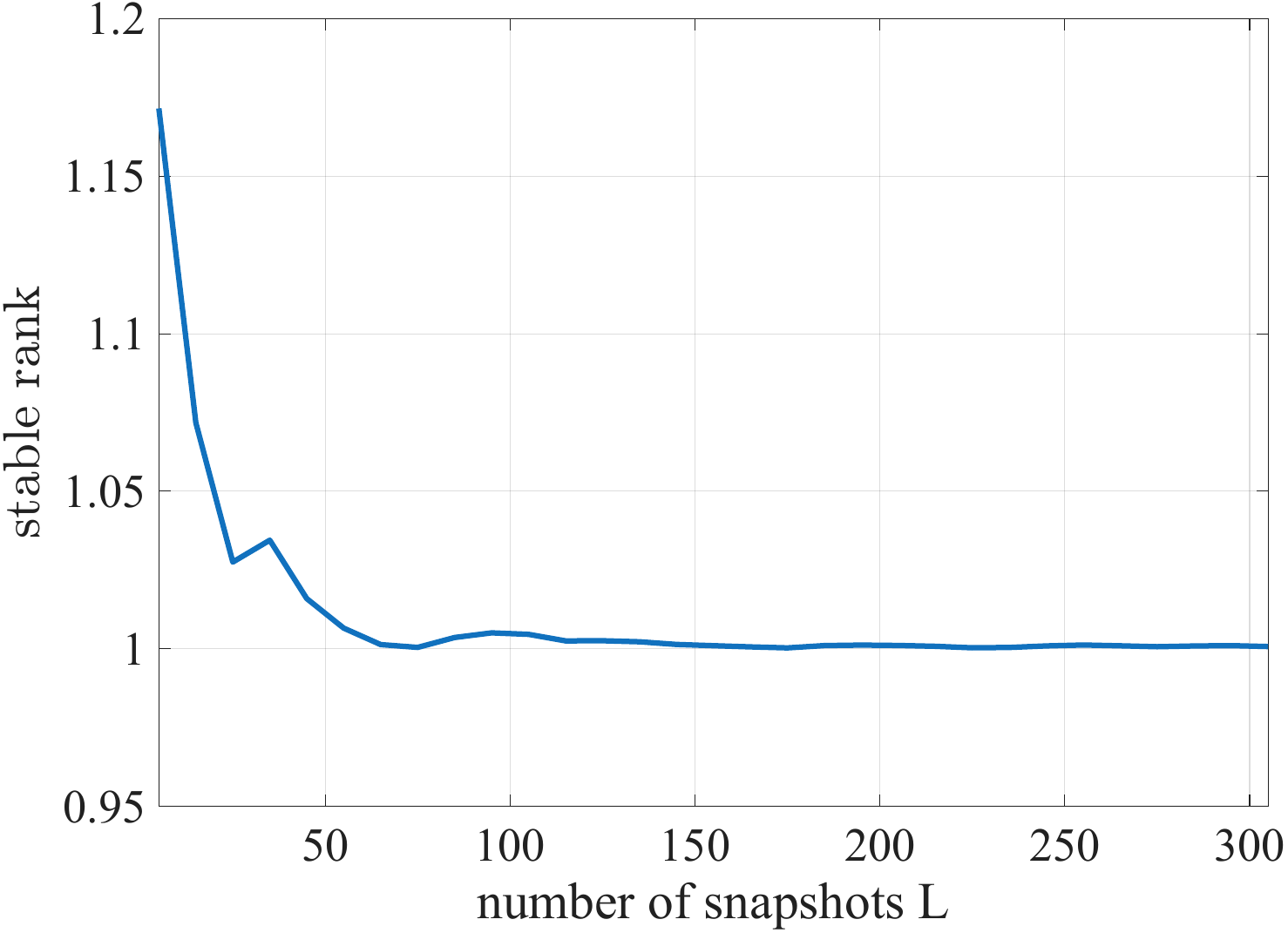}
    \caption{Stable rank of adversarial noise vs.\ $L$.}
    \label{fig:stablerank}
  \end{subfigure}
\end{figure}

In addition to this asymptotic rank-$1$ property of the adversarial noise, we 
also illustrate the empirical convergence in a Monte Carlo simulation. 
Figure~\ref{fig:stablerank} illustrates that an upper bound on the rank of 
$\bm{Z_{\text{adv}}}$ called the stable rank, which is 
$\norm{\bm{Z_{\text{adv}}}}_{\mathrm{F}}^2/\norm{\bm{Z}_{\text{adv}}}^2$, 
converges to one as the number of snapshots $L$ increases.

The preceding remark establishes that the adversarial noise has an asymptotic 
rank-one structure, with its stable rank converging to one as the number of 
snapshots grows. Consequently, its spectral and Frobenius norms coincide as 
$L \to \infty$.

\subsubsection{Comparison of Error Bounds (Theorem~\ref{thm:minimizer_error} vs. Theorem~\ref{thm:mainalt})}
\label{subsec:comp_of_bounds}
We now revisit the error bound in Theorem~\ref{thm:minimizer_error} and 
simplify it under an adversarial perturbation model to enable comparison with 
the inverse-map bound in Theorem~\ref{thm:mainalt}.
Recall that the estimation error satisfies
\begin{align}
\mathrm{d}_2(\bm{\gamma}_\star,\bm{\tau})
& \le   
c_2 \,\sqrt{K}\, \sqrt{\frac{E_g}{E_{g'}}}
\cdot \frac{\|\bm{R}\|}{T E_g  r_{\min}(\bm{X})^2}.
\label{eq:adv_start}
\end{align}
In the adversarial setting, we choose $c=0$ in the definition of $\bm R$, yielding
\begin{align*}
L\|\bm{R}\|
&=  \left\| \bm Y \bm Y^\mathsf{H} - \bm Y_0 \bm Y_0^\mathsf{H} \right\| \le  \left( 2 \|\bm Y_0 \bm Z^\mathsf{H}\| + \|\bm Z \bm Z^\mathsf{H}\| \right).
\label{eq:R_bound_adv}
\end{align*}
In the high-SNR condition where the cross-term dominates the pure noise term 
$\|\bm Y_0 \bm Z^\mathsf{H}\| \gg \|\bm Z \bm Z^\mathsf{H}\|$, the residual 
covariance is primarily governed by the cross-term, and we obtain
\begin{align*}
L \|\bm{R}\|
\lesssim  \|\bm Y_0 \bm Z^\mathsf{H}\|
\le  \|\bm A_{\bm \gamma}\| \, \|\bm X\| \, \|\bm Z\|.
\end{align*}

Using $\|\bm A_{\bm{\gamma}}\| \lesssim \sqrt{T E_g} \|\bm X\|$ 
(cf.~\cite[Theorem~1]{ferreira2023conditionNumber}), we obtain 
$L\|\bm{R}\| \lesssim \sqrt{T E_g} \, \|\bm X\| \, \|\bm Z\|$.
Substituting this bound into~\eqref{eq:adv_start} and using 
$\kappa = \frac{\|\bm X\|}{r_{\min}(\bm X)}$, we obtain
\begin{equation}
\label{eq:simplifies_gd_bound}
\mathrm{d}_2(\bm{\gamma}_\star,\bm{\tau}) 
\lesssim  
\frac{\sqrt{K} \, \kappa^2 \, \|\bm{Z}\|}{\sqrt{T E_{g'}} \, \|\bm X\|}.
\end{equation}

On comparing \eqref{eq:simplifies_gd_bound} with the inverse-map bound in 
\eqref{eq:theorem_main_result}, we observe that in the limit $L \to \infty$, 
the adversarial perturbation $\bm{Z}_{\mathrm{adv}}$ has stable rank one, 
ensuring that its spectral and Frobenius norms coincide. As a result, the two 
bounds differ only in the additional $\sqrt{K}$ factor appearing in the gradient 
descent expression. This establishes that the inverse-map bound is sharper by a 
factor of $\sqrt{K}$, thereby confirming its theoretical advantage in the 
high-SNR, large-snapshot regime.

This section establishes that the VarProSD objective admits a unique local 
minimizer within the basin $\mathcal N(\bm{\tau},\varrho)$ and characterizes its 
estimation accuracy. We will now analyze the behavior of optimization methods 
used to minimize the VarProSD objective in the next section. In particular, we 
establish local convergence guarantees for gradient descent when initialized within the basin of attraction.

\subsection{Local Convergence Analysis of Gradient Descent}
\label{subsec:gd_convergence}
We now analyze the behavior of gradient descent for minimizing the VarProSD
objective. Leveraging the local regularity properties established in
Theorem~\ref{thm:local_geometry}, we show that gradient descent initialized
inside $\mathcal N(\bm{\tau},\varrho)$ remains in this neighborhood and
converges linearly to the unique local minimizer.

Gradient descent updates the spike locations according to 
\begin{align*}
\bm{\gamma}_{t+1}
=
\bm{\gamma}_t - \alpha \nabla_{\bm{\gamma}} \ell(\bm{\gamma}_t),
\end{align*}
where $\alpha>0$ denotes a constant step size and $\bm{\ell}({\bm{\gamma}})$ is the cost function is defined in \eqref{eq:loss_function}.
The gradient of the VarProSD objective admits a closed-form expression,
stated below for completeness.
\begin{lemma}
\label{lem:expression_for_q_alg}
Suppose that all entries of $\bm{\gamma}$ are distinct and define
\begin{align}
\label{def:lambda}
[\bm{\Lambda}]_{i,i} = - \mathsf{j} 2 \pi f_i, 
\qquad i \in [N].
\end{align}
Then, the partial derivative of $\ell(\bm{\gamma})$ with respect to the
$k$-th coordinate $\gamma_k$ is given by
\begin{align*}
\frac{\partial \ell(\bm{\gamma})}{\partial \gamma_k}
=
- \frac{1}{L}\,
\mathrm{Re}\!\left[
\bm{e}_k^\mathsf{H}
(\bm{A}_{\bm{\gamma}})^\dagger
\bm{Y}\bm{Y}^\mathsf{H}
\bm{P}_{\bm{\gamma}}^\perp
\bm{\Lambda}
\bm{A}_{\bm{\gamma}}
\bm{e}_k
\right].
\end{align*}
\end{lemma}
The proof of Lemma~\ref{lem:expression_for_q_alg} is given in Appendix~\ref{subsecA3:grad}.
\newline \noindent 
The explicit form in Lemma~\ref{lem:expression_for_q_alg} enables efficient
evaluation of the gradient without numerical differentiation and will be used
in the convergence analysis below. We now establish the local convergence guarantee for gradient descent.

\begin{theorem}[Local linear convergence of gradient descent]
\label{thm:gd_convergence}
Under the assumptions of Theorem~\ref{thm:local_geometry}, 
there exists a constant step size $\alpha>0$ such that, for any initialization
$\bm{\gamma}_0 \in \mathcal N(\bm{\tau},\varrho)$, 
the gradient descent iterates
$\{\bm{\gamma}_t\}_{t\ge0}$ converge
linearly to the unique local minimizer $\bm{\gamma}_\star$ within the basin, i.e.,
\begin{equation}
\label{eq:gd_linear_conv}
\mathrm{d}_2(\bm{\gamma}_t,\bm{\gamma}_\star)
\le
\vartheta^t \,
\mathrm{d}_2(\bm{\gamma}_0,\bm{\gamma}_\star),
\quad \forall t \ge 0,
\end{equation}
where 
\[
\vartheta
=
\frac{6r_{\max}^2(\bm X)}{6r_{\max}^2(\bm X)+r_{\min}^2(\bm X)}.
\]
Consequently, $\mathcal N(\bm{\tau},\varrho)$ serves as a basin of attraction for gradient descent.
\end{theorem}
The proof of Theorem~\ref{thm:gd_convergence} is provided in Section~\ref{subsec:proof_gd}.
\noindent 
Theorem~\ref{thm:gd_convergence} shows that, once initialized within the basin of attraction, gradient descent converges at a linear rate to the local minimizer. This behavior is governed by the local strong convexity and smoothness of the objective established in Theorem~\ref{thm:local_geometry}.

\section{Numerical Results}\label{sec3 numerical result}
This section presents a series of numerical simulations designed to 
corroborate the theoretical framework presented in the preceding 
sections. The code to reproduce those experiments is available at \url{https://github.com/Meghna2608/varprosd-spike-deconvolution}. We demonstrate the effectiveness of our analyses for 
representative PSFs including the Dirac delta, the Gaussian kernel, 
and a Morlet-type wavelet. The Gaussian PSF is given by
$g(t) = \frac{1}{\sqrt{2\pi}\sigma} 
\exp\left(-\frac{t^2}{2\sigma^2}\right)$,
where $\sigma$ controls the effective width.

Throughout this section, the number of Fourier measurements is 
fixed at $N = 51$, and the sampling bandwidth $B = N/T$ is varied 
by adjusting the sampling stepsize $1/T$. Unless otherwise stated, 
we set the number of spikes to $K=3$ and the number of snapshots to 
$L=20$. The spike amplitudes are randomly selected, while the spike 
locations are chosen to ensure a minimum separation of at least 
$\Delta = 0.1$. In most experiments, we report the median 
$\ell_2$ torus distance $\mathrm{d}_2(\bm{\tau}, \bm{\gamma})$ 
defined in~\eqref{eq:d2_torus} over $25$ Monte Carlo realizations; 
however, we use the mean error when comparing against the 
Cram\'er--Rao bound (CRB) for consistency. Unless otherwise 
specified, the noise variance is set to $\sigma_{\bm{Z}}^2 = 0.025$. Under this configuration, assuming a Gaussian PSF 
with $\sigma = 0.1$ and bandwidth $B = 11$, the expected 
$\mathrm{SNR}$ is approximately $15$\,dB. 

In our theoretical analysis we focus on gradient descent (GD), while in practice the variable projection problem is typically solved using the Gauss--Newton (GN) method~\cite{golub1973differentiation,espanol2025local}. Accordingly, we use gradient descent in the experiments that corroborate our theoretical findings and select the optimal sampling bandwidth according to the PSF characteristics. We additionally demonstrate the performance of Gauss--Newton in the experiments that assess estimation performance once the iterates lie within the basin, each initialized from modified ESPRIT unless stated otherwise. Any parameter values that deviate from these defaults are stated explicitly in the corresponding subsection.

\subsection{Initialization methods and sensitivity to bandwidth}
The first experiment investigates various initialization methods. 
We focus on practical methods derived from variants of the ESPRIT algorithm.
Prior work has shown that the original ESPRIT 
algorithm~\cite{roy1989esprit} remains applicable in the presence of 
a PSF~\cite{bresler1989resolution}. Alternatively, ESPRIT can be 
applied after compensating for the PSF by filtering the Fourier 
measurements with the inverse filter $\bm{G}^{-1}$, referred to as 
``equalized ESPRIT''~\cite{yang2001channel}. Furthermore, the core 
formulation of ESPRIT can be modified to directly incorporate the 
known PSF structure, known as ``modified 
ESPRIT''~\cite{swindlehurst1999methods}.

Figure~\ref{fig:initialization_vs_B2} shows the mean 
$\mathrm{d}_2(\bm{\tau}, \bm{\gamma})$ error for these methods as a 
function of the sampling bandwidth $B = N/T$, under two scenarios: a 
Dirac PSF and a Gaussian pulse with $\sigma = 0.1$. For reference, 
we also include the CRB curve, which represents the fundamental 
limit of estimation accuracy.
For the Dirac PSF, all three ESPRIT variants exhibit similar 
behavior: the error decreases nearly monotonically with $B$, 
indicating that higher sampling bandwidth consistently improves 
performance in this case.
For the Gaussian PSF, however, the behavior is more nuanced. 
Equalized ESPRIT fails early, with error diverging as $B$ increases 
due to noise amplification from inverting the PSF at higher 
frequencies. Standard ESPRIT performs better, showing an initial 
decrease in error that eventually plateaus as $B$ continues to grow. 
In contrast, modified ESPRIT consistently achieves the lowest error 
across most of the range: its error decreases with $B$, reaches a 
minimum, and then gradually increases at larger $B$ values as 
additional noisy frequency bins are incorporated.

\begin{figure}[t]
  \centering
  \begin{subfigure}[b]{0.49\textwidth}
    \includegraphics[width=\textwidth]
    {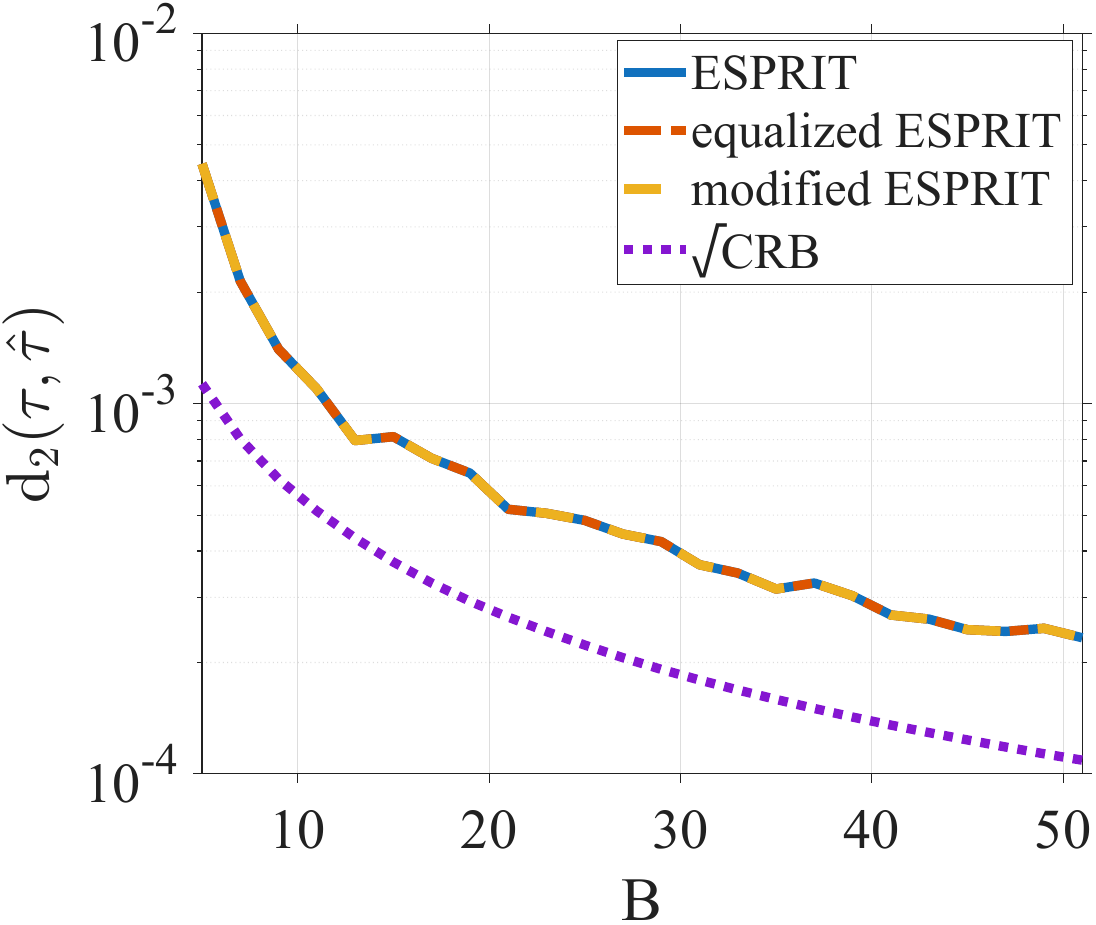}
    \caption{Dirac PSF}
  \end{subfigure}
  \hfill
  \begin{subfigure}[b]{0.49\textwidth}
    \includegraphics[width=\textwidth]
    {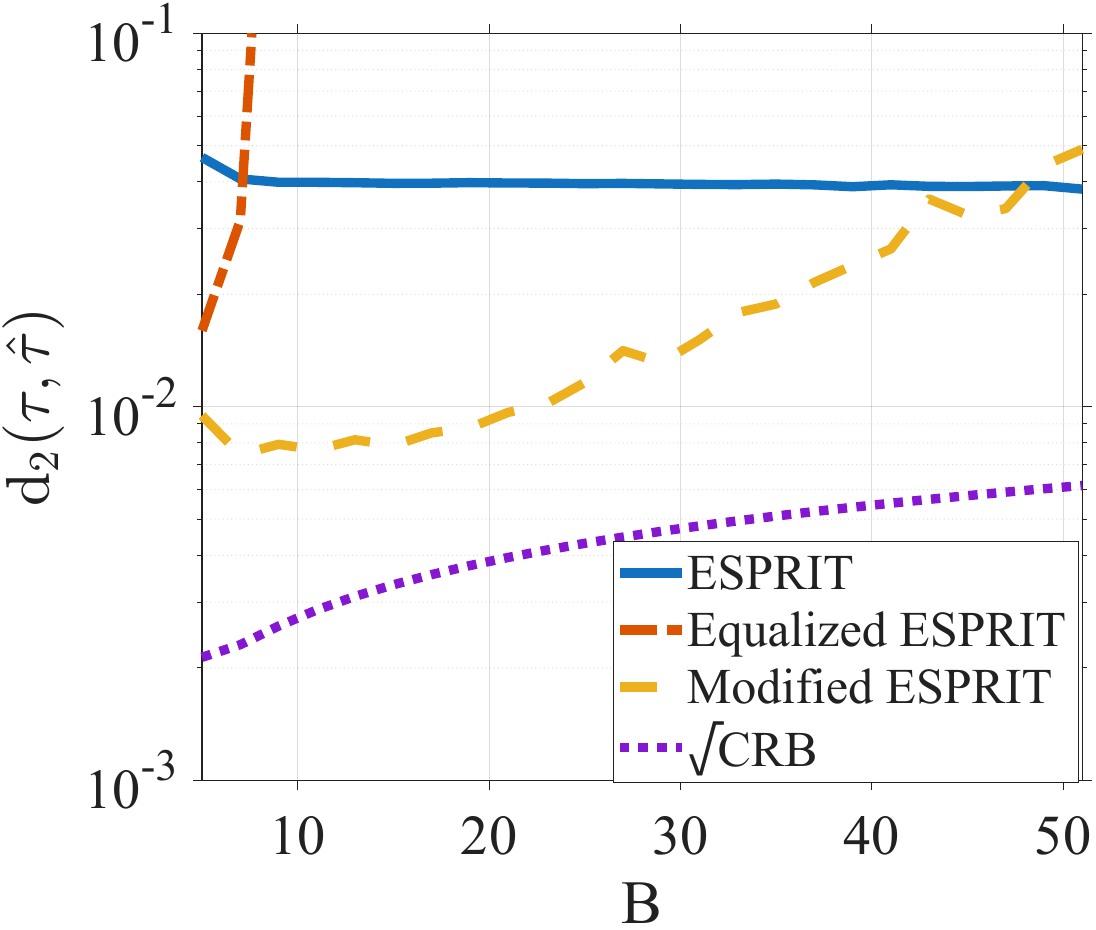}
    \caption{Gaussian PSF ($\sigma = 0.1$)}
  \end{subfigure}
  \caption{Performance of ESPRIT initialization methods and CRB across 
varying sampling bandwidth $B = N/T$ for two PSFs. The number of 
Fourier measurements $N$ is fixed; $B$ is varied by adjusting the 
sampling stepsize $1/T$.}
  \label{fig:initialization_vs_B2}
\end{figure}

These results highlight two key insights. First, while all ESPRIT 
variants perform similarly for the Dirac PSF, modified ESPRIT offers 
robust initialization under more realistic Gaussian PSFs. However, even modified ESPRIT (the best-performing initialization) exhibits a persistent gap to the $\sqrt{\mathrm{CRB}}$, as seen in 
Figure~\ref{fig:initialization_vs_B2}(b). This gap motivates the 
iterative refinement via gradient descent used in subsequent sections. Second, performance is highly sensitive to the choice of sampling bandwidth 
$B$. This implies that selecting an appropriate $B$ is critical. So, in the next subsection, we present a 
principled approach for selecting $B$ based on the characteristics 
of the underlying PSF.

\subsection{Optimal bandwidth selection via PSF characteristics}

Having established that modified ESPRIT provides the most effective initialization, we next turn to the selection of the sampling bandwidth $B$. 
Increasing the bandwidth provides access to more frequency information, but it can also introduce additional noise through high-frequency bins. To address this 
trade-off, we leverage the PSF characteristics introduced in 
Section~\ref{sec:main_results} to identify an optimal value for $B$.

Recall from Section~\ref{sec:main_results} that the basin radius 
$\varrho$ in~\eqref{eq:basin_radius_cond} depends on the PSF 
characteristic $\rho$, defined 
in~\eqref{eq:rho_def_max}, which captures the flatness of the 
PSF's power spectral density and its derivatives.
Smaller values of $\rho$
correspond to a larger basin of convexity, enabling successful 
recovery when the spikes are separated enough, as described 
in~\eqref{eq:basin_radius_cond}. It also depends on the 
sampling bandwidth $B=N/T$, and its behavior as a function of $B$ is 
illustrated in Figure~\ref{fig:2d_psf_charac}.

\begin{figure}[t]
  \centering
  \includegraphics[width=0.55\textwidth]{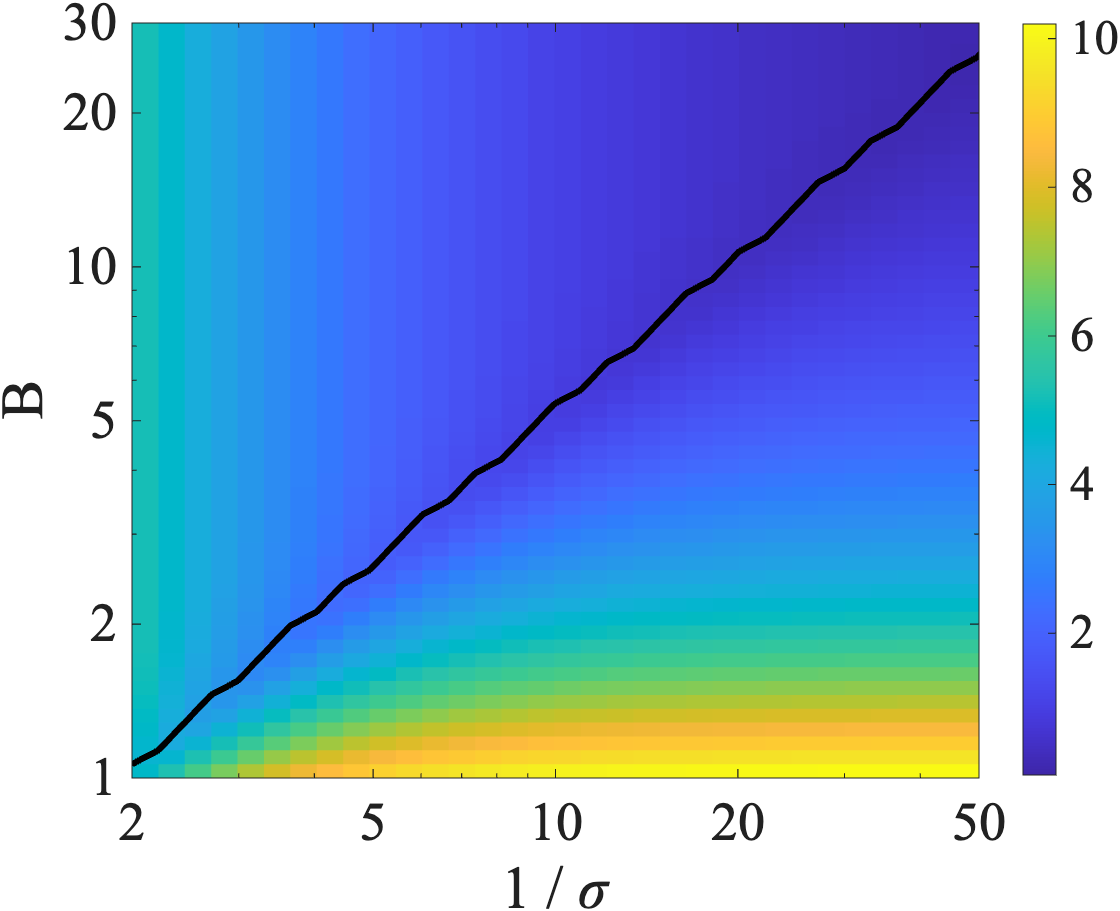}
  \caption{PSF characteristic $\rho$ plotted as a function of sampling bandwidth $B = N/T$ and inverse pulse width $1/\sigma$ on a log-log scale. $N$ is fixed; $B$ is varied by adjusting $T$.
The black curve indicates the optimal bandwidth $B_{\mathrm{opt}}$ that minimizes $\rho$ for each
$1/\sigma$.}
  \label{fig:2d_psf_charac}
\end{figure}

Figure~\ref{fig:2d_psf_charac} plots $\rho$ as a function of $B$ 
and inverse pulse width $1/\sigma$, with the black curve marking 
the bandwidth $B_{\mathrm{opt}}$ that minimizes $\rho$ for each 
$1/\sigma$. We next verify that this $\rho$-based criterion 
predicts the empirically optimal bandwidth.

\begin{figure}[t]
  \centering
  \begin{subfigure}[b]{0.49\textwidth}
    \includegraphics[width=\textwidth]{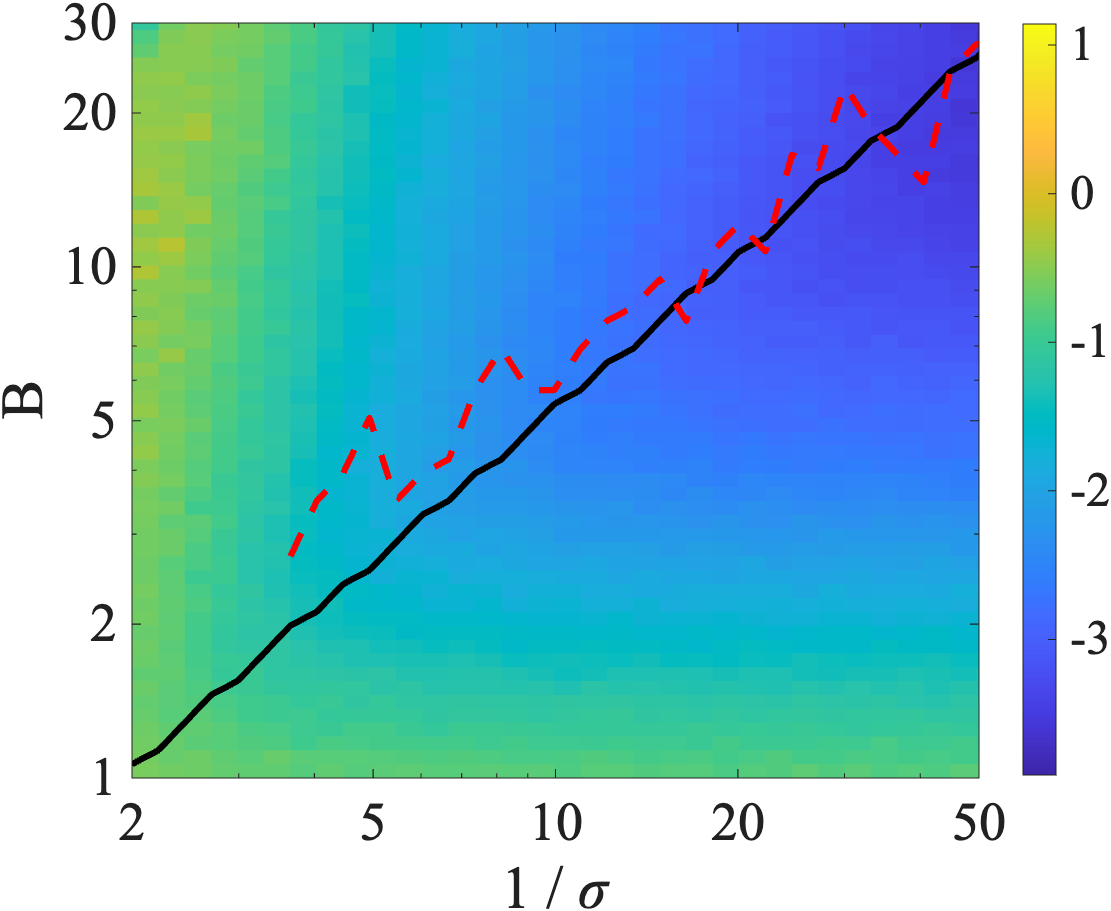}
    \caption{Oracle GD}
  \end{subfigure}
  \hfill
  \begin{subfigure}[b]{0.49\textwidth}
    \includegraphics[width=\textwidth]{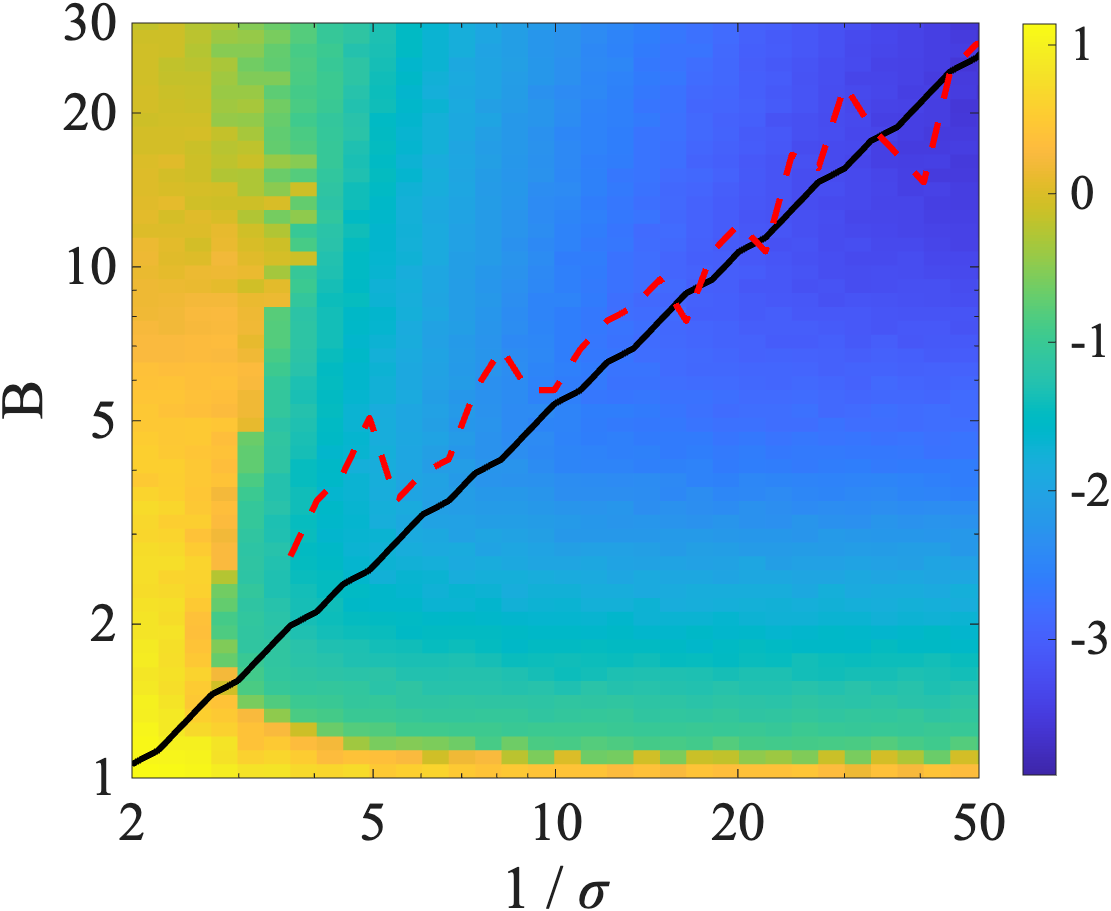}
    \caption{Modified ESPRIT + GD}
  \end{subfigure}
  \caption{Median $\log_{10}(\mathrm{d}_2)$ error of gradient descent 
  as a function of $B$ and $1/\sigma$, under oracle initialization (a) 
  and modified ESPRIT initialization (b). The black curve indicates 
  $B_{\mathrm{opt}}$ from $\rho$ (as in 
  Figure~\ref{fig:2d_psf_charac}); the red dashed curve indicates the 
  empirically optimal $B$ minimizing the $\mathrm{d}_2$ error.}
  \label{fig:2D_phase_plot}
\end{figure}

In Figure~\ref{fig:2D_phase_plot}, the transition boundaries in 
the error plots for oracle GD (where the algorithm is 
initialized at the ground-truth locations $\bm{\tau}$) and 
modified ESPRIT + GD align closely with the $B_{\mathrm{opt}}$ 
(black) curve derived from the $\rho$ criterion. The red dashed 
curve, which shows the bandwidth that minimizes the empirical 
$\mathrm{d}_2$ error across Monte Carlo trials, lies in close 
agreement with the theoretical prediction. Performance improves 
with $B$ up to some optimal bandwidth, after which it either 
saturates or degrades. For large $\sigma$ and large $B$, oracle GD 
continues to succeed because it is initialized at the ground truth 
and always remains within the basin of attraction, but modified 
ESPRIT fails: its initialization lies outside the shrinking basin 
of attraction, leading to the prominent yellow failure bands in 
the top-left region of the modified ESPRIT panel. Our theoretical 
analysis is most relevant in the ``blue'' region of the figures, 
corresponding to smaller $\sigma$ values; at larger $\sigma$ and 
large $B$, our theory does not fully explain the behavior.

Next, we broaden our experiment beyond the Gaussian PSF to include 
another pulse shape that illustrates a different spectral decay 
behavior: a Morlet-type wavelet. The frequency response for the 
Morlet-type wavelet is given as
\[
G_{\text{morlet}}(f) = 25\!\left(\frac{1}{4\pi(f-5\pi)^2+1} 
+ \frac{1}{4\pi(f+5\pi)^2+1}\right).
\]

\begin{figure}[t]
  \centering
  \begin{subfigure}[t]{0.49\textwidth}
    \includegraphics[width=\textwidth]{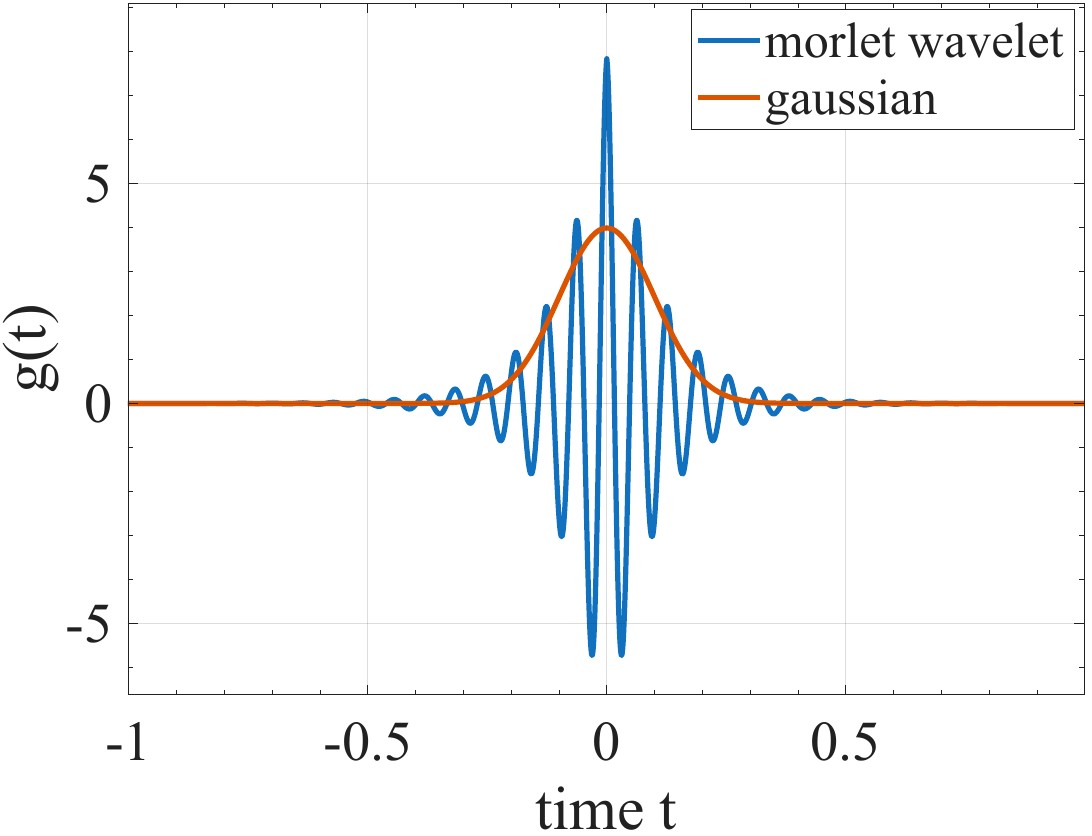}
    \caption{Time-domain shapes $g(t)$}
  \end{subfigure}
  \hfill
  \begin{subfigure}[t]{0.49\textwidth}
    \includegraphics[width=\textwidth]{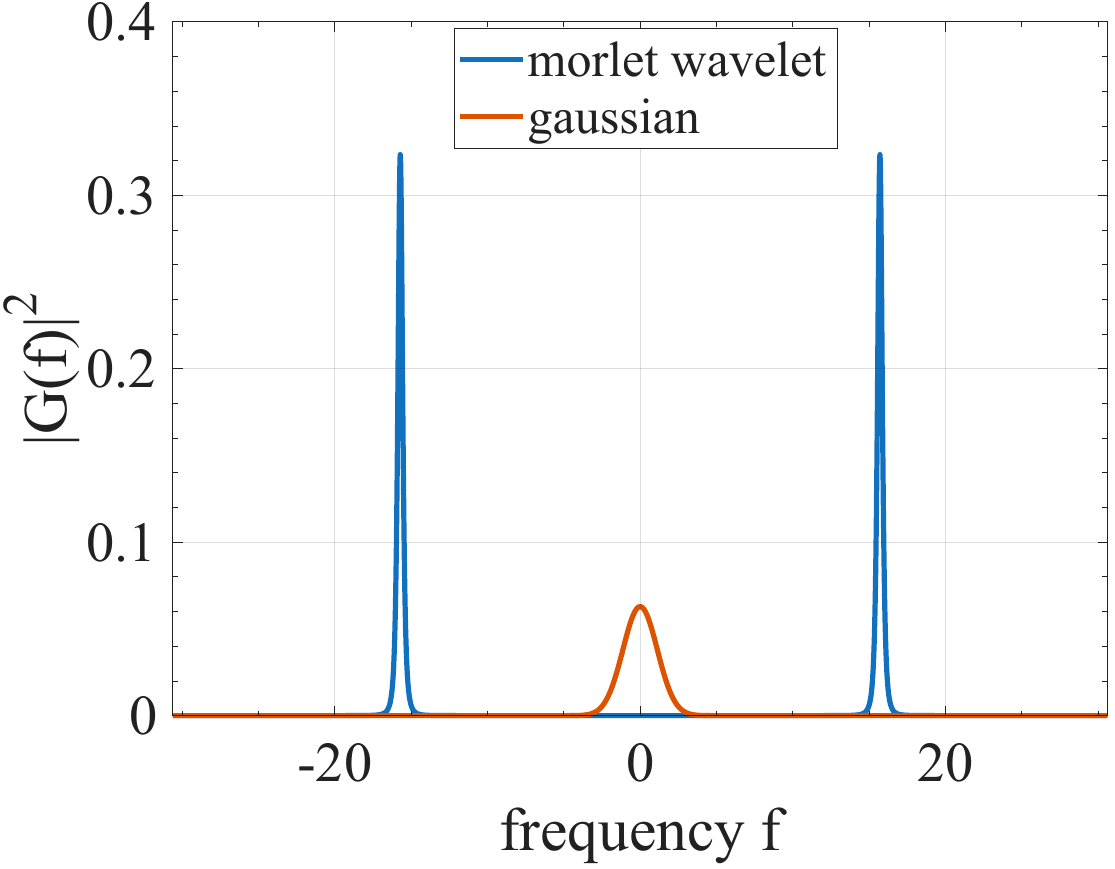}
    \caption{Power spectral densities $|G(f)|^2$}
  \end{subfigure}
  \caption{Illustration of the Gaussian and Morlet PSFs in the time domain (left) and frequency domain (right).}
  \label{fig:psf_charac}
\end{figure}

Figure~\ref{fig:psf_charac} shows the time-domain shapes and 
power spectral densities of the Gaussian and Morlet PSFs for 
illustration purposes. The Gaussian pulse exhibits a rapidly decaying PSD, with most of its energy concentrated in the main lobe near zero frequency. In contrast, the Morlet wavelet has a 
more oscillatory time-domain shape and a PSD consisting of two 
Lorentzian-like lobes centered at $\pm 5\pi \approx \pm 15.7$, 
which also decay sharply but away from the origin.

\begin{figure}[t]
  \centering
  \begin{subfigure}[t]{0.48\textwidth}
    \includegraphics[width=\textwidth]{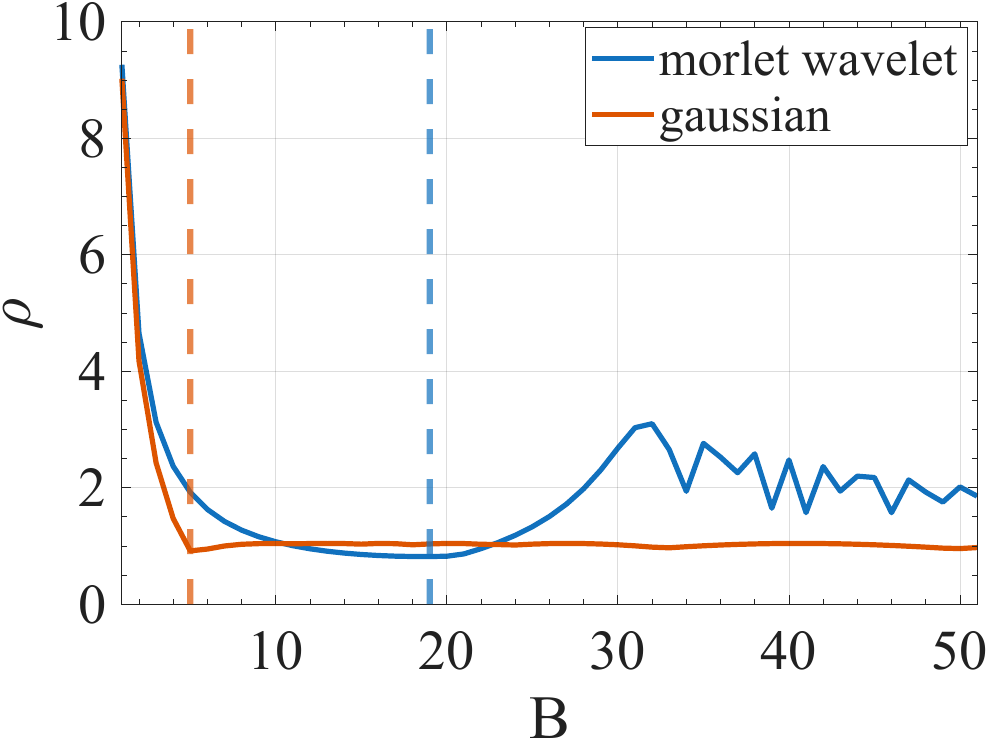}
    \caption{$\rho$ vs.\ sampling bandwidth $B$}
  \end{subfigure}

  \vspace{0.5em}

  \begin{subfigure}[t]{0.49\textwidth}
    \includegraphics[width=\textwidth]{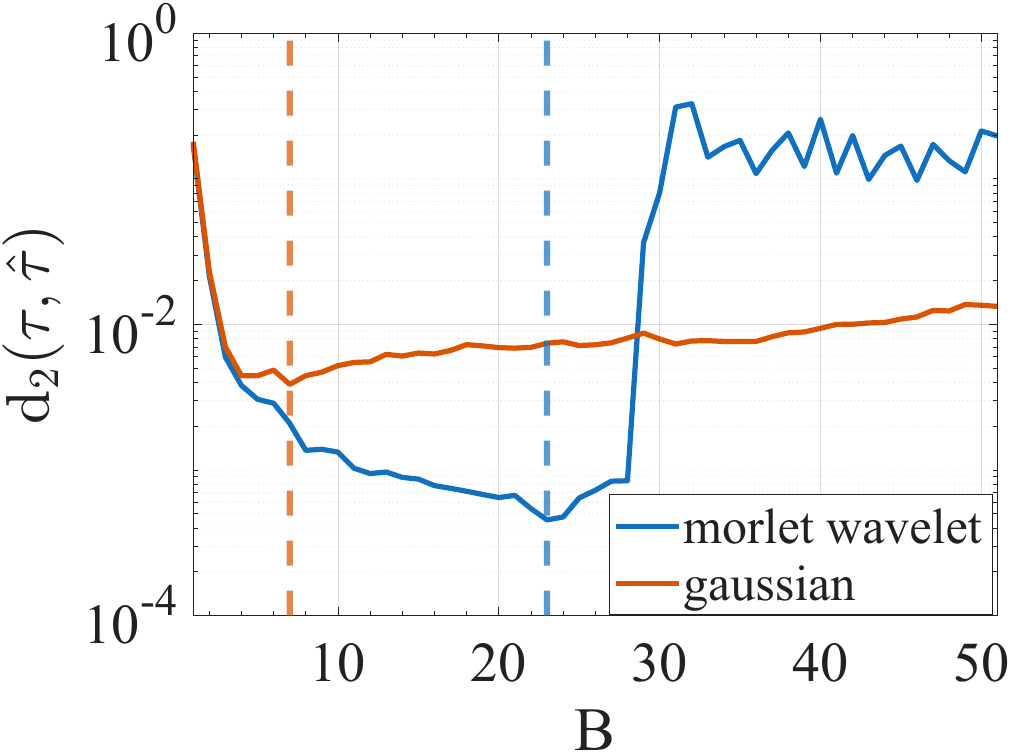}
    \caption{Modified ESPRIT + GD}
  \end{subfigure}
  \caption{Behavior of $\rho$ and estimation error as functions of sampling 
  bandwidth $B$ for Gaussian and Morlet PSFs. 
  (a)~$\rho$ vs.\ $B$, with vertical dashed lines indicating 
  $B_{\mathrm{opt}}$ minimizing $\rho$ for each PSF. (b)~Median 
  $\mathrm{d}_2(\bm{\tau}, \hat{\bm{\tau}})$ error versus $B$ for 
  modified ESPRIT + GD, with vertical dashed lines indicating the 
  empirically optimal $B$ minimizing the error for each PSF.}
  \label{fig:psf_cgd_gn}
\end{figure}

These spectral characteristics directly explain the trends in 
$\rho$ shown in Figure~\ref{fig:psf_cgd_gn}(a). To capture the 
majority of the signal energy, the sampling bandwidth $B$ must 
be at least large enough to cover the main lobes of $|G(f)|^2$ 
for each PSF. This sets the location of the ``sweet spot'' 
bandwidth in $\rho$: for the Gaussian, $\rho$ reaches its 
minimum near $B \approx 4$, matching the edge of its main lobe, 
while for the Morlet, the minimum occurs near $B \approx 19$, 
large enough to capture the lobes at $\pm 5\pi$. Beyond these 
sweet spots, $\rho$ flattens (Gaussian) or rises again (Morlet) 
because additional bandwidth captures only noise without adding 
signal energy. The empirical error curve in 
Figure~\ref{fig:psf_cgd_gn}(b) for modified ESPRIT + GD 
shows the same phenomenon: performance improves up to this 
range of $B$, after which the error starts to increase. Although 
the exact $B_{\mathrm{opt}}$ that minimizes $\rho$ does not 
always coincide with the $B$ that minimizes the median error, the 
two values are closer to each other, confirming that $\rho$ provides a reliable indicator of the near-optimal bandwidth (which is less than the maximum sampling bandwidth).

Finally, we examine how bandwidth choice interacts with the number 
of Fourier samples $N$ and the PSF's effective width $\sigma$. 

\begin{figure}[t]
  \centering
  \begin{subfigure}[t]{0.32\textwidth}
    \includegraphics[width=\textwidth]{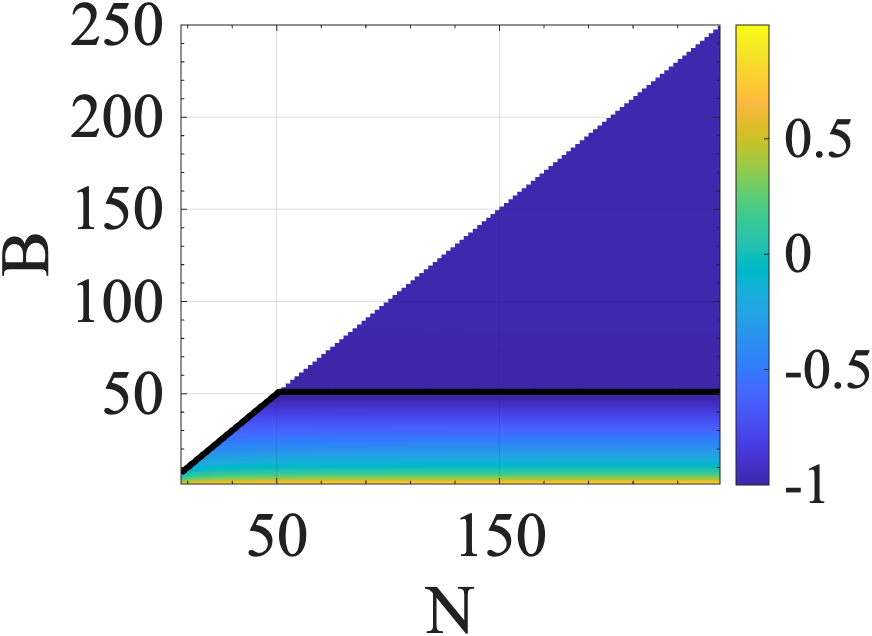}
    \caption{$\sigma = 0.01$}
  \end{subfigure}
  \hfill
  \begin{subfigure}[t]{0.32\textwidth}
    \includegraphics[width=\textwidth]{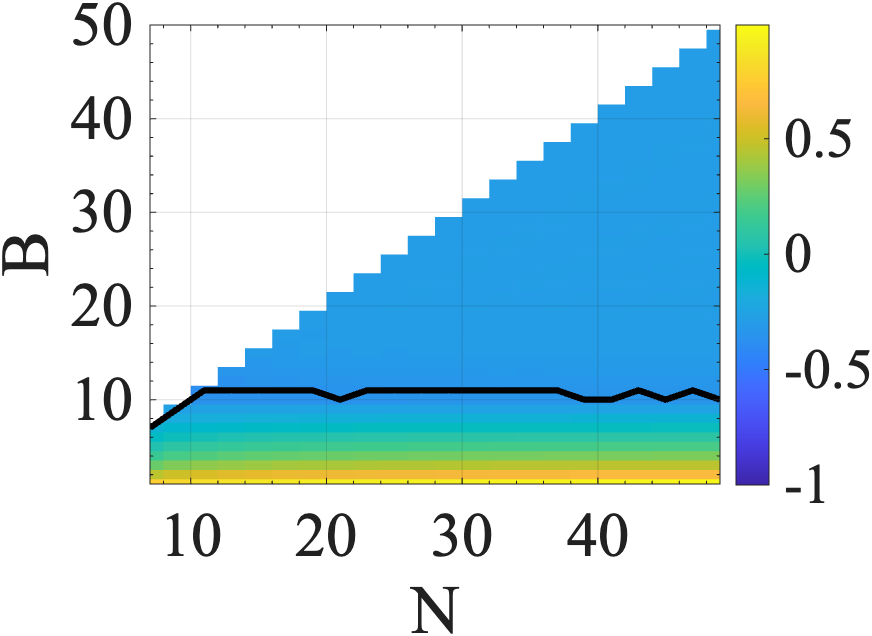}
    \caption{$\sigma = 0.05$}
  \end{subfigure}
  \hfill
  \begin{subfigure}[t]{0.32\textwidth}
    \includegraphics[width=\textwidth]{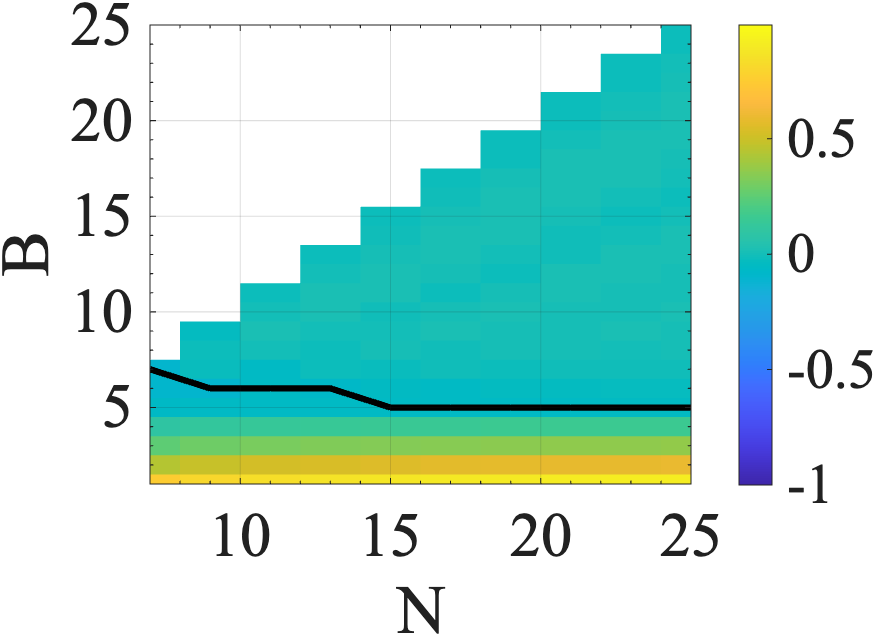}
    \caption{$\sigma = 0.10$}
  \end{subfigure}
  \caption{Parameter $\rho$ as a function of $B$ and $N$ for three Gaussian 
  PSF widths $\sigma$. Colormaps show $\log_{10}\rho$; the black 
  curve indicates $B_{\mathrm{opt}}$ minimizing $\rho$ for each $N$.}
  \label{fig:rho2}
\end{figure}

Figure~\ref{fig:rho2} shows the parameter $\rho$ as a function of $B$ and $N$ for three representative values of $\sigma$, where $B$ ranges from $1$ to $N$ such that the effective period $T = N/B$ always exceeds unity. For narrow pulses ($\sigma = 0.01$), $B_{\mathrm{opt}}$ grows with $N$ initially but quickly saturates 
near $B \approx 50$, beyond which additional frequency bins provide negligible benefit. For $\sigma = 0.05$, the saturation occurs earlier around $B \approx 15$, reflecting the fact that a broader kernel in time corresponds to a more localized spectrum, so fewer bins suffice. For the widest PSF considered ($\sigma = 0.1$), the frequency response is sharply localized, 
and the saturation occurs much earlier, around $B \approx 5$.

\begin{figure}[t]
  \centering
  \begin{subfigure}[t]{0.32\textwidth}
    \includegraphics[width=\textwidth]{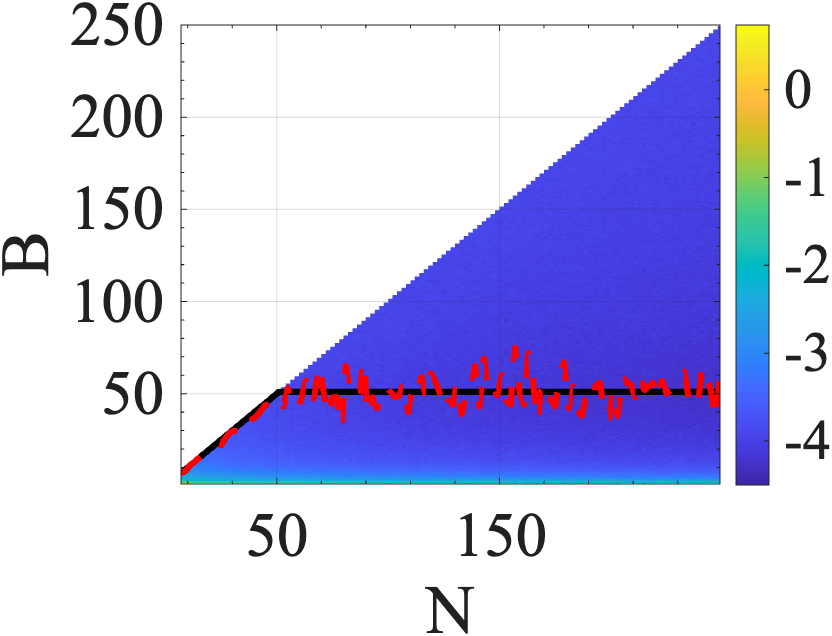}
    \caption{GD, $\sigma = 0.01$}
  \end{subfigure}
  \hfill
  \begin{subfigure}[t]{0.32\textwidth}
    \includegraphics[width=\textwidth]{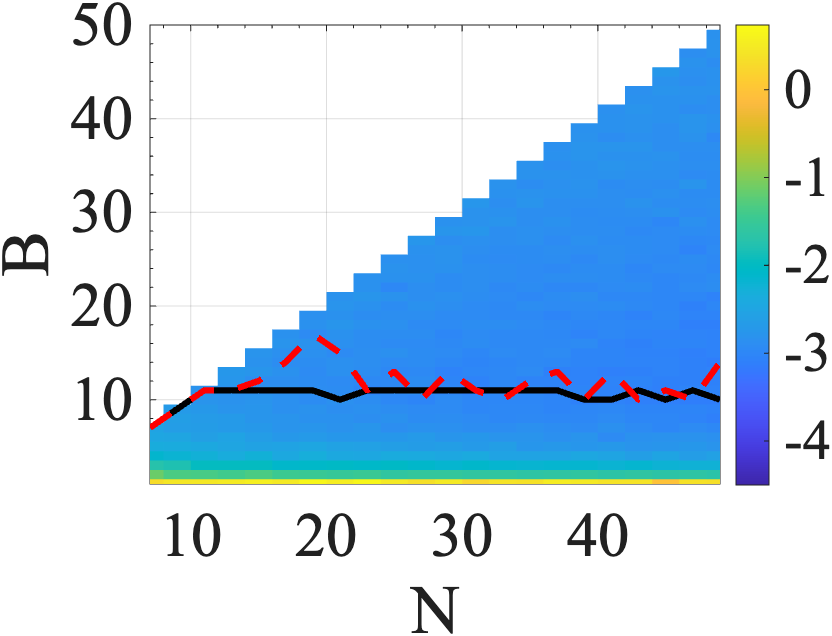}
    \caption{GD, $\sigma = 0.05$}
  \end{subfigure}
  \hfill
  \begin{subfigure}[t]{0.32\textwidth}
    \includegraphics[width=\textwidth]{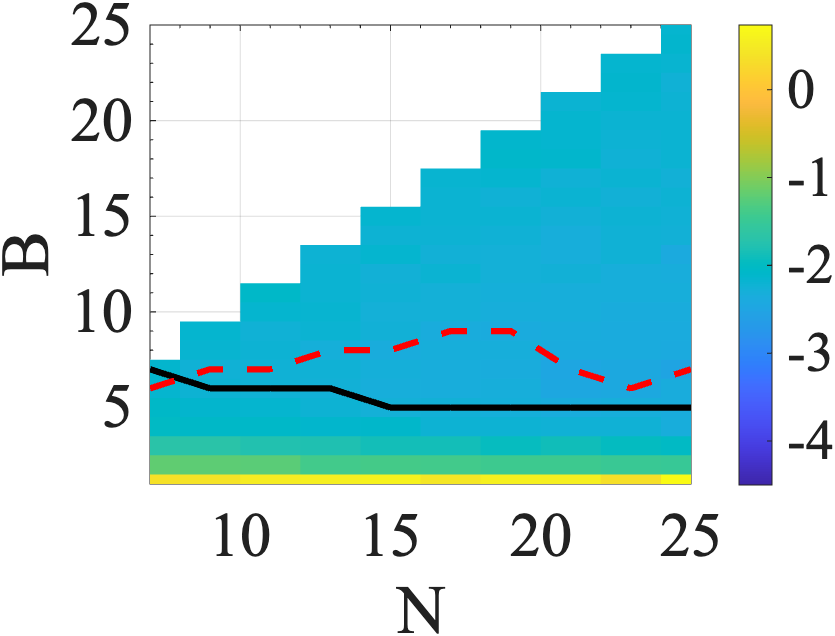}
    \caption{GD, $\sigma = 0.10$}
  \end{subfigure}
  \caption{Performance of Mod-ESPRIT + GD versus $B$ and $N$ for three 
  PSF widths. Colormaps show $\log_{10}$ of the median $\mathrm{d}_2$ 
  error. The black curve is $B_{\mathrm{opt}}$ predicted from $\rho$, 
  and the red dashed curve is the empirically optimal $B$ that 
  minimizes the median error.}
  \label{fig:erro2}
\end{figure}

The error plots in Figure~\ref{fig:erro2} for modified ESPRIT + GD mirror these trends. The accuracy 
improves as $B$ increases up to the predicted threshold (which 
depends on $\sigma$), but once $B$ is fixed, further increasing 
$N$ yields only marginal gains. The empirically optimal $B$ 
(red dashed curve) lies in close agreement with the 
$B_{\mathrm{opt}}$ predicted from $\rho$ (black curve) across all 
$\sigma$ values. These results show that 
bandwidth $B$, rather than the number of Fourier samples $N$, is 
the primary driver of resolution and accuracy, and that the scale 
at which $B$ saturates is dictated by the PSF width $\sigma$.

\subsection{Error scaling with key parameters}

In the next series of experiments, we evaluate how the performance 
of our algorithm scales with key parameters like $\mathrm{SNR}$ and 
number of snapshots $L$. For each setting, we first choose the 
sampling bandwidth $B$ and then vary the parameter of interest. In this subsection, we also present the performance of the Gauss--Newton method, as it is the method of choice for variable 
projection in practice.  

\begin{figure}[t]
  \centering
  \begin{subfigure}[b]{0.49\textwidth}
    \includegraphics[width=\textwidth]{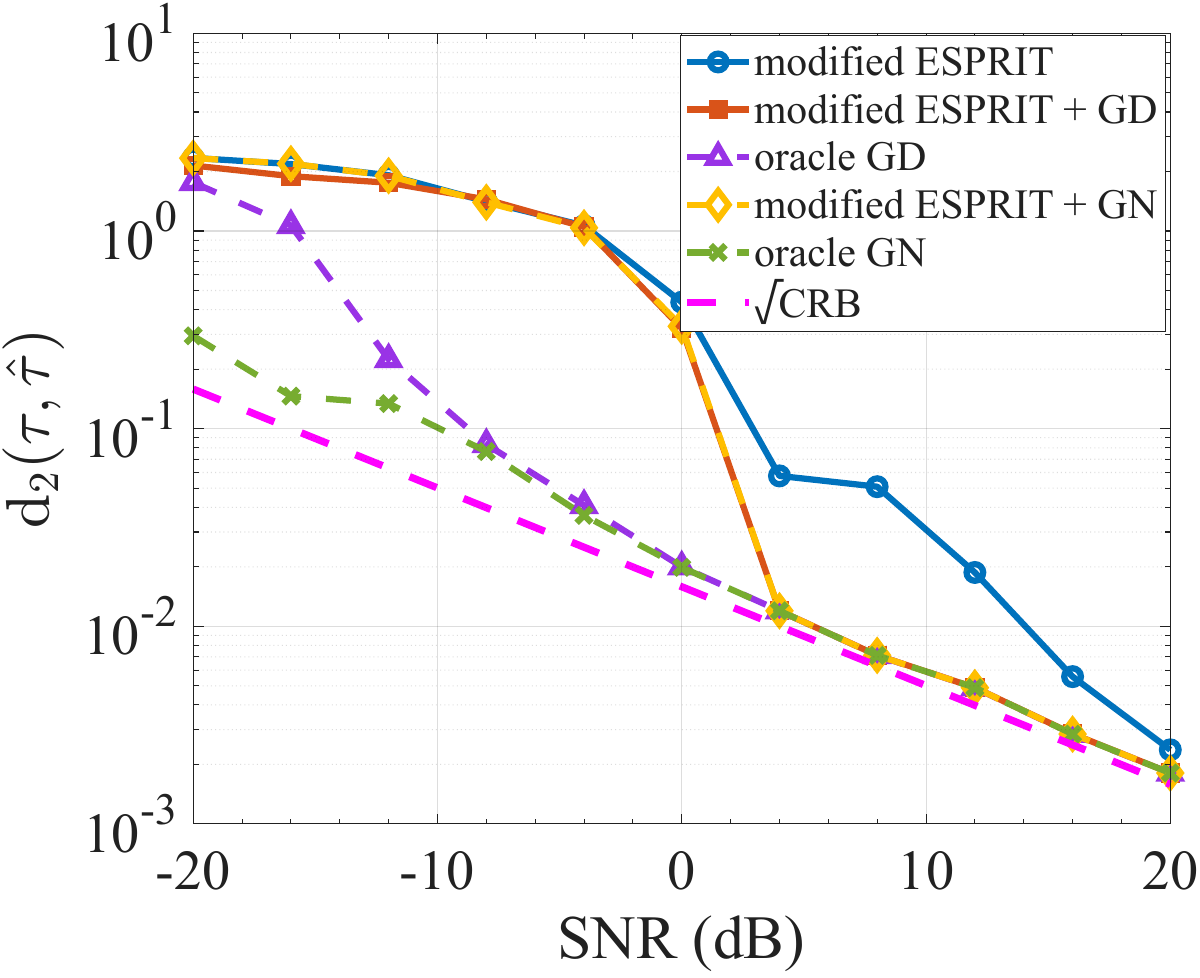}
  \end{subfigure}
  \hfill
  \begin{subfigure}[b]{0.49\textwidth}
    \includegraphics[width=\textwidth]{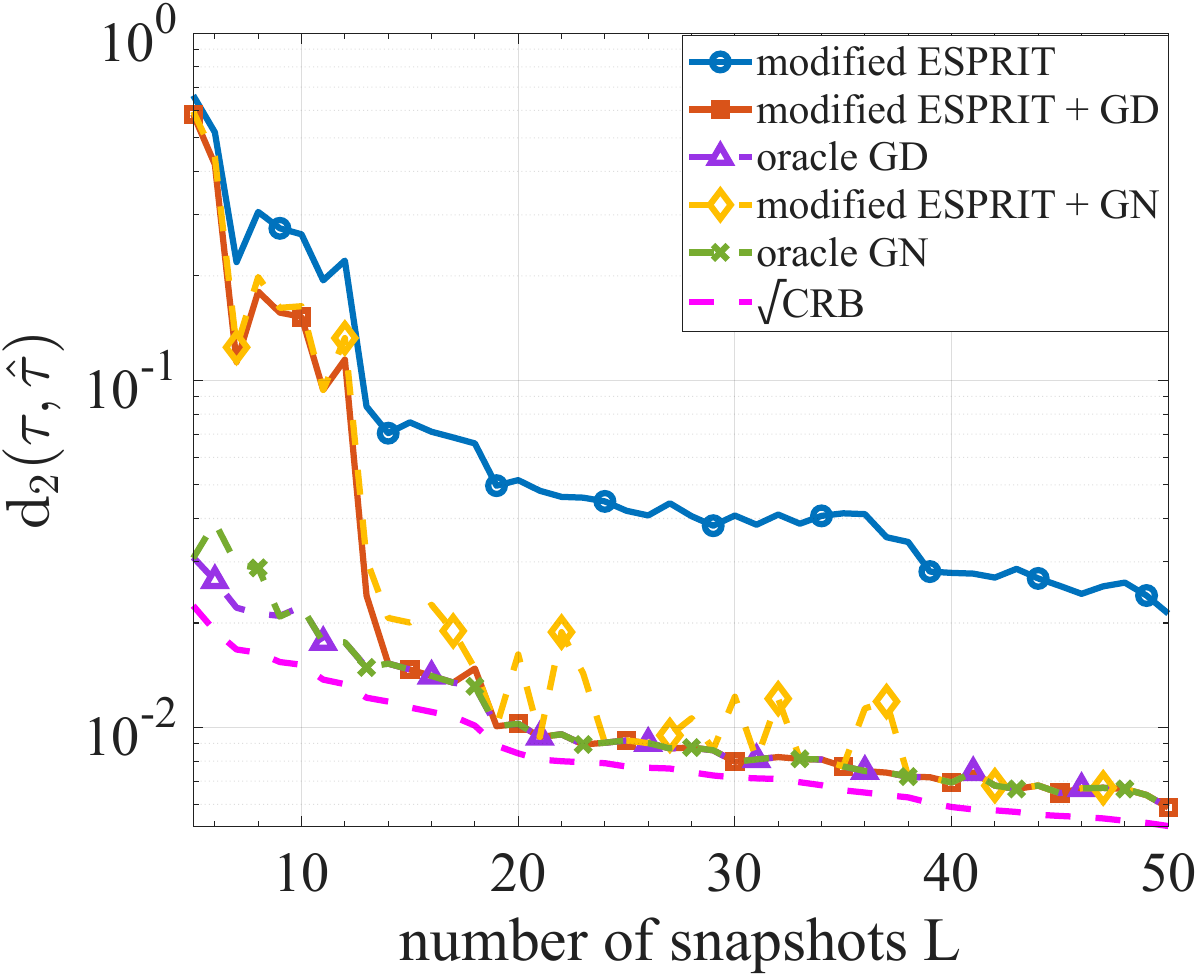}
  \end{subfigure}
  \caption{Performance of various methods with varying (a)~$\mathrm{SNR}$ and 
(b)~number of snapshots $L$. For (b), the noise variance is set to $\sigma_{\bm Z}^2 = 0.25$, corresponding to approximately $5$\,dB 
SNR.}
  \label{fig:error_performance_crb2}
\end{figure}

Figure~\ref{fig:error_performance_crb2}(a) shows how the mean 
recovery error varies with $\mathrm{SNR}$. We use the mean of 
Monte Carlo trials to ensure consistency with the CRB, which is 
defined in terms of mean-squared error. Both gradient descent and 
Gauss--Newton substantially improve upon the initialization. At 
very low $\mathrm{SNR}$ ($< 5$\,dB), however, both modified ESPRIT 
+ GD/GN and oracle GD/GN fail to match the CRB, with modified 
ESPRIT diverging earlier due to the sensitivity of its 
initialization under severe noise. For moderate to high 
$\mathrm{SNR}$ ($5$--$20$\,dB), the performance of modified ESPRIT 
+ GD/GN nearly coincides with oracle GD/GN and closely approaches 
the CRB.

Figure~\ref{fig:error_performance_crb2}(b) illustrates how the 
error scales with the number of snapshots $L$ in the presence of 
random Gaussian noise. For this experiment, we use a higher noise 
variance $\sigma_{\bm{Z}}^2 = 0.25$, corresponding to an 
$\mathrm{SNR}$ of approximately $5$\,dB. As predicted by our theoretical results, the error from GD decreases proportionally to $1/\sqrt{L}$. Gauss--Newton shows similar overall scaling but exhibits occasional spikes for modified ESPRIT + GN due to the initialization sensitivity. 

\subsection{Performance under random and adversarial noise}

In the final experiment, we investigate how the performance of 
oracle and modified ESPRIT initialized refinement varies as a 
function of $\mathrm{SNR}$ under two different noise models, using 
both GD and GN. The first noise model is random Gaussian noise with 
independent zero-mean entries. The second is an adversarial 
perturbation, constructed from the top right singular vector of the 
Jacobian of the inverse map introduced in 
Section~\ref{sec:adv_noise}. To ensure a fair comparison, the 
adversarial perturbation is rescaled in each trial so that its 
Frobenius norm matches that of the Gaussian noise at the same 
$\mathrm{SNR}$.

\begin{figure}[t]
  \centering
  \begin{subfigure}[b]{0.49\textwidth}
    \includegraphics[width=\textwidth]
    {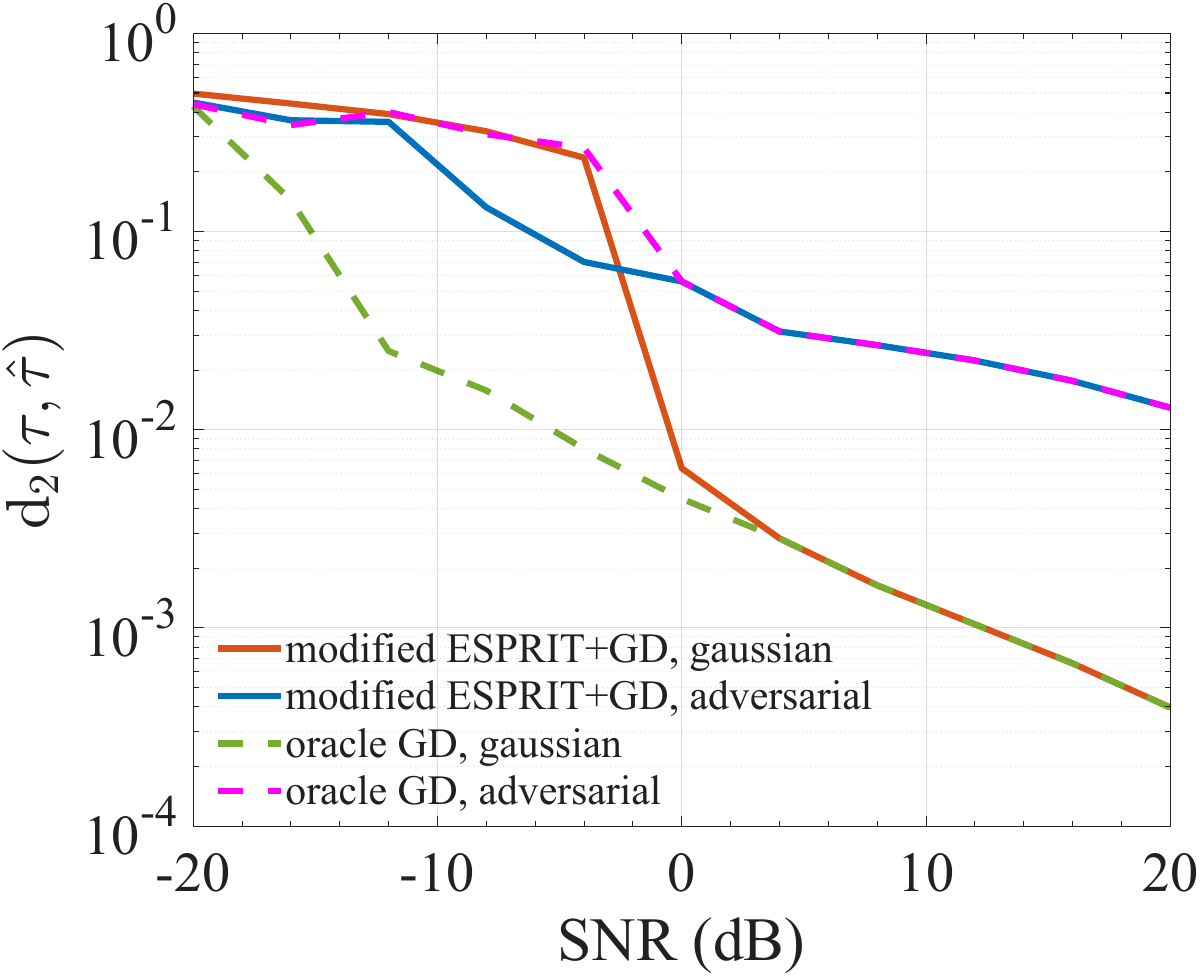}
    \caption{Modified ESPRIT + GD}
  \end{subfigure}
  \hfill
  \begin{subfigure}[b]{0.49\textwidth}
    \includegraphics[width=\textwidth]
    {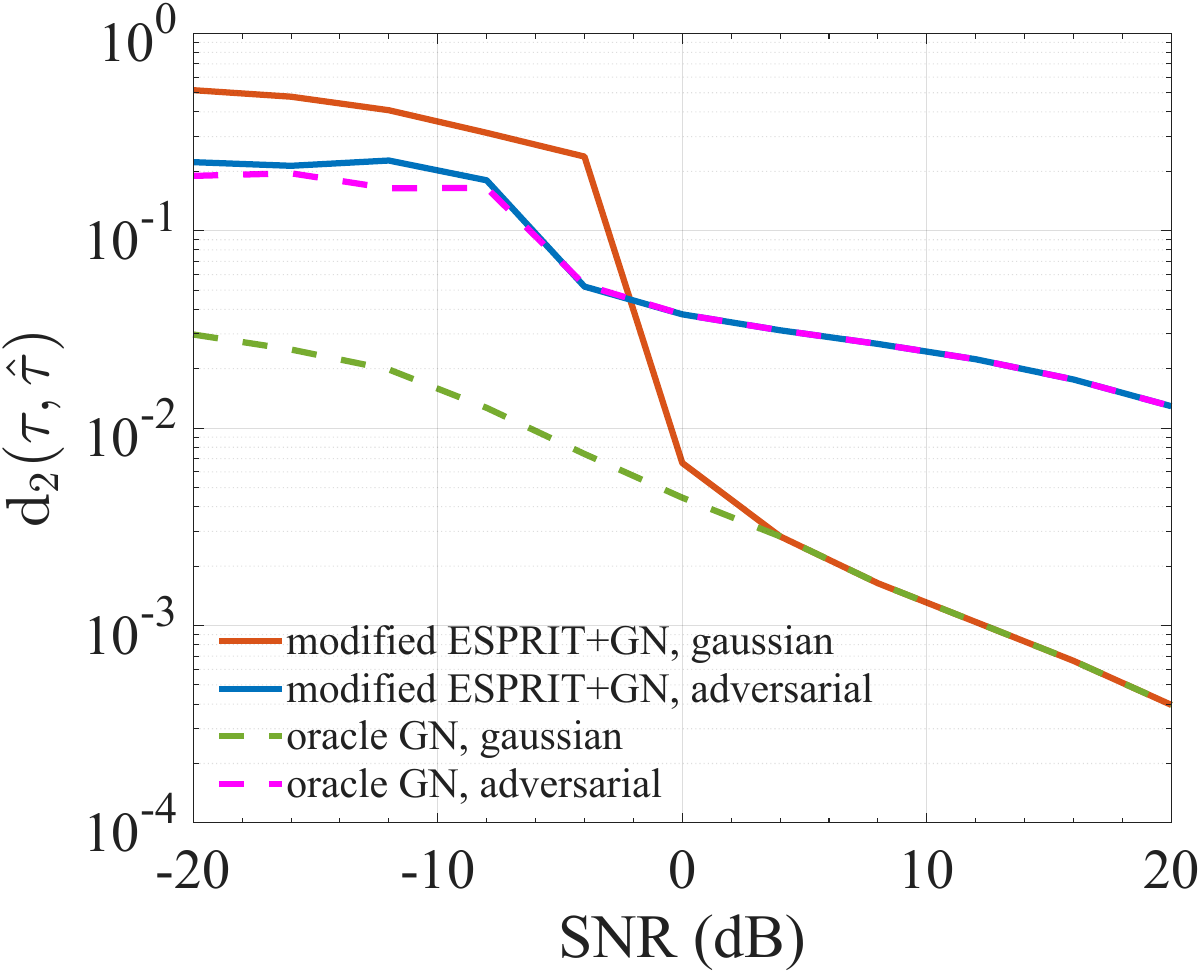}
    \caption{Modified ESPRIT + GN}
  \end{subfigure}
  \caption{Mean $\mathrm{d}_2(\bm{\tau}, \hat{\bm{\tau}})$ error 
  versus $\mathrm{SNR}$ under random Gaussian and adversarial noise 
  for GD (left) and GN (right).}
  \label{fig:error_vs_twonoise_mean}
\end{figure}

Figure~\ref{fig:error_vs_twonoise_mean} shows the mean 
$\mathrm{d}_2$ error as a function of $\mathrm{SNR}$. In the 
moderate-to-high SNR regime (above approximately $5$~dB), a clear 
and consistent gap emerges. The error under adversarial noise is 
significantly larger than under random Gaussian noise at the same 
$\mathrm{SNR}$, confirming that structured perturbations are 
substantially more damaging. Below approximately $0$~dB, the 
adversarial noise is strong enough to cause the estimator to 
collapse, and the behavior in this regime is not captured by our 
analysis. For both GD and GN, modified ESPRIT initialization 
achieves performance comparable to the oracle once the 
$\mathrm{SNR}$ is sufficient for the initialization to lie within 
the basin of attraction.

\section{Key Technical Lemmas about Conditioning of Structured Matrices}
\label{sec4 key lemma}
This section presents the supporting lemmas that bound some of the key quantities used in our analysis. 

\begin{lemma}
\label{lem:bounding_sing_val_phi_gamm_phi_tau}

Let $\Delta$ be the minimum separation in $\bm{\tau}$ defined in \eqref{eq:min_sep}. 
Consider another set $\bm{\gamma}$ such that $\mathfrak{d}_\infty < \Delta/2$ where $ \mathfrak{d}_\infty :=
\mathrm{d}_{\infty}(\bm{\tau},\bm{\gamma})$.
Suppose that the truncated power spectral density $P_g$ is a function in $L_1$ and of bounded variation, and that $f \mapsto 4 \pi^2 f^2 P_g(f)$ is also in $L_1$.
Let $\bm{A}_{\bm{\tau}},\bm{A}_{\bm{\gamma}},E_g,E_{g'},\rho_{g},\rho_{g'}$ be defined from $g$ using \eqref{eq:def_Agamma}, \eqref{eq:def_energy}, \eqref{eq:def_rho_g}.  
Then it holds for all $\mathbf{u}, \mathbf{v} \in \mathbb{C}^K$ that
\begin{subequations}\label{lem:bound_quadratic_term}
    \begin{align}
\label{eq:lemma_quadra_majorant}
    \norm{\bm{A}_{\bm{\gamma}}\mathbf{u} - \bm{A}_{\bm{\tau}} \mathbf{v} }_{2}^2 
    &\leq 
    T\!\cdot\!E_{g} \bigg( 1 + \frac{2\rho_g}{3(\Delta-2\mathfrak{d}_\infty)} \bigg) \| \mathbf{u} - \mathbf{v} \|^2 
    + 
    \mathfrak{d}_\infty^2\,
    T\!\cdot\!E_{g'} \bigg(1 + \frac{ \rho_{g'}}{2(\Delta-2\mathfrak{d}_\infty)}  \bigg)  |\Re (\mathbf{u}^{\mathsf{H}} \mathbf{v})|, \\
\label{eq:lemma_quadra_minorant}
    \norm{\bm{A}_{\bm{\gamma}}\mathbf{u} + \bm{A}_{\bm{\tau}} \mathbf{v} }_{2}^2 & \geq T\!\cdot\!E_{g}  \bigg( 1 - \frac{2\rho_g}{3 (\Delta-2\mathfrak{d}_\infty)} \bigg)\| \mathbf{u} + \mathbf{v} \|^2 -  
    \mathfrak{d}_\infty^2\,
   T\!\cdot\!E_{g'}  \bigg(1+ \frac{ \rho_{g'}}{2(\Delta-2\mathfrak{d}_\infty) }  \bigg) |\Re (\mathbf{u}^{\mathsf{H}} \mathbf{v})|.
\end{align}
\end{subequations}
\end{lemma}

\begin{proof}
Let $\mathbf{u}, \mathbf{v} \in \mathbb{C}^{K}$ be arbitrarily fixed. 
Then, using the definition in \eqref{eq:def_Agamma}, i.e. $\bm{A}_{\bm{\tau}} = \bm{G}\bm{\Phi}_{\bm{\tau}}$ and $ \bm{A}_{\bm{\gamma}} = \bm{G}\bm{\Phi}_{\bm{\gamma}}$, 
the left-hand side of \eqref{eq:lemma_quadra_majorant} is written as
\begin{align}
\norm{\bm{A}_{\bm{\gamma}}\mathbf{u} + \bm{A}_{\bm{\tau}} \mathbf{v} }_{2}^2 
& = \displaystyle \sum_{l=-\frac{N-1}{2}}^{\frac{N-1}{2}} \bigg| \sum_{j=1}^K  \hat{g}\left(\frac{l}{T}\right) u_j e^{i 2 \pi \gamma_j l/T} + \hat{g}\left(\frac{l}{T}\right) v_j e^{i 2 \pi \tau_j l/T} \bigg|^2 \nonumber \\
& = \sum_{l=-\frac{N-1}{2}}^{\frac{N-1}{2}} \left|\hat{g}\left(\frac{l}{T}\right)\right|^2 \bigg| \sum_{j=1}^K  u_j e^{i 2 \pi \gamma_j l/T} +  v_j e^{i 2 \pi \tau_j l/T} \bigg|^2 \nonumber \\ 
& = \sum_{l=-\infty}^\infty P_g\left(\frac{l}{T}\right)  \bigg| \sum_{j=1}^K  u_j e^{i 2 \pi \gamma_j l/T} +  v_j e^{i 2 \pi \tau_j l/T} \bigg|^2. \label{eq:GPhi_tau_sq_step1}
\end{align}
Then we use Lemma~\ref{lem:ferreira_thm_3} to introduce a $(\Delta- 2\delta)-$bandlimited majorant function $C_+$ of $P_g$ such that  $C_+(f) \geq P_g(f)$ for all $f \in \mathbb{R}$. Hence, the right-hand side of \eqref{eq:GPhi_tau_sq_step1} is further upper-bounded to provide 
\begin{align}
\MoveEqLeft[0] \norm{\bm{A}_{\bm{\gamma}}\mathbf{u} +\bm{A}_{\bm{\tau}} \mathbf{v} }_{2}^2 & \\
& \leq  \sum_{l=-\infty}^\infty  C_+\left(\frac{l}{T}\right) \bigg| \sum_{j=1}^K u_j e^{i 2 \pi \gamma_j l/T} + v_j e^{i 2 \pi \tau_j l/T} \bigg|^2  \nonumber \\
& =  \sum_{j,j'=1}^K \sum_{l = -\infty}^{\infty} \left(
 u_j \bar{u}_{j'}   C_+\left(\frac{l}{T}\right)   e^{i 2 \pi (\gamma_j -\gamma_{j'}) l/T}   
+ 2 \Re \left(u_j \bar{v}_{j'}  C_+\left(\frac{l}{T}\right) e^{i 2 \pi (\gamma_j -\tau_{j'}) l/T} \right) \right) \nonumber \\
& + \sum_{j,j'=1}^K \sum_{l = -\infty}^{\infty} v_j \bar{v}_{j'}  C_+\left(\frac{l}{T}\right)   e^{i 2 \pi (\tau_j -\tau_{j'}) l/T}. \label{eq:GPhi_tau_sq_step2}
\end{align}
For fixed $j$ and $j'$, since the Fourier transform of $f \mapsto C_+(f) e^{i 2 \pi (\gamma_j -\gamma_{j'}) f}$ is written as 
$\widehat{C}_+(\xi - \gamma_j +\gamma_{j'})$ where $\widehat{C}_+(\xi) = \int_{-\infty}^\infty C_+(f) e^{-i2\pi \xi f} df$, by the Poisson summation formula, we obtain
\[
\sum_{l = -\infty}^{\infty} C_+\left(\frac{l}{T}\right) e^{i 2 \pi (\gamma_j -\gamma_{j'}) l/T} 
= T\sum_{l = -\infty}^{\infty} \widehat{C}_+(lT - \gamma_j +\gamma_{j'}). 
\]
Plugging in this result to \eqref{eq:GPhi_tau_sq_step2} yields
\begin{multline}
\label{eq:expand_quadratic_form}
\norm{\bm{ A}_{\bm{\gamma}}\mathbf{u} +\bm{A}_{\bm{\tau}} \mathbf{v} }_{2}^2 \\
\leq T \sum_{j,j'=1}^K \sum_{l = -\infty}^{\infty}
\bigg ( u_j \bar{u}_{j'} \widehat{C}_+(lT-\gamma_j +\gamma_{j'}) 
+ 2 \Re \left( u_j \bar{v}_{j'} \right) \widehat{C}_+(lT- \gamma_j +\tau_{j'})  \\
\qquad \qquad  + v_j \bar{v}_{j'} \widehat{C}_+(lT-\tau_j +\tau_{j'}) \bigg) .
\end{multline}
Since the majorant $C_+$ is $(\Delta-2\mathfrak{d}_\infty)$-bandlimited, we can further simplify the upper bound in \eqref{eq:expand_quadratic_form} by dropping most of the summands. First, we consider the separation conditions on $\bm{{\tau}}$ and $\bm{\gamma}$ to ensure that their entries are well-spaced.
Since $d_{\mathbb{T}}$ is a valid metric defined on torus via \eqref{eq:torus_def}, the triangle inequality implies that
$d_{\mathbb{T}}(\gamma_j,\tau_{j'}) \geq  d_{\mathbb{T}}({\tau}_j,\tau_{j'})- d_{\mathbb{T}}(\tau_{j} ,\gamma_j) \geq \Delta - \mathfrak{d}_\infty$  for all $j \neq j'$. Furthermore, this also implies a minimum separation condition on $\hat{\bm{\tau}}$ satisfying $d_{\mathbb{T}}(\gamma_j,\gamma_{j'}) \geq  d_{\mathbb{T}}(\gamma_{j},\tau_{j'}) - d_{\mathbb{T}}(\gamma_{j'} , {\tau}_{j'}) \geq \Delta - 2\mathfrak{d}_\infty$ for all $j \neq j'$. Due to the minimum separation condition, unless $j=j'$ and $l=0$, we have $|lT - \gamma_j + \gamma_{j^\prime}| \geq \Delta-2\mathfrak{d}_\infty$. Hence, since $\widehat{C}_+$ is supported within $[-(\Delta-2\mathfrak{d}_\infty), (\Delta-2\mathfrak{d}_\infty)]$, the terms $\widehat{C}_+(lT - \gamma_j + \gamma_{j^\prime})$ becomes zero unless $j=j'$ and $l=0$. 
The same argument holds for the other two quantities  $\widehat{C}_+(lT - \tau_j + \tau_{j^\prime})$ and $\widehat{C}_+(lT - \gamma_j + \tau_{j^\prime})$ which also vanish when $l\neq0$ or $j \neq j^\prime$. Hence, one can simplify the expression~\eqref{eq:expand_quadratic_form} as follows
\begin{align}
\label{eq:FT_majorant_upperbound}
\MoveEqLeft[0] \norm{\bm{ A}_{\bm{\gamma}}\mathbf{u} + \bm{A}_{\bm{\tau}} \mathbf{v} }_{2}^2 &\nonumber \\
& \leq T \widehat{C}_+(0) \displaystyle \sum_{j=1}^K \bigg( |u_j|^2 + 2 \Re \left( u_j \bar{v}_j \right)  +|v_j|^2 \bigg) + 2 T \sum_{j=1}^K \big( \widehat{C}_+(-\gamma_j+\tau_j) - \widehat{C}_+(0) \big)  \Re \left(u_j \bar{v}_j \right) \nonumber \\
& = T \widehat{C}_+(0) \displaystyle \sum_{j=1}^K  |u_j + v_j|^2
+  2T \sum_{j=1}^K \big( \widehat{C}_+(-\gamma_j+\tau_j) - \widehat{C}_+(0) \big)  \Re \left(u_j \bar{v}_j \right).
\end{align}
Next, we further upper-bound the right-hand side of \eqref{eq:FT_majorant_upperbound} by upper-bounding each of the summands separately. We first utilize the results from Lemma~\ref{lem:ferreira_thm_3} to bound $\widehat{C}_+(0)$ in \eqref{eq:FT_majorant_upperbound} as 
\begin{align*}
\widehat{C}_+(0) &= E_g + \left(\widehat{C}_+(0) - E_g\right) \nonumber \\
 &= E_g + \int_{-\infty}^{\infty} \left(C_+(f) - P_g(f) \right) df \nonumber \\
 &\leq   E_{g} + \frac{2}{3} V(P_g)(\Delta-2\mathfrak{d}_\infty)^{-1}.
\end{align*}
We now proceed to bound the term $ \widehat{C}_+(-\gamma_j+\tau_j) - \widehat{C}_+(0)$ in \eqref{eq:FT_majorant_upperbound} by using the Lemma~\ref{lem:ferreira_thm_3}. The construction of the majorant $C_+$ admits that it is even function and its Fourier transform is twice-differentiable. Furthermore, its first derivative $\widehat{C}_+'(0) = -i 2\pi \int_{-\infty}^\infty f C_+(f) df = 0$ since $f C_+(f)$ is an odd function whose integral over any symmetric interval is $0$.
Furthermore, the second order Taylor series expansion of $\widehat{C}_+$ at $0$ ensure the existence of $\alpha \in [0, |\gamma_j - \tau_j|] $ such that the following expression hold:
\begin{align} 
\label{eq:C_TaylorExpansion_max}
\big| \hat{C}_+( -\gamma_j + \tau_j) -\hat{C}_+(0) \big| &=  \bigg| \frac{ {(-\gamma_j + \tau_j)}^2}{2} \hat{C}_+''(\alpha) \bigg| \leq \frac{ \mathfrak{d}_\infty^2}{2} \!\cdot\! |\hat{C}_+''(\alpha) |  
\end{align}
where $\hat{C}_+'$ and $\hat{C}_+''$ represent the first and second derivative of $\hat{C}_+$ respectively. Next, we upper bound the term in the right-hand side of \eqref{eq:C_TaylorExpansion_max} by using that fact the construction of majorant also admits that $\hat{C}_+''$ is a good approximation of 
 $\hat{P}_g''$ which represents the second derivative of the Fourier transform of the power spectral density. Hence,
 one can utilize the absolute difference between
$\hat{C}_+''(\alpha)$ and $\hat{P}_g''(\alpha)$ to bound $|\hat{C}_+''(\alpha) |$ in \eqref{eq:C_TaylorExpansion_max}. For all $\alpha \in \mathbb{R}$, it comes that
\begin{align}\label{eq:positiveFourierTransform}
    | \hat{C}_+''(\alpha) - \hat{P}_g''(\alpha)| &=  \left| \int_{-\infty}^{\infty} - 4 \pi^2 f^2 \bigg( C_+(f) - P_g(f) \bigg) e^{i 2\pi \alpha f} df \right|  \nonumber \\
    & \leq \int_{-\infty}^{\infty} \left| -4 \pi^2 f^2 \bigg( C_+(f) - P_g(f) \bigg) e^{i 2\pi \alpha f} \right| df \nonumber\\
    & = \int_{-\infty}^{\infty}  4 \pi^2 f^2 \bigg( C_+(f) - P_g(f) \bigg)  df  \nonumber \\
    & = - \hat{C}_+'' (0) + \hat{P}_g''(0)
\end{align}
where we can remove the absolute value in the second last equality since $C_+(f)-P_g(f)\geq 0$ for all $f$ as $C_+$ is a majorant for $P_g$. 
Since $ 4 \pi^2 f^2 P_g(f) \geq 0$ for all $ f \in \mathbb{R}$, the Fourier transform  $|\hat{P}_g''|$ achieves its maximal value at $0$. 
Furthermore, by the triangle inequality, for all $\alpha \in \mathbb{R}$
\begin{align}\label{eq:intermediate}
     |\hat{P}_g''(\alpha)| &= \left| \int_{-\infty}^{\infty} -4\pi^2 f^2 P_g(f) e^{i 2\pi\alpha f} df \right| \nonumber \\
     &\leq \int_{-\infty}^{\infty} \left| -4\pi^2 f^2 P_g(f) e^{i 2\pi\alpha f} \right| df = \int_{-\infty}^{\infty} 4\pi^2 f^2 P_g(f) df = - P_g''(0) = E_{g'}
\end{align}
Thus, utilizing \eqref{eq:intermediate} and the triangle inequality, one can write
\eqref{eq:positiveFourierTransform} as
\begin{equation}
\label{eq:boundCalpha_majorant}
 \left| \hat{C}_+''(\alpha)\right|  \leq  \left|\hat{P}_g''(\alpha)\right|   - \hat{C}_+^{\prime\prime}(0) + \hat{P}_g(0) =  E_{g'} \left(1 + \frac{ - \hat{C}_+''(0) - E_{g'}}{E_{g'}} \right)
\end{equation}
where the term $- \hat{C}_+''(0) - E_{g'}$ in \eqref{eq:boundCalpha_majorant} can be bounded using the results from Lemma~\ref{lem:ferreira_thm_3} as 
\begin{align}
\label{eq:majorant_bound}
    -\hat{C}_+^{\prime\prime}(0) - E_{g'} &\leq \frac{1}{2} V(P_{g'})(\Delta-2\mathfrak{d}_\infty)^{-1}.
\end{align}
Finally, substituting  \eqref{eq:majorant_bound} into \eqref{eq:boundCalpha_majorant} and back into \eqref{eq:C_TaylorExpansion_max} and \eqref{eq:FT_majorant_upperbound} concludes the upper bound on $\norm{\bm{A}_{\bm{\gamma}}\mathbf{u} + \bm{A}_{\bm{\tau}} \mathbf{v} }_{2}^2$ in \eqref{eq:lemma_quadra_majorant}.

The second assertion in \eqref{eq:lemma_quadra_minorant} is derived similarly using the minorant $C_{-}$ of $P_g$ by Lemma~\ref{lem:ferreira_thm_3}. 
Similar to the derivation of \eqref{eq:FT_majorant_upperbound}, a lower bound on $\norm{\bm{A}_{\bm{\gamma}}\mathbf{u} +\bm{A}_{\bm{\tau}} \mathbf{v} }_{2}^2$ is given by
\begin{align}
\label{eq:FT_minorant_lowerbound}
\norm{\bm{A}_{\bm{\gamma}}\mathbf{u} + \bm{A}_{\bm{\tau}} \mathbf{v} }_{2}^2 \geq T \hat{C}_-(0) \displaystyle \sum_{j=1}^K  |u_j + v_j|^2
+ 2 T \sum_{j=1}^K \big( \hat{C}_-(-\gamma_j+\tau_j) - \hat{C}_-(0) \big)  \Re \left(u_j \bar{v}_j \right).
\end{align}
The minorant $C_-$ in \eqref{eq:FT_minorant_lowerbound} is also twice-differentiable, even symmetric and its first derivative of the Fourier transform at $0$ satisfies  $\hat{C}_-'(0)=0$. Furthermore, its second order Taylor series expansions at $0$ also ensures that the following expression holds
\begin{align*}
\big| \hat{C}_-( -\gamma_j + \tau_j) -\hat{C}_-(0) \big| &=  \bigg| \frac{ {(-\gamma_j + \tau_j)}^2}{2} \hat{C}_-''(\alpha) \bigg| \leq \frac{ \mathfrak{d}_\infty^2}{2} \!\cdot\! |\hat{C}_-''(\alpha) | 
\end{align*}
where $\alpha \in [0, |\gamma_j - \tau_j|] $.
Finally, we utilize the results from Lemma~\ref{lem:ferreira_thm_3} in a similar manner to bound $\hat{C}_-(0) $ and $|\hat{C}_-''(\alpha) |$ to obtain the expression in \eqref{eq:lemma_quadra_minorant}, which completes the proof. 
\end{proof}

The following lemma is obtained as a corollary of Lemma~\ref{lem:bounding_sing_val_phi_gamm_phi_tau}. 

\begin{lemma}
\label{cor:bounding_sing_val_phi_gamm_phi_tau}
Under the hypothesis of Lemma~\ref{lem:bounding_sing_val_phi_gamm_phi_tau},
the following inequalities hold:
\begin{equation}
\label{eq:main_lem_res}
\norm{\bm{A}_{\bm{\gamma}}^\mathsf{H} \bm{A}_{\bm{\tau}}  - T\!\cdot\!E_{g} \bm{I}_K }_2
 \leq \frac{ 2V(P_{g})}{3(\Delta-2\mathfrak{d}_\infty)} +  \mathfrak{d}_\infty^2  T\!\cdot\!E_{g'} \bigg( 1 + \frac{1}{2} \rho_{g'} (\Delta-2\mathfrak{d}_\infty)^{-1}\bigg)
\end{equation}
\begin{equation}
\label{eq:main_lem_res2}
\norm{\bm{P_{\bm{\gamma}}}^\perp \bm{A}_{\bm{\tau}}}_2 \leq \norm{\bm{A}_{\bm{\gamma}} -\bm{A}_{\bm{\tau}}}_{2} 
 \leq  \mathfrak{d}_\infty \sqrt{T\!\cdot\!E_{g'} \bigg( 1 + \frac{1}{2} \rho_{g'} (\Delta-2\mathfrak{d}_\infty)^{-1}\bigg)}.
\end{equation} 
\end{lemma}
\begin{proof}
Using the definition of $\bm{A}_{\bm{\tau}}$ and $\bm{A}_{\bm{\gamma}}$ from \eqref{eq:def_Agamma},  the left-hand side of \eqref{eq:main_lem_res} can be written as 
\begin{align}
\label{intermediate_eq}
{\norm{\bm{A}_{\bm{\gamma}}^\mathsf{H}   \bm{A}_{\bm{\tau}}  - T \,E_{g} \bm{I}}}_2 = \sup_{\substack{\norm{\mathbf{u}}_2 = 1 \\ \norm{\mathbf{v}}_2 = 1}} \bigg| \mathbf{u}^\mathsf{H}\bm{A}_{\bm{\gamma}}^\mathsf{H}   \bm{A}_{\bm{\tau}} \mathbf{v}  - T\,E_{g} \mathbf{u}^\mathsf{H} \mathbf{v} \bigg|. 
\end{align}
We first consider the parallelogram identity for the complex components in the right-hand side of \eqref{intermediate_eq} as
\begin{multline}
\label{eq:parallelo_1}
\mathbf{u}^\mathsf{H}\bm{A}_{\bm{\gamma}}^\mathsf{H}  \bm{A}_{\bm{\tau}} \mathbf{v} 
= \frac{\norm{\bm{ A}_{\bm{\gamma}}\mathbf{u} + \bm{A}_{\bm{\tau}} \mathbf{v} }_{2}^2 - \norm{\bm{ A}_{\bm{\gamma}}\mathbf{u}  - \bm{A}_{\bm{\tau}} \mathbf{v} }_{2}^2  }{4} \\
+ j \frac{\norm{j\bm{ A}_{\bm{\gamma}}\mathbf{u} - \bm{A}_{\bm{\tau}} \mathbf{v} }_{2}^2 - \norm{j\bm{A}_{\bm{\gamma}}\mathbf{u}  +  \bm{A}_{\bm{\tau}} \mathbf{v} }_{2}^2  }{4}
\end{multline}
and  
\begin{align}
\label{eq:parallelo_2}
\mathbf{u}^\mathsf{H} \mathbf{v} 
& = \frac{\norm{\mathbf{u} +\mathbf{v} }_{2}^2 - \norm{\mathbf{u}  - \mathbf{v} }_{2}^2  }{4} + j \frac{\norm{j\,\mathbf{u} - \mathbf{v} }_{2}^2 - \norm{j\mathbf{u}  +  \mathbf{v} }_{2}^2  }{4}.
\end{align}
Then, we take the difference of the two quantities in \eqref{eq:parallelo_1}, \eqref{eq:parallelo_2} and scale $\mathbf{u}^\mathsf{H} \mathbf{v}$ with a constant $E_{g}$ such that we obtain 
\begin{align*}
\big| \mathbf{u}^\mathsf{H}\bm{A}_{\bm{\gamma}}^\mathsf{H}  \bm{A}_{\bm{\tau}} \mathbf{v}  - E_{g} \mathbf{u}^\mathsf{H} \mathbf{v} \big|
 = & \bigg| \frac{1}{4}  (\norm{\bm{ A}_{\bm{\gamma}}\mathbf{u} +\bm{A}_{\bm{\tau}} \mathbf{v} }_{2}^2 - E_{g} \norm{\mathbf{u} +\mathbf{v} }_{2}^2) \nonumber \\
& - \frac{1}{4} (\norm{\bm{A}_{\bm{\gamma}}\mathbf{u}  - \bm{A}_{\bm{\tau}} \mathbf{v} }_{2}^2 - E_{g} \norm{\mathbf{u}  - \mathbf{v} }_{2}^2 ) \nonumber \\
& + \frac{j}{4} (\norm{j\bm{A}_{\bm{\gamma}}\mathbf{u} -\bm{A}_{\bm{\tau}} \mathbf{v} }_{2}^2 - E_{g} \norm{j\mathbf{u} - \mathbf{v} }_{2}^2) \nonumber \\
& - \frac{j}{4} (\norm{j \bm{A}_{\bm{\gamma}}\mathbf{u}  + \bm{A}_{\bm{\tau}} \mathbf{v} }_{2}^2 - E_{g} \norm{j\mathbf{u}  +  \mathbf{v} }_{2}^2)\bigg|.
\end{align*}
Furthermore, we also have the following inequality
\begin{align}
\label{eq:parallelogram_iden_psf_bound}
   \big| \mathbf{u}^\mathsf{H}\bm{A}_{\bm{\gamma}}^\mathsf{H}  \bm{A}_{\bm{\tau}} \mathbf{v}  - E_{g} \mathbf{u}^\mathsf{H} \mathbf{v} \big| 
   & \leq \bigg| \frac{1}{4} (\norm{\bm{A}_{\bm{\gamma}}\mathbf{u} +\bm{A}_{\bm{\tau}} \mathbf{v}}_{2}^2 - E_{g} \norm{\mathbf{u} +\mathbf{v} }_{2}^2)\bigg| \nonumber \\
   & \quad + \bigg|  \frac{1}{4} (\norm{\bm{A}_{\bm{\gamma}}\mathbf{u}  - \bm{A}_{\bm{\tau}} \mathbf{v}}_{2}^2 - E_{g} \norm{\mathbf{u}  - \mathbf{v} }_{2}^2 )\bigg| \nonumber \\
& \quad + \bigg|\frac{1}{4} (\norm{j\bm{A}_{\bm{\gamma}}\mathbf{u} -\bm{A}_{\bm{\tau}} \mathbf{v} }_{2}^2 - E_{g} \norm{j\mathbf{u} - \mathbf{v} }_{2}^2)\bigg| \nonumber\\
& \quad + \bigg|  \frac{1}{4} (\norm{j \bm{A}_{\bm{\gamma}}\mathbf{u}  + \bm{A}_{\bm{\tau}} \mathbf{v} }_{2}^2 - E_{g} \norm{j\mathbf{u}  +  \mathbf{v} }_{2}^2)\bigg|.
\end{align}
Using the results in \eqref{lem:bound_quadratic_term} from Lemma~\ref{lem:bounding_sing_val_phi_gamm_phi_tau}, one could show analogously that the following inequalities hold for all the  terms in the parallelogram identity in \eqref{eq:parallelogram_iden_psf_bound} as
\begin{subequations}
\label{eq:parallelo_iden_interim}
\begin{align}
\begin{aligned}[t]
\left\vert \norm{\bm{A}_{\bm{\gamma}}\mathbf{u}
+\bm{A}_{\bm{\tau}}\mathbf{v}}_{2}^{2}
- E_{g}\norm{\mathbf{u}+\mathbf{v}}_{2}^{2} \right\vert
&\leq
\frac{2V(P_{g})}{3(\Delta-2\mathfrak{d}_\infty)}
\left\Vert \mathbf{u}+\mathbf{v}\right\Vert_{2}^{2}
\\
&\quad
+ \mathfrak{d}_\infty^{2}E_{g'}
\bigg(1+\frac{1}{2}
\frac{V(P_{g'})}
     {E_{g'}(\Delta-2\mathfrak{d}_\infty)}
\bigg)
\left\vert\Re(\mathbf{u}^{\mathsf H}\mathbf{v})\right\vert ,
\end{aligned}
\\[0.5em]
\begin{aligned}[t]
\left\vert \norm{\bm{A}_{\bm{\gamma}}\mathbf{u}
-\bm{A}_{\bm{\tau}}\mathbf{v}}_{2}^{2}
- E_{g}\norm{\mathbf{u}-\mathbf{v}}_{2}^{2}\right\vert
&\leq
\frac{2V(P_{g})}{3(\Delta-2\mathfrak{d}_\infty)}
\left\Vert \mathbf{u}-\mathbf{v}\right\Vert_{2}^{2}
\\
&\quad
+ \mathfrak{d}_\infty^{2}E_{g'}
\bigg(1+\frac{1}{2}
\frac{V(P_{g'})}
     {E_{g'}(\Delta-2\mathfrak{d}_\infty)}
\bigg)
\left\vert\Re(\mathbf{u}^{\mathsf H}\mathbf{v})\right\vert ,
\end{aligned}
\\[0.5em]
\begin{aligned}[t]
\left\vert \norm{j\bm{A}_{\bm{\gamma}}\mathbf{u}
-\bm{A}_{\bm{\tau}}\mathbf{v}}_{2}^{2}
- E_{g}\norm{j\mathbf{u}-\mathbf{v}}_{2}^{2}\right\vert
&\leq
\frac{2V(P_{g})}{3(\Delta-2\mathfrak{d}_\infty)}
\left\Vert j\mathbf{u}-\mathbf{v}\right\Vert_{2}^{2}
\\
&\quad
+ \mathfrak{d}_\infty^{2}E_{g'}
\bigg(1+\frac{1}{2}
\frac{V(P_{g'})}
     {E_{g'}(\Delta-2\mathfrak{d}_\infty)}
\bigg)
\left\vert\Re(\mathbf{u}^{\mathsf H}\mathbf{v})\right\vert ,
\end{aligned}
\\[0.5em]
\begin{aligned}[t]
\left\vert \norm{j\bm{A}_{\bm{\gamma}}\mathbf{u}
+\bm{A}_{\bm{\tau}}\mathbf{v}}_{2}^{2}
- E_{g}\norm{j\mathbf{u}+\mathbf{v}}_{2}^{2}\right\vert
&\leq
\frac{2V(P_{g})}{3(\Delta-2\mathfrak{d}_\infty)}
\left\Vert j\mathbf{u}+\mathbf{v}\right\Vert_{2}^{2}
\\
&\quad
+ \mathfrak{d}_\infty^{2}E_{g'}
\bigg(1+\frac{1}{2}
\frac{V(P_{g'})}
     {E_{g'}(\Delta-2\mathfrak{d}_\infty)}
\bigg)
\left\vert\Re(\mathbf{u}^{\mathsf H}\mathbf{v})\right\vert .
\end{aligned}
\end{align}
\end{subequations}
Next, we substitute \eqref{eq:parallelo_iden_interim}
into \eqref{eq:parallelogram_iden_psf_bound}, and utilize the triangle inequality to yield
\[  \bigg| \mathbf{u}^\mathsf{H}\bm{A}_{\bm{\gamma}}^\mathsf{H}   \bm{A}_{\bm{\tau}} \mathbf{v}  - E_{g} \mathbf{u}^\mathsf{H} \mathbf{v} \bigg|\leq \frac{{2}}{3}V(P_{g})(\Delta-2\mathfrak{d}_\infty)^{-1} +  \mathfrak{d}_\infty^2 E_{g'}  \bigg(1+ \frac{1}{2} \frac{V(P_{g'}) (\Delta-2\mathfrak{d}_\infty)^{-1}}{E_{g'}}  \bigg)  {\norm{\mathbf{v}}}_2
\]
which concludes the desired result in \eqref{eq:main_lem_res}.
For the second result in \eqref{eq:main_lem_res2}, we use the fact that 
$\bm{P_{\bm{\gamma}}}^\perp \bm{A}_{\bm{\gamma}}=\bm{P_{\bm{\gamma}}}^\perp \bm{G} \bm{\Phi}_{\bm{\gamma}}  = \bm{0} $ where $\bm{P_{\bm{\gamma}}}^\perp = \bm{I}_N - \bm{G} \bm{\Phi}_{\bm{\gamma}} (\bm{G} \bm{\Phi}_{\bm{\gamma}})^\dagger$ and thus, one can write
\[
\norm{\bm{P_{\bm{\gamma}}}^\perp\bm{A}_{\bm{\tau}}}_2 = 
\norm{\bm{P_{\bm{\gamma}}}^\perp \bm{G} \bm{\Phi}_{\bm{\tau}}}_2
= 
\norm{\bm{P_{\bm{\gamma}}}^\perp \bm{G} \left( \bm{\Phi}_{\bm{\tau}} - \bm{\Phi}_{\bm{\gamma}} \right)}_2
\leq \left\|  \bm{A}_{\bm{\tau}} - \bm{A}_{\bm{\gamma}} \right\|_2
\]
and lastly we substitute $\mathbf{v} = -\mathbf{u}$ as a special case in \eqref{eq:FT_majorant_upperbound} to obtain the upper bound in \eqref{eq:main_lem_res2}, which completes the proof.
\end{proof}

\begin{lemma}  
\label{lem:exsv_schur_comp}
Let $\bm{\gamma} = \{\gamma_k\}_{k=1}^K \subset \mathbb{T}$ be an arbitrary set of points on the length-T torus, and denote by $\Delta_{\bm{\gamma}}$ the minimum separation of $\bm{\gamma}$ as defined in~\eqref{eq:min_sep}. Furthermore let $\bm{S}_{\bm{\gamma}} = \bm{A}_{\bm{\gamma}}^\mathsf{H} \bm{\Lambda}^\mathsf{H}
 \bm{P_{\bm{\gamma}}}^\perp
 \bm{\Lambda} \bm{A}_{\bm{\gamma}}$ and $\bm{F}_{\bm{\gamma}} = \bm{P}_{\bm{\gamma}}^\perp \bm{\Lambda} \bm{A}_{\bm{\gamma}}$, where 
 $\bm{A}_{\bm{\gamma}}$ and $\bm{P_{\bm{\gamma}}}^\perp$are defined in
 \eqref{eq:def_Agamma} and
\eqref{eq:def:orth_proj},  respectively. Also, let $[\bm{\Lambda}]_{i,i} = - \mathsf{j} 2 \pi f_i \quad \text{for} \quad i \in [N].$
Then, the inequalities
\begin{align*}
    \norm{\bm{S}_{\bm{\gamma}} - T\!\cdot\!E_{g'}\bm{I}_K}
    \leq \frac{2}{3} T\!\cdot\!E_{g'} \rho_{g'}  \Delta_\gamma^{-1}
\end{align*}
\begin{align*}
\norm{\bm{F}_{\bm{\gamma}}}  \leq  {\left\Vert \bm{S}_{\bm{\gamma}} \right\Vert}_2^{1/2} 
\leq   \sqrt{T\!\cdot\!E_{g'} + {\norm{\bm{S}_{\bm{\gamma}} - T\!\cdot\!E_{g'} \bm{I}_K}}} \leq \sqrt{T\!\cdot\!E_{g'} \left(1 + \frac{2}{3} \rho_{g'}  \Delta_\gamma^{-1}\right)}
\end{align*}
hold, where the constants $E_{g'}$ and $\rho_{g'}$ defined in \eqref{eq:def_energy}, \eqref{eq:def_rho_g}. 
\end{lemma}
\begin{proof}
    Let $\alpha = \sqrt{\frac{E_{g'}}{E_{g}}}$ be a scaling factor, define $\bm{U} = [\alpha \bm{A}_{\bm{\tau}}, \; \bm{\Lambda} \bm{A}_{\bm{\tau}}]$, and $\bm{M} = \bm{U}^{\ast} \bm{U}$. 
By a direct calculation, we have the block decomposition
\begin{align*}
\bm{M} = \begin{bmatrix}
\alpha^2 \bm{A}_{\bm{\tau}}^{\ast}   \bm{A}_{\bm{\tau}} 
& 
\alpha \bm{A}_{\bm{\tau}}^{\ast} \bm{\Lambda}  \bm{A}_{\bm{\tau}} 
\\ 
\alpha \bm{A}_{\bm{\tau}}^{\ast} \bm{\Lambda}^* \bm{A}_{\bm{\tau}} 
& \bm{A}_{\bm{\tau}}^{\ast} \bm{\Lambda}^* \bm{\Lambda} \bm{A}_{\bm{\tau}} 
\end{bmatrix}.
\end{align*}
By the linear independence of trigonometric polynomials and their derivatives, and since $\bm{G}$ has at least $2K$ non-zero diagonal entries, the matrices $ \bm{A}_{\bm{\tau}}$, $\bm{\Lambda} \bm{A}_{\bm{\tau}}$, and  $\bm{U}$ are full column rank. This implies the matrix $\bm{M}$, and its two diagonal blocks are invertible. Hence, from the Schur block inversion formula, we have
\begin{align*}
    \bm{M}^{-1} = \begin{bmatrix}
        \ast & \ast \\
        \ast & \bm{S}^{-1}
    \end{bmatrix},
\end{align*}
where we neglected the derivation of blocks marked with an asterisk. Since $\bm{M}$, $\bm{S}$ and their inverse are positive-definite, one may write
\begin{equation*}
    \lambda_{\max} \left(\bm{M}^{-1} \right) \geq  \lambda_{\max} \left( \bm{S}^{-1} \right), \quad \lambda_{\min} \left(\bm{M}^{-1} \right) \leq  \lambda_{\min} \left( \bm{S}^{-1} \right).
\end{equation*}
It comes with $\lambda_{\min}(\bm{Q}^{-1}) = \lambda_{\max}(\bm{Q})^{-1}$, and $\lambda_{\max}(\bm{Q}^{-1}) = \lambda_{\min}(\bm{Q})^{-1}$ for any positive definite matrix $\bm{Q}$ on the inequalities
\begin{align}\label{eq:bound_S_M}
    \lambda_{\max} \left(\bm{M} \right) & \geq  \lambda_{\max} \left( \bm{S} \right),  & \lambda_{\min} \left(\bm{M} \right) & \leq  \lambda_{\min} \left( \bm{S} \right).
\end{align}
With~\eqref{eq:bound_S_M}, we seek to bound the extremal eigenvalues of $\bm{M}$. Relevant bounds are provided in the following Lemma, recalled from~\cite{ferreira2023conditionNumber} eigenvalues of $\bm{M}$, and rely on the Beurling--Selberg extremal approximation of functions with bounded variation.

\begin{lemma}\label{lem:second_order} For any $\Delta > \frac{2}{3} \rho \kappa^2$, one has the inequalities
\begin{align*}
        \lambda_{\min} \left(\bm{M} \right) & \geq T\!\cdot\!E_{g'}(1-\frac{3}{2}\rho\Delta^{-1}), & \lambda_{\max} \left(\bm{M} \right) & \geq T\!\cdot\!E_{g'}(1+\frac{3}{2}\rho\Delta^{-1}).
\end{align*}
\end{lemma}
One concludes immediately with~\eqref{eq:bound_S_M} and Lemma~\ref{lem:second_order}.

\end{proof}

\section{Proof of Main Results}
\label{sec5 derivation proof}
We now prove the theoretical guarantees stated in 
Section~\ref{sec:main_results}. The argument proceeds in three stages.

We first establish the local geometric properties of the VarProSD 
objective within the basin $\mathcal{N}(\bm{\tau},\varrho)$. 
Specifically, we show that the objective is strongly convex and has 
Lipschitz continuous gradient throughout this region 
(Lemma~\ref{lem:local_geometry}), which immediately yields the 
existence and uniqueness of a local minimizer 
$\bm{\gamma}_\star \in \mathcal{N}(\bm{\tau},\varrho)$ and proves Theorem~\ref{thm:local_geometry}.

Building on this local geometry, we then analyze gradient descent 
through a single one-step recursion on the iterates. This recursion 
simultaneously yields linear convergence to $\bm{\gamma}_\star$ and 
an explicit estimation error bound, which together prove Theorem~\ref{thm:minimizer_error} 
and Theorem~\ref{thm:gd_convergence}.

Finally, we establish a deterministic stability guarantee under 
adversarial noise via the local Lipschitz property of the inverse map, 
which proves Theorem~\ref{thm:mainalt}.

\subsection{Local Geometry of the VarProSD Objective within the Basin}
\label{subsec:proof_local_geometry}

We first establish the local geometric properties of the VarProSD objective
inside the basin $\mathcal N(\bm{\tau},\varrho)$. In particular, we show that the
objective is strongly convex and has Lipschitz continuous gradient in this
region. These properties form the foundation for all subsequent convergence and
stability analyses.

\begin{lemma}[Local Geometry Within the Basin]
\label{lem:local_geometry}
Let $\mathcal N(\bm{\tau},\varrho)$ denote the basin of convexity defined in
\eqref{eq:neighborohood}, where the radius $\varrho$ is given by
\eqref{eq:basin_radius_cond}. Suppose that the residual covariance satisfies
\begin{align*}
\|\bm R\|
\le
c_2 \sqrt{E_g E_{g'}}\, T\, r_{\min}^2(\bm X)
\sqrt{\frac{E_g}{E_{g'}} \wedge \frac{E_{g'}}{E_{g''}}}
\end{align*}
for a sufficiently small absolute constant $c_2>0$. Then, for all
$\bm{\gamma} \in \mathcal N(\bm{\tau},\varrho)$, the Hessian of the VarProSD
objective $\ell(\bm{\gamma})$ satisfies
\begin{align}
\label{eq:min_sing_value_lemma}
\sigma_{\min}\!\left(\nabla_{\bm{\gamma}}^2 \ell(\bm{\gamma})\right)
&\ge
\frac{1}{3} E_{g'}\, T\, r_{\min}^2(\bm X), \\
\label{eq:max_sing_value_lemma}
\sigma_{\max}\!\left(\nabla_{\bm{\gamma}}^2 \ell(\bm{\gamma})\right)
&\le
E_{g'}\, T\, r_{\max}^2(\bm X).
\end{align}
Consequently, $\ell(\bm{\gamma})$ is $\mu$-strongly convex and has $\nu$-Lipschitz continuous gradient on $\mathcal N(\bm{\tau},\varrho)$, where
$\mu = \frac{1}{3} E_{g'} T r_{\min}^2(\bm{X})$ and $\nu = E_{g'} T r_{\max}^2(\bm{X})$. 
\end{lemma}
The proof of Lemma~\ref{lem:local_geometry} is given in Appendix~\ref{secA4}.

\paragraph{Proof of Theorem~\ref{thm:local_geometry}}
By Lemma~\ref{lem:local_geometry}, the objective function is strongly convex on
$\mathcal N(\bm{\tau},\varrho)$, which guarantees the existence and uniqueness of
a local minimizer $\bm{\gamma}_\star$ within this basin.

\subsection{Proof of Theorem~\ref{thm:minimizer_error} and Theorem~\ref{thm:gd_convergence}}
\label{subsec:proof_gd}

We analyze gradient descent under the local geometric conditions established in
Lemma~\ref{lem:local_geometry}, which guarantees strong convexity and smoothness
of the objective within the basin $\mathcal N(\bm{\tau},\varrho)$.

The proof is obtained by establishing two properties of the gradient descent
iterates. First, assuming $\bm{\gamma}_t \in \mathcal N(\bm{\tau},\varrho)$, we
derive the following one-step recursion for the estimation error
\begin{equation}
\label{eq:gd_recursion}
\mathrm{d}_2(\bm{\gamma}_{t+1},\bm{\tau})^2
\;\le\;
\vartheta^2\,\mathrm{d}_2(\bm{\gamma}_t,\bm{\tau})^2
\;+\;
\frac{C K}{T^2 E_g E_{g'}}
\frac{\|\bm{R}\|^2}{r_{\min}^2(\bm{X})\,r_{\max}^2(\bm{X})},
\end{equation}
where
\[
\vartheta
=
\frac{6r_{\max}^2(\bm{X})}{6r_{\max}^2(\bm{X})+r_{\min}^2(\bm{X})}
\in (0,1).
\]
Second, we show that the gradient descent trajectory remains within
$\mathcal N(\bm{\tau},\varrho)$ for all iterations.
\begin{align*}
\bm{\gamma}_t \in \mathcal N(\bm{\tau},\varrho)
\;\Longrightarrow\;
\bm{\gamma}_{t+1} \in \mathcal N(\bm{\tau},\varrho).
\end{align*}
This step is essential since the basin of attraction is centered at
$\bm{\tau}$ rather than at the local minimizer.

We start by proving the recursion in \eqref{eq:gd_recursion}.
Assume that $\bm{\gamma}_t \in \mathcal N(\bm{\tau},\varrho)$, then without loss of
generality, we may assume that
$\mathrm{d}_2(\bm{\gamma}_t,\bm{\tau}) = \|\bm{\gamma}_t - \bm{\tau}\|_2$ by
choosing the appropriate 
element-wise shifts by integer multiples
of $T$. Recall that in the noiseless case, $\bm{\tau}$ is a global minimizer of the
cost function, and hence
$\nabla_{\bm{\gamma}}\ell(\bm{\tau},\bm{0})=\bm{0}$. Furthermore, since
$\bm{A}_{\bm{\gamma}}^\dagger \bm{P}_{\bm{\gamma}}^\perp=\bm{0}$, the expression of
the gradient in Lemma~\ref{lem:expression_for_q_alg} can be equivalently written as
\begin{align*}
\frac{\partial \ell(\bm{\gamma},\bm{Z})}{\partial \gamma_k}
=
-\frac{1}{L}\mathrm{Re}\!\left[
\bm{e}_k^\mathsf{T}
\bm{A}_{\bm{\gamma}}^\dagger
(\bm{Y}\bm{Y}^\mathsf{H}-c\bm{I}_N)
\bm{P}_{\bm{\gamma}}^\perp
\bm{\Lambda}
\bm{A}_{\bm{\gamma}}
\bm{e}_k
\right],
\end{align*}
for any $c\in\mathbb{R}$.
By definition of the residual covariance matrix 
$\bm{R} = \frac{1}{L}\bm{Y}\bm{Y}^\mathsf{H} - \frac{1}{L}\bm{Y}_0\bm{Y}_0^\mathsf{H} - c\bm{I}_N$ 
with $\bm{Y}_0 = \bm{G}\bm{\Phi}_{\bm{\tau}}\bm{X}$, the gradient at $\bm{\gamma}_t$ 
admits the decomposition
\begin{align*}
\nabla_{\bm{\gamma}}\ell(\bm{\gamma}_t,\bm{Z})
&=
\nabla_{\bm{\gamma}}\ell(\bm{\gamma}_t,\bm{0})
-
\operatorname{Re}\!\left[
\operatorname{diag}\!\left(
(\bm{A}_{\bm{\gamma}_t})^\dagger
\bm{R}
\bm{P}_{\bm{\gamma}_t}^\perp
\bm{\Lambda}
\bm{A}_{\bm{\gamma}_t}
\right)
\right].
\end{align*}
Using the gradient descent update with step size
$\alpha = \frac{2}{\mu+\nu}$, we obtain
\begin{align}
\label{eq:gd_proof_interm1}
\bm{\gamma}_{t+1}-\bm{\tau}
&=
\underbrace{
\bm{\gamma}_t-\bm{\tau}
-\alpha\nabla_{\bm{\gamma}}\ell(\bm{\gamma}_t,\bm{0})
+\alpha\nabla_{\bm{\gamma}}\ell(\bm{\tau},\bm{0})
}_{\chi_1}
-
\underbrace{
\alpha\,\mathrm{Re}\!\left[
\operatorname{diag}\!\left(
(\bm{A}_{\bm{\gamma}_t})^\dagger
\bm{R}
\bm{P}_{\bm{\gamma}_t}^\perp
\bm{\Lambda}
\bm{A}_{\bm{\gamma}_t}
\right)
\right]
}_{\chi_2}.
\end{align}
Squaring both sides in \eqref{eq:gd_proof_interm1} yields
\[
\|\bm{\gamma}_{t+1}-\bm{\tau}\|_2^2
=
\|\chi_1\|_2^2+\|\chi_2\|_2^2
-2\langle\chi_1,\chi_2\rangle.
\]
By the AM-GM inequality, for any $\varphi>0$, the cross-term can be upper bounded as
\[
2\langle\chi_1,\chi_2\rangle
\le
\varphi\|\chi_1\|_2^2+\frac{1}{\varphi}\|\chi_2\|_2^2,
\]
which implies
\begin{equation}
\label{eq:gd_proof_interm2}
\|\bm{\gamma}_{t+1}-\bm{\tau}\|_2^2
\le
(1+\varphi)\|\chi_1\|_2^2
+\left(1+\frac{1}{\varphi}\right)\|\chi_2\|_2^2.
\end{equation}

Since the objective is $\mu$-strongly convex and $\nu$-smooth on
$\mathcal N(\bm{\tau},\varrho)$, the classic result from
\cite[Theorem~2.1.15]{nesterov2013introductory} can be used to bound the $\|\chi_1\|_2$ term in \eqref{eq:gd_proof_interm2} as
\begin{align}
\|\chi_1\|_2
& \le
\left(
\frac{\nu-\mu}{\nu+\mu}
\right)
\|\bm{\gamma}_t-\bm{\tau}\|_2 \nonumber\\
& \leq  \left(
\frac{3r_{\max}^2(\bm{X})}{3r_{\max}^2(\bm{X})+r_{\min}^2(\bm{X})}
\right) \|\bm{\gamma}_t-\bm{\tau}\|_2
\label{eq:chi1_bound_final}
\end{align}
where
$\mu = \frac{1}{3} E_{g'} T r_{\min}^2(\bm{X})$ and $\nu = E_{g'} T r_{\max}^2(\bm{X})$. 

Next, we bound the perturbation term $\|\chi_2\|_2$. Using the inequality $\|\operatorname{diag}(\bm{M})\|_2 \le \sqrt{K}\,\|\bm{M}\|$ for any matrix $\bm{M}$, together with submultiplicativity of the operator norm,
we obtain
\begin{align}
\|\chi_2\|_2
&\leq
\alpha \sqrt{K}\,
\bigl\|
\bm{A}_{\bm{\gamma}_t}^\dagger
\bm{R}
\bm{P}_{\bm{\gamma}_t}^\perp
\bm{\Lambda}
\bm{A}_{\bm{\gamma}_t}
\bigr\|
\nonumber\\
&\leq
\frac{6 \sqrt{K}}{ T \, E_{g'}\Bigl(r_{\min}^2(\bm{X})+3\,r_{\max}^2(\bm{X})\Bigr)}
\;\cdot\;
\underbrace{\norm{\bm{A}_{\bm{\gamma}_t}^\dagger}}_{\mathrm{(c)}}
\;\cdot\;
\norm{\bm{R}}
\;\cdot\;
\norm{\bm{P}_{\bm{\gamma}_t}^\perp}
\;\cdot\;
\underbrace{\norm{\bm{\Lambda}\bm{A}_{\bm{\gamma}_t}}}_{\mathrm{(d)}}
\nonumber\\
& \leq
\frac{6 \sqrt{K}}{3\,r_{\max}^2(\bm{X})}
\;\cdot\;
\frac{1}{T \, \sqrt{E_{g}E_{g'}}}
\;\cdot\;
\frac{\eta_+}{\eta_-}
\;\cdot\;
\norm{\bm{R}},
\label{eq:chi2_bound_final}
\end{align}
where the quantities $\eta_+$ and $\eta_-$ are defined as
\begin{align}
\label{eq:def_eta}
\eta_+
=
\sqrt{1 + \frac{2}{3} \rho (\Delta-2\delta)^{-1}},
\qquad
\eta_-
=
\sqrt{1 - \frac{2}{3} \rho (\Delta-2\delta)^{-1}},
\end{align}
and $\delta 
:= 
\mathrm{d}_{\infty}(\bm{\tau},\bm{\gamma}) $. Here, $\norm{\bm{P}_{\bm{\gamma}_t}^\perp} \leq 1$, and the terms
$\mathrm{(c)}$ and $\mathrm{(d)}$ are bounded using the result of
\cite[Theorem~1]{ferreira2023conditionNumber}. For completeness, a paraphrased
version of these bounds in our notation is provided in the appendix
(Lemma~\ref{lem:ferreira_thm1} and Lemma~\ref{lem:cor_of_ferreira_thm1}). By the definition of the distance metric, any
$\bm{\gamma}\in\mathcal N(\bm{\tau},\varrho)$ satisfies
$\mathrm{d}_{\infty}(\bm{\gamma},\bm{\tau}) \le \mathrm{d}_{2}(\bm{\gamma},\bm{\tau}) \le \varrho$. Then 
using the assumption of the Theorem in \eqref{eq:neighborohood}, it further implies $\delta < \frac{1}{2} \left(\Delta - \frac{2}{3}\,\rho \, \kappa^2\right)$ and one can bound the  constants $\eta_+,
\eta_-$ as 
\begin{align*}
 \eta_+ \leq \sqrt{\frac{4}{3}}, \quad
\eta_- \geq \sqrt{\frac{2}{3}}. 
\end{align*}

Next, substituting the bounds in \eqref{eq:chi1_bound_final} and
\eqref{eq:chi2_bound_final} into \eqref{eq:gd_proof_interm2} yields
\begin{align}
\label{eq:gd_intermediate}
\|\bm{\gamma}_{t+1}-\bm{\tau}\|_2^2
&\le
(1+\varphi)
\left(
\frac{3r_{\max}^2(\bm{X})}{3r_{\max}^2(\bm{X})+r_{\min}^2(\bm{X})}
\right)^2
\|\bm{\gamma}_t-\bm{\tau}\|_2^2
+
\left(1+\frac{1}{\varphi}\right)
\frac{C K}{r_{\max}^4(\bm{X})}
\frac{\|\bm{R}\|^2}{T^2 E_g E_{g'}}.
\end{align}
The bound \eqref{eq:gd_intermediate} holds for any $\varphi>0$. We choose
$\varphi$ such that
\[
(1+\varphi)
\left(\frac{3r_{\max}^2(\bm X)}{3r_{\max}^2(\bm X)+r_{\min}^2(\bm X)}\right)^2
=
\left(\frac{6r_{\max}^2(\bm X)}{6r_{\max}^2(\bm X)+r_{\min}^2(\bm X)}\right)^2,
\]
which gives the simplified upper bound on $1+\frac{1}{\varphi}$ as
\begin{align}
\label{eq:gd_intermediate_phi1}
1+\frac{1}{\varphi} \leq \frac{48}{7}\,
\frac{r_{\max}^2(\bm{X})}{r_{\min}^2(\bm{X})}.
\end{align}

Substituting \eqref{eq:gd_intermediate_phi1} into \eqref{eq:gd_intermediate} yields
\begin{equation}
\label{eq:gd_one_step_recursion}
\|\bm{\gamma}_{t+1}-\bm{\tau}\|_2^2
\le
\underbrace{
\left(
\frac{6r_{\max}^2(\bm{X})}{6r_{\max}^2(\bm{X})+r_{\min}^2(\bm{X})}
\right)^2
}_{\vartheta}
\|\bm{\gamma}_t-\bm{\tau}\|_2^2
+
\underbrace{
\frac{C K}{T^2 E_g E_{g'}}
\frac{\|\bm{R}\|^2}{r_{\min}^2(\bm{X})r_{\max}^2(\bm{X})}
}_{\varepsilon_R}.
\end{equation}
We now convert the above bound from the $\ell_2$ metric to the
$\mathrm{d}_2$ distance. By construction, for the current iterate
$\bm{\gamma}_t$ we have
$\mathrm{d}_2(\bm{\gamma}_t,\bm{\tau})=\|\bm{\gamma}_t-\bm{\tau}\|_2$.
Moreover, by definition of $\mathrm{d}_2$,
\[
\mathrm{d}_2(\bm{\gamma}_{t+1},\bm{\tau})
\le
\|\bm{\gamma}_{t+1}-\bm{\tau}\|_2.
\]
Combining this inequality with \eqref{eq:gd_one_step_recursion} yields the
one-step recursion~\eqref{eq:gd_recursion}.
We now show that the gradient descent trajectory remains within
$\mathcal N(\bm{\tau},\varrho)$ for all iterations.
Suppose that $\bm{\gamma}_t \in \mathcal N(\bm{\tau},\varrho)$.
Using the recursion~\eqref{eq:gd_recursion} together with the induction hypothesis
$\mathrm{d}_2(\bm{\gamma}_t,\bm{\tau}) \le \varrho$
and the condition on $\norm{\bm{R}}$ in~\eqref{eq:cond_on_R1}, which ensures
\[
\frac{C K}{T^2 E_g E_{g'}}
\frac{\|\bm{R}\|^2}{r_{\min}^2(\bm{X})\,r_{\max}^2(\bm{X})}
\le
(1-\vartheta^2)\varrho^2,
\]
we obtain
\[
\mathrm{d}_2(\bm{\gamma}_{t+1},\bm{\tau})^2
\le
\vartheta^2\,\varrho^2 + (1-\vartheta^2)\varrho^2
=
\varrho^2, \, 
\]
which implies $\bm{\gamma}_{t+1} \in \mathcal N(\bm{\tau},\varrho)$. Using this result and the initial assumption that $\bm{\gamma}_0 \in \mathcal N(\bm{\tau},\varrho)$, we obtain
$\bm{\gamma}_t \in \mathcal N(\bm{\tau},\varrho)$ for all $t \ge 0$.
As a consequence, the recursion~\eqref{eq:gd_recursion} holds for all iterations,
and taking the limit superior yields
\[
\limsup_{t\to\infty}
\mathrm{d}_2(\bm{\gamma}_{t},\bm{\tau})^2
\le
\frac{1}{1-\vartheta^2}
\frac{C K}{T^2 E_g E_{g'}}
\frac{\|\bm{R}\|^2}{r_{\min}^2(\bm{X})\,r_{\max}^2(\bm{X})}.
\]

Finally, since all iterates remain within the neighborhood
$\mathcal N(\bm{\tau},\varrho)$ where the objective function
$\ell(\bm{\gamma},\bm{Z})$ is $\mu$-strongly convex and $\nu$-smooth, classical
results for gradient descent~\cite[Theorem~2.1.15]{nesterov2013introductory}
guarantee that the sequence $\{\bm{\gamma}_t\}$ converges linearly to the unique
local minimizer $\bm{\gamma}_\star \in \mathcal N(\bm{\tau},\varrho)$.
In particular, the iterates satisfy
\[
\mathrm{d}_2(\bm{\gamma}_{t},\bm{\gamma}_\star)^2
\le
\vartheta^{2t}\mathrm{d}_2(\bm{\gamma}_{0},\bm{\gamma}_\star)^2,
\]
which proves the linear convergence claim in~\eqref{eq:gd_linear_conv}. To get the
explicit estimation error bound, since $\bm{\gamma}_t \to \bm{\gamma}_\star$, we apply the limsup bound to
the recursion~\eqref{eq:gd_recursion} and obtain 
\begin{align*}
\mathrm{d}_2(\bm{\gamma}_\star, \bm{\tau})^2
\le
\frac{1}{1 - \vartheta^2}
\cdot
\frac{C K}{T^2 E_g E_{g'}}
\cdot
\frac{\|\bm R\|^2}{r_{\min}^2(\bm X)\, r_{\max}^2(\bm X)},
\end{align*}
which, after taking square roots, gives
\begin{equation}
\label{eq:err_bound_intermediate}
\mathrm{d}_2(\bm{\gamma}_\star, \bm{\tau})
\le
\frac{c_2 \sqrt{K}}{T \sqrt{E_g E_{g'}}}
\cdot
\frac{\|\bm R\|}{r_{\min}(\bm X)\, r_{\max}(\bm X)}
\cdot
\frac{1}{\sqrt{1 - \vartheta^2}}.
\end{equation}
We now simplify the right-hand side of \eqref{eq:err_bound_intermediate} into 
the form stated in Theorem~\ref{thm:minimizer_error}. A direct computation gives
\[
1 - \vartheta^2 
=
\frac{r_{\min}^2(\bm X) \big(12 r_{\max}^2(\bm X) + r_{\min}^2(\bm X)\big)}
     {\big(6 r_{\max}^2(\bm X) + r_{\min}^2(\bm X)\big)^2},
\]
so that
\[
\frac{1}{r_{\min}(\bm X)\, r_{\max}(\bm X) \sqrt{1 - \vartheta^2}}
=
\frac{1}{r_{\min}^2(\bm X)}
\cdot
\frac{6\varkappa^2 + 1}{\varkappa \sqrt{12\varkappa^2 + 1}},
\]
where $\varkappa := r_{\max}(\bm X) / r_{\min}(\bm X) \ge 1$ denotes the row 
dynamic range of the amplitude matrix. Multiplying and dividing by 
$\sqrt{E_g}$, the bound in \eqref{eq:err_bound_intermediate} becomes
\begin{equation}
\label{eq:err_bound_explicit}
\mathrm{d}_2(\bm{\gamma}_\star, \bm{\tau})
\le
c_2 \sqrt{K} \,
\sqrt{\frac{E_g}{E_{g'}}}
\cdot
\frac{\|\bm R\|}{T E_g\, r_{\min}^2(\bm X)}
\cdot
\frac{6\varkappa^2 + 1}{\varkappa \sqrt{12\varkappa^2 + 1}}.
\end{equation}
where $\frac{6\varkappa^2 + 1}{\varkappa \sqrt{12\varkappa^2 + 1}} \leq \frac{7}{\sqrt{13}}$ since $\varkappa \geq 1$. Absorbing 
this $\mathcal{O}(1)$ factor into the constant $c_2$ in 
\eqref{eq:err_bound_explicit} yields the assertion of 
Theorem~\ref{thm:minimizer_error},
\[
\mathrm{d}_2(\bm{\gamma}_\star, \bm{\tau})
\le
c_2 \sqrt{K}\, \sqrt{\frac{E_g}{E_{g'}}}\,
\frac{\|\bm R\|}{T E_g\, r_{\min}^2(\bm X)}.
\]
This completes the proof of Theorem~\ref{thm:minimizer_error} and Theorem~\ref{thm:gd_convergence}.

\subsection{Proof of Theorem~\ref{thm:mainalt}}
\label{subsec:proof_mainalt}
To begin with the proof, we first recall that $\bm{\gamma}_\star$ is defined as the image of the inverse map introduced in \eqref{eq:Psi_def}. Our aim is to bound the perturbation in $\mathrm{d}_2(\bm{\gamma}_\star,\bm{\tau})$ using the local Lipschitz analysis of $\bm{\psi}$. By definition, the inverse map $\bm{\psi}$ is \textit{locally Lipschitz continuous} at $\bm{z} = \bm{0} \in \mathbb{C}^{NL\times 1}$ if there exists a neighborhood $\mathcal{N}_{\bm{Z}}$ around $\bm{z} = \bm{0}$ and a constant $L_{\bm{\psi},\mathcal{N}} > 0$ such that
\begin{equation}
\label{eq:local_Lip2}
\frac{\norm{\bm{\psi}(\bm{z}_1)- \bm{\psi}(\bm{z}_2)}_2}{\norm{\bm{z}_1- \bm{z}_2}_\mathrm{2}} \leq L_{\bm{\psi},\mathcal{N}}, \quad \forall \bm{z}_1, \bm{z}_2 \in \mathcal{N}_{\bm{Z}}: \bm{z}_1 \neq \bm{z}_2. 
\end{equation}
Here, the local Lipschitz constant $L_{\bm{\psi},\mathcal{N}}$ is characterized via the spectral norm of the Jacobian matrix of $\bm{\psi}$ as
\begin{equation*}
L_{\bm{\psi},\mathcal{N}} = \sup_{\bm{z} \in \mathcal{N}} \norm{\nabla_{\bm{z}}\bm{\psi}(\bm{z})}
\end{equation*}
where $\nabla_{\bm{z}}\bm{\psi}(\bm{z})$ is defined in \eqref{eq:jacobian}. 
Then, by utilizing the local Lipschitz property introduced in \eqref{eq:local_Lip2}, one can obtain the following upper bound on $\mathrm{d}_{2}(\bm{\gamma}_\star,\bm{\tau})$ as
\begin{align}
\label{eq:lip_cond_ub}
\mathrm{d}_{2}(\bm{\gamma}_\star,\bm{\tau}) \leq \norm{\bm{\gamma}_\star-\bm{\tau}}_2 
\leq
\sup_{\bm{z} \in \mathcal{N}_{\bm{Z}}} \norm{\nabla_{\bm{z}}\bm{\psi}(\bm{z})} \cdot
\norm{\bm{z}}_\mathrm{2}, \quad \forall \bm{z} \in \mathcal{N}_{\bm{Z}}, 
\end{align}
where 
$\mathcal{N}_{\bm{Z}} = \{\mathrm{vec}(\bm{Z}) :\;\norm{\bm{Z}} \, \,\text{satisfies}\, 
\eqref{eq:cond_on_R1} \text{ and } \eqref{eq:cond_on_Z_in_lemma}\}$. Now, it only remains to derive an upper-bound on $\norm{\nabla_{\bm{z}}\bm{\psi}(\bm{z})}$. 

Thus, we substitute $\bm{\gamma} = \bm{\gamma}_\star$ and utilize the properties of the spectral norm in \eqref{eq:jacobian} and obtain the following upper-bound
\begin{equation}
\label{eq:jacobian_bnd}
\norm{\nabla_{\bm{z}} \bm{\psi}(\bm{z})} 
\leq  \frac{\norm{\nabla_{\bm{z}}\nabla_{\bm{\gamma}} \bm{\ell}({\bm{\gamma}_\star} ;\bm{Z}) }}{\sigma_{\min}(\nabla^2_{\bm{\gamma}} \bm{\ell}({\bm{\gamma}_\star}) )}.
\end{equation}
In the right-hand side of \eqref{eq:jacobian_bnd}, we identify that the lower-bound on $\sigma_{\min}\left(\nabla^2_{\bm{\gamma}} \bm{\ell}(\bm{\gamma}_\star)\right)$ is already computed in Lemma~\ref{lem:local_geometry} with $\bm{\gamma}_\star$ substituted by $\bm{\gamma}$. Next, we present the upper-bound on $\norm{\nabla_{\bm{z}}\nabla_{\bm{\gamma}} \bm{\ell}(\bm{\gamma}_\star;\bm{Z}) }$ in the following lemma.  
\begin{lemma}
\label{lem:spectral_norm_grad} Let $\delta 
:= 
\mathrm{d}_{\infty}(\bm{\tau},\bm{\gamma}) $ and 
suppose $\delta$ satisfies \eqref{eq:neighborohood}. Then, the following result holds
\begin{equation}
\label{eq:lemma_spectral_norm_jac_z}
\norm{\nabla_{\bm{z}}\nabla_{\bm{\gamma}} \bm{\ell}(\bm{\gamma}_\star;\bm{Z})}  
  \leq c\, \bigg(  \underbrace{ \sqrt{T\, E_{g'}}  \cdot \norm{\bm{X}}}_{\psi_1} +   \underbrace{\sqrt{T\, E_{g'}}  \cdot \norm{\bm{X}} \cdot \delta \cdot \sqrt{\frac{E_{g'}}{E_{g}}}}_{\psi_2}  + \underbrace{\sqrt{\frac{E_{g'}}{E_{g}}}  \cdot \norm{\bm{Z}}}_{\psi_3}\bigg).   
\end{equation}
for some small constant $c>0$. 
\end{lemma}
The proof of Lemma~\ref{lem:spectral_norm_grad} is given in Appendix~\ref{secA5}.
\noindent \newline
Since $\bm{\gamma} \in \mathcal N(\bm{\tau},\varrho)$, the matching distance satisfies $\delta \le \varrho$. By choosing the constant $c_1$ sufficiently small in \eqref{eq:basin_radius_cond}, one can make $\psi_2$ dominated by the first term $\psi_1$ in \eqref{eq:lemma_spectral_norm_jac_z}. Similarly, 
choosing $c_5$ sufficiently small in \eqref{eq:cond_on_Z_in_lemma} also makes $\psi_3$ dominated by $\psi_1$. 
Hence, the expression in \eqref{eq:lemma_spectral_norm_jac_z} can be simplified as
\begin{align}
    \label{eq:simpl_cros_driv}
\norm{\nabla_{\bm{z}}\nabla_{\bm{\gamma}} \bm{\ell}(\bm{\gamma}_\star;\bm{Z})}  
 \lesssim \sqrt{T\, E_{g'}}  \cdot \norm{\bm{X}}. 
\end{align}
Next, we plug in \eqref{eq:min_sing_value_lemma} and \eqref{eq:simpl_cros_driv} into \eqref{eq:jacobian_bnd} to obtain
\begin{align}
\label{eq_of_lem:spectral_norm_jacob2_ub}
\norm{\nabla_{\bm{z}} \bm{\psi}(\bm{z})} 
& \lesssim 
\frac{ \kappa^2}{  \sqrt{T\, E_{g'}}\!\cdot\! \norm{\bm{X}}}, 
\end{align}
and finally, plugging in \eqref{eq_of_lem:spectral_norm_jacob2_ub} into the right-hand side of \eqref{eq:lip_cond_ub} gives the desired assertion in \eqref{eq:theorem_main_result}. This concludes the proof.

\section{Discussion}\label{sec6 discussion}
We study multi-snapshot spike deconvolution under the variable projection formulation (VarProSD). The objective is nonconvex, so gradient descent recovers the spike locations only when initialized within a basin of convexity around the ground truth. Prior work does not quantify the size of
this basin. In this work, we characterize it explicitly through a radius $\varrho$ expressed in interpretable problem parameters: the spectral properties of the PSF, the minimum spike separation, the sampling bandwidth, and the
dynamic range of the amplitudes. To our knowledge, this is the first quantifiable local convergence guarantee for gradient descent in multi-snapshot spike deconvolution under an arbitrary PSF.

The unifying technical ingredient in our analysis is the Beurling--Selberg extremal approximation. We use it to bound the conditioning of the
structured generalized Vandermonde matrices that arise in the VarProSD landscape. These conditioning bounds support both parts of our analysis. They control the curvature of the objective, which yields the basin of
convexity. They also control the Hessian and cross-derivative terms that govern the Jacobian of the inverse map in the adversarial stability
analysis. The bounds are independent of the number of spikes and hold for any PSF of bounded variation. Our convergence and estimation-error
guarantees therefore depend explicitly on interpretable problem parameters.

Our analysis also emphasizes the role of the sampling bandwidth $B = N/T$. Since the spectral quantities that determine the basin radius are themselves functions of $B$, our characterization leads to a principled, PSF-driven criterion for selecting the bandwidth. Our experiments show
that this criterion closely tracks the empirically optimal choice and reveals that a larger bandwidth is not always advantageous.

One open question concerns the gap between the basin of convexity and the larger region from which gradient descent converges in practice. The radius $\varrho$ follows from a sufficient condition for strong convexity, and is therefore conservative. Gradient descent may reach $\bm{\gamma}_\star$ from initializations outside this region. Characterizing the larger region would require guarantees under weaker assumptions than strong convexity. We leave this to future work.

A further direction is to extend the present convergence analysis to the Gauss--Newton method. Gauss--Newton is the method of choice for variable
projection in practice. The closed-form Jacobian Gramian and Hessian, together with the Beurling--Selberg conditioning bounds derived in this work, can be used to analyze the local convergence of Gauss--Newton in this setting.

\appendix
This appendix collects the technical material supporting the main results. Appendices~\ref{secA1} and~\ref{secA2} recall prior-art results on the conditioning of structured matrices and on the Khatri-Rao and Hadamard products that are repeatedly invoked in our proofs. Appendix~\ref{secA3} collects the closed-form derivative identities for the VarProSD objective (gradient, Jacobian Gramian, Hessian, and cross-derivative with respect to the noise). Appendices~\ref{secA4} and~\ref{secA5} present the proofs of Lemma~\ref{lem:local_geometry} and Lemma~\ref{lem:spectral_norm_grad}, which underpin the local geometry of the basin and the adversarial stability analysis, respectively.

\section{Prior Art on the Conditioning of Structured Matrices}
\label{secA1}

The following lemmas are presented as paraphrased versions or special cases of key results from the literature. They are included here for completeness and to ensure the paper is self-contained.

The bandlimited approximation of functions of bounded variation has been widely studied in previous works \cite{ferreira2023second,ferreira2023conditionNumber,vaaler2023NumberLattice}. These results are of crucial importance to our analysis as they provide explicit constructions and bounds for bandlimited approximations that majorize or minorize the power spectral density $P_g$.
For completeness, we recall the main result from the literature in the following lemma. It includes explicit constructions of bandlimited majorants and minorants, along with bounds on their residuals in terms of the total variation and the bandwidth parameter $\beta$.

\begin{lemma}[{A special case of \cite[Theorem~3]{ferreira2023second}}]
\label{lem:ferreira_thm_3}
Suppose that $P_g$ is a function of bounded variation, and that $P_g$ and $f \mapsto 4 \pi^2 f^2 P_g(f)$ are absolutely integrable.
Let $\beta > 0$.
Then there exist $\beta$-bandlimited approximations of $P_g$ by a majorant $C_{+}$ and a minorant $C_{-}$ that satisfy
\begin{align*}
    C_+(f) \geq P_g(f), \quad C_-(f) \leq P_g(f), \quad \forall f \in \mathbb{R},
\end{align*}
and $ \widehat{C}_+(u) = \widehat{C}_-(u) = 0$ for $|u| \geq \beta$. Furthermore,
\begin{align*}
    \int_{-\infty}^\infty \left( C_+(f)-P_g(f) \right) df
    &= \int_{-\infty}^\infty \left( P_g(f) - C_-(f) \right) df
    \leq \tfrac{2}{3} V(P_g) \beta^{-1}, \\
    \int_{-\infty}^\infty 4 \pi^2 f^2\left( C_+(f)-P_g(f) \right) df
    &= \int_{-\infty}^\infty 4 \pi^2 f^2\left( P_g(f) - C_-(f) \right) df
    \leq \tfrac{1}{2} V(P_{g'}) \beta^{-1}.
\end{align*}
In addition, $\widehat{C}_+$ and $\widehat{C}_-$ are twice differentiable.
Furthermore, if $P_g$ is even symmetric, then $C_+$ and $C_-$ are also even symmetric.
\end{lemma}

\begin{proof}
We begin with the explicit construction of $C_+$ and $C_-$ from \cite[Theorem~3]{ferreira2023second}:
\begin{subequations}
\begin{align*}
    C_{+}(f) &= (P_g \ast J_\beta)(f) + (2\beta)^{-1} \int_{-\infty}^{\infty} K_\beta (f-u)\, |dP_g(u)|, \\
    C_{-}(f) &= (P_g \ast J_\beta)(f) - (2\beta)^{-1} \int_{-\infty}^{\infty} K_\beta (f-u)\, |dP_g(u)|,
\end{align*}
\end{subequations}
where $J_\beta$ and $K_\beta$ are even, bandlimited auxiliary functions defined in \cite{ferreira2023second}. These constructions satisfy the properties listed in Lemma~\ref{lem:ferreira_thm_3}, except for the even symmetry. Now consider the special case where $g$ is real. In this case, the magnitude of the Fourier transform is even, implying that $P_g$ is even symmetric. Furthermore, since $|dP_g|$ represents the absolute derivative measure of an even function, it is also symmetric. Therefore, by construction, $C_+$ and $C_-$ inherit this even symmetry.

Moreover, Lemma~\ref{lem:ferreira_thm_3} guarantees that $C_+(f)-P_g(f)$ and $P_g(f)-C_-(f)$ decay sufficiently fast at high frequencies. Thus, by the Riemann-Lebesgue lemma, the Fourier transforms of $C_+$ and $C_-$ are twice differentiable.
\end{proof}

The following lemma provides bounds on the minimum and maximum singular values of $\bm{A}_{\bm{\tau}} = \bm{G}\bm{\Phi}_{\bm{\tau}}$.

\begin{lemma}[{A paraphrase of \cite[Theorem~1]{ferreira2023conditionNumber}}]
\label{lem:ferreira_thm1}
Using the definitions of $\bm{\Phi}_{\bm{\tau}}$ and $\bm{G}$ in \eqref{eq:Phi_tau_def} and \eqref{eq:G_def}, and of the minimum separation $\Delta$ in \eqref{eq:min_sep}, the following holds for all $\bm{\tau} \in \mathbb{T}^K$:
\[
 \sigma_{\max} \left( \bm{G}\bm{\Phi}_{\bm{\tau}} \right) \leq  \sqrt{T E_{g} \left(1 + \tfrac{1}{2} \rho_g \Delta^{-1} \right)}, \quad
 \sigma_{\min} \left(\bm{G}\bm{\Phi}_{\bm{\tau}} \right) \geq \sqrt{T E_{g} \left(1 - \tfrac{1}{2} \rho_g \Delta^{-1} \right)},
\]
where $E_g$ and $\rho_g$ are defined in \eqref{eq:def_energy} and \eqref{eq:def_rho_g}.
\end{lemma}

The following lemma is obtained as a corollary of Lemma~\ref{lem:ferreira_thm1}.

\begin{lemma}
\label{lem:cor_of_ferreira_thm1}
Let $\bm{\Lambda} \in \mathbb{C}^{N \times N}$ be the diagonal matrix satisfying $[\bm{\Lambda}]_{i,i} = -\mathsf{j} 2\pi f_i$ for $i \in [N]$. Then, for all $\bm{\tau} \in \mathbb{T}^K$,
\[
\norm{\bm{\Lambda}\bm{G}\bm{\Phi}_{\bm{\tau}}} \leq \sqrt{T E_{g'} \left(1 + \tfrac{1}{2} \rho_{g'} \Delta^{-1} \right)}
\]
and
\[
 \norm{\bm{\Lambda}^2\bm{G}\bm{\Phi}_{\bm{\tau}}} \leq \sqrt{T E_{g''} \left(1 + \tfrac{1}{2} \rho_{g''} \Delta^{-1} \right)},
\]
where $E_{g'}$, $E_{g''}$ are defined in \eqref{eq:def_energy} and $\rho_{g'}$, $\rho_{g''}$ are defined in \eqref{eq:def_rho_g}.
\end{lemma}

The next lemma is a special case of Lemma~\ref{lem:ferreira_thm1}.

\begin{lemma}
\label{lem:G_frob}
Consider the diagonal matrix $\bm{G} \in \mathbb{C}^{N \times N}$ as defined in \eqref{eq:G_def}. Then
\[
\norm{\bm{G}}_F \leq \sqrt{T E_g \left(1+\tfrac{1}{2}\rho_g \right)}.
\]
\end{lemma}

\begin{proof}
Let $\mathbf{1}_{N} \in \mathbb{C}^N$ denote the all-ones vector. Since $\bm{G}$ is diagonal,
\[
\norm{\bm{G}}_F^2
= \norm{\bm{G} \mathbf{1}_{N}}^2 = \sigma_{\max}^2(\bm{G} \mathbf{1}_{N}) = \sigma_{\max}^2(\bm{G} \bm{\Phi}_0),
\]
where the last equality identifies the all-ones vector with the structured matrix representing a single pulse located at $0$. It follows from Lemma~\ref{lem:ferreira_thm1} that
\[
\sigma_{\max}^2(\bm{G} \bm{\Phi}_0) \leq T E_g \left(1+\tfrac{1}{2}\rho_g \right),
\]
which concludes the proof.
\end{proof}

\section{Properties of the Khatri-Rao and Hadamard Products}
\label{secA2}
This appendix collects standard properties of the Khatri-Rao and Hadamard products that are used in our theoretical analysis. 

In the following lemma, we first recall the spectral norm bound for the Hadamard product.

\begin{lemma}[{A special case of \cite[Theorem~1]{zhan1997inequalities}}]
\label{lem:zhan_hadamard}
Let $\bm{A}, \bm{B} \in \mathbb{C}^{m\times n}$. Then
\[
\sigma_{\max}(\bm{A} \odot \bm{B}) \leq \min\{r_1(\bm{A}),\, c_1(\bm{A})\} \cdot \sigma_{\max}(\bm{B}),
\]
where $r_1(\bm{A})$ and $c_1(\bm{A})$ denote the maximum Euclidean row and column lengths of $\bm{A}$, respectively.
\end{lemma}

The next lemma relates the Khatri-Rao product to the Hadamard product via the Gramian.

\begin{lemma}[{A paraphrase of \cite[Proposition~6.4.2]{rao1998matrix} and \cite{bro1998multi}}]
\label{lem:khatri_hadamard}
Let $\bm{A}, \bm{B} \in \mathbb{C}^{m\times n}$ be two complex-valued matrices with the same number of columns. Then
\[
(\bm{A} \ast \bm{B})^\mathsf{H} (\bm{A} \ast \bm{B}) = (\bm{A}^\mathsf{H} \bm{A}) \odot (\bm{B}^\mathsf{H} \bm{B}).
\]
\end{lemma}
The next following lemma is a corollary of Lemma~\ref{lem:zhan_hadamard} and Lemma~\ref{lem:khatri_hadamard}.

\begin{lemma}
\label{lem:khatri_spectral_norm}
For two complex matrices $\bm{A}, \bm{B} \in \mathbb{C}^{m\times n}$,
\[
\norm{\bm{A} \ast \bm{B}} \leq \norm{\bm{A}} \cdot \norm{\bm{B}}.
\]
\end{lemma}

\begin{proof}
Recall that for any matrix $\bm{M} \in \mathbb{C}^{m\times n}$, the spectral norm satisfies
\[
\norm{\bm{M}}^2 = \sigma_{\max}^2(\bm{M}) = \lambda_{\max}(\bm{M}^\mathsf{H}\bm{M}),
\]
where $\lambda_{\max}$ denotes the largest eigenvalue. Hence,
\[
\norm{\bm{A} \ast \bm{B}}^2 = \lambda_{\max}\!\left((\bm{A} \ast \bm{B})^\mathsf{H}(\bm{A} \ast \bm{B})\right).
\]
By using Lemma~\ref{lem:khatri_hadamard},
\[
\lambda_{\max}\!\left((\bm{A} \ast \bm{B})^\mathsf{H}(\bm{A} \ast \bm{B})\right)
= \lambda_{\max}\!\left((\bm{A}^\mathsf{H}\bm{A}) \odot (\bm{B}^\mathsf{H}\bm{B})\right)
= \sigma_{\max}^2(\bm{A} \odot \bm{B}).
\]
Furthermore, by Lemma~\ref{lem:zhan_hadamard},
\[
\sigma_{\max}^2(\bm{A} \odot \bm{B}) \leq \sigma_{\max}^2(\bm{A}) \cdot \sigma_{\max}^2(\bm{B}),
\]
which concludes the proof.
\end{proof}

\section{Gradient, Jacobian and Hessian Computations for the VarProSD Objective}
\label{secA3}

This appendix collects the gradient, Jacobian Gramian, Hessian, and cross-derivative of the VarProSD objective $\ell(\bm{\gamma})$ defined in \eqref{eq:loss_function} that are used throughout the proofs of the main results. We first derive the expression for the gradient (Lemma~\ref{lem:expression_for_q_alg}), then the expression for Jacobian Gramian of the residual in Lemma~\ref{lem:expression_for_jac_alg}, then the Hessian decomposition. Finally, we derive the cross-derivative with respect to the noise (Lemma~\ref{lem:jacobian_iota}).

\subsection{Proof of Lemma~\ref{lem:expression_for_q_alg}}
\label{subsecA3:grad}

Recall that the cost function in \eqref{eq:loss_function} can be written as
\[
\ell(\bm{\gamma}) = \frac{1}{2L}\norm{\bm{P}_{\bm{\gamma}}^\perp \bm{Y}}_\mathrm{F}^2 = \frac{1}{2L}\Big(\norm{\bm{Y}}_\mathrm{F}^2 - \mathrm{Trace}(\bm{Y}\bm{Y}^\mathsf{H}\bm{P}_{\bm{\gamma}})\Big).
\]
Taking the partial derivative with respect to $\gamma_k$,
\begin{align}
\frac{\partial \ell(\bm{\gamma})}{\partial \gamma_k}
= -\frac{1}{2L}\,\frac{\partial}{\partial \gamma_k}\mathrm{Trace}\!\left(\bm{Y}\bm{Y}^\mathsf{H}\bm{P}_{\bm{\gamma}}\right).
\label{eq:first_deriv}
\end{align}
Since $\bm{Y}\bm{Y}^\mathsf{H}$ does not depend on $\gamma_k$, we only need to compute the derivative of the projection $\bm{P}_{\bm{\gamma}}$. Extending the Fr{\'e}chet derivative of the orthogonal projection in \cite[Lemma~4.1]{golub1973differentiation} to the complex-valued case by replacing the transpose with the Hermitian transpose, one can obtain
\begin{align}
\frac{\partial \bm{P}_{\bm{\gamma}}}{\partial \gamma_k}
&= \bm{P}_{\bm{\gamma}}^\perp\,\frac{\partial \bm{A}_{\bm{\gamma}}}{\partial \gamma_k}\,\bm{A}_{\bm{\gamma}}^\dagger
 + \left(\bm{P}_{\bm{\gamma}}^\perp\,\frac{\partial \bm{A}_{\bm{\gamma}}}{\partial \gamma_k}\,\bm{A}_{\bm{\gamma}}^\dagger\right)^\mathsf{H} \nonumber\\
&= \bm{P}_{\bm{\gamma}}^\perp\bm{\Lambda}\bm{A}_{\bm{\gamma}}\bm{e}_k\bm{e}_k^\mathsf{T}\bm{A}_{\bm{\gamma}}^\dagger
 + \left(\bm{P}_{\bm{\gamma}}^\perp\bm{\Lambda}\bm{A}_{\bm{\gamma}}\bm{e}_k\bm{e}_k^\mathsf{T}\bm{A}_{\bm{\gamma}}^\dagger\right)^\mathsf{H},
\label{eq:projection_deriv}
\end{align}
where we used
\begin{align}
\label{eq:Phi_gamma_deriv}
\frac{\partial \bm{A}_{\bm{\gamma}}}{\partial \gamma_k}
= \bm{G}\bm{\Lambda}\bm{\Phi}_{\bm{\gamma}}\bm{e}_k\bm{e}_k^\mathsf{T}
= \bm{\Lambda}\bm{A}_{\bm{\gamma}}\bm{e}_k\bm{e}_k^\mathsf{T}.
\end{align}
Substituting \eqref{eq:projection_deriv} into \eqref{eq:first_deriv} gives
\[
\frac{\partial \ell(\bm{\gamma})}{\partial \gamma_k}
= -\frac{1}{2L}\,\mathrm{Trace}\!\left(\bm{Y}\bm{Y}^\mathsf{H}\bm{P}_{\bm{\gamma}}^\perp\bm{\Lambda}\bm{A}_{\bm{\gamma}}\bm{e}_k\bm{e}_k^\mathsf{T}\bm{A}_{\bm{\gamma}}^\dagger
+ \left(\bm{Y}\bm{Y}^\mathsf{H}\bm{P}_{\bm{\gamma}}^\perp\bm{\Lambda}\bm{A}_{\bm{\gamma}}\bm{e}_k\bm{e}_k^\mathsf{T}\bm{A}_{\bm{\gamma}}^\dagger\right)^\mathsf{H}\right).
\]
Using $\mathrm{Trace}(\bm{M}+\bm{M}^\mathsf{H}) = 2\operatorname{Re}[\mathrm{Trace}(\bm{M})]$ for any square matrix $\bm{M}$, together with the cyclic property of the trace,
\begin{align}
\label{eq:first_deriv_gamma}
\frac{\partial \ell(\bm{\gamma})}{\partial \gamma_k}
&= -\frac{1}{L}\operatorname{Re}\!\left[\mathrm{Trace}\!\left(\bm{Y}^\mathsf{H}\bm{P}_{\bm{\gamma}}^\perp\bm{\Lambda}\bm{A}_{\bm{\gamma}}\bm{e}_k\bm{e}_k^\mathsf{T}\bm{A}_{\bm{\gamma}}^\dagger\bm{Y}\right)\right] \nonumber\\
&= -\frac{1}{L}\operatorname{Re}\!\left[\bm{e}_k^\mathsf{T}\bm{A}_{\bm{\gamma}}^\dagger\bm{Y}\bm{Y}^\mathsf{H}\bm{P}_{\bm{\gamma}}^\perp\bm{\Lambda}\bm{A}_{\bm{\gamma}}\bm{e}_k\right],
\end{align}
which gives the assertion of Lemma~\ref{lem:expression_for_q_alg} and concludes the proof.

\subsection{Jacobian Gramian of the residual}
\label{subsecA3:jac}

To derive the Hessian of the VarProSD objective, a natural approach is to first express the objective in residual form and compute the Jacobian of that residual, since the Hessian of a least-squares objective decomposes into a Jacobian Gramian term and a residual curvature term. Recall from \eqref{eq:loss_function} that the objective admits the residual representation
\begin{align*}
\bm f(\bm{\gamma})
=
\mathrm{vec}\!\left(\bm P_{\bm{\gamma}}^\perp \bm Y \right)
\in \mathbb{C}^{NL},
\end{align*}
whose Jacobian with respect to $\bm{\gamma}$ is
\begin{equation}
\label{eq:jacobian_gn_alg}
\bm J(\bm{\gamma})
=
\frac{\partial \bm f(\bm{\gamma})}{\partial \bm{\gamma}}
=
\Big[
\tfrac{\partial \bm f}{\partial \gamma_1},
\ldots,
\tfrac{\partial \bm f}{\partial \gamma_K}
\Big]
\in \mathbb{C}^{(NL)\times K}.
\end{equation}
In this subsection we compute the Jacobian Gramian $\bm{J}(\bm{\gamma})^\mathsf{H}\bm{J}(\bm{\gamma})$ in closed form, and the residual curvature term is computed in the Hessian decomposition that follows.

\begin{lemma}
\label{lem:expression_for_jac_alg}
Suppose that all entries of $\bm{\gamma}$ are distinct and $\bm{\Lambda}$ is defined as in \eqref{def:lambda}. Then, the Jacobian Gramian admits the following expression.
\begin{multline*}
\bm{J}(\bm{\gamma})^\mathsf{H}\bm{J}(\bm{\gamma})
=
[\bm{A}_{\bm{\gamma}}^\mathsf{H}
\bm{\Lambda}^\mathsf{H}
\bm{P}_{\bm{\gamma}}^\perp
\bm{\Lambda}
\bm{A}_{\bm{\gamma}}
\odot
(\bm{A}_{\bm{\gamma}}^\dagger
\bm{Y}\bm{Y}^\mathsf{H}
\bm{A}_{\bm{\gamma}}^{\dagger\mathsf{H}})^\mathsf{T}]
\nonumber\\
+
[\bm{A}_{\bm{\gamma}}^\dagger
\bm{A}_{\bm{\gamma}}^{\dagger\mathsf{H}}
\odot
(\bm{A}_{\bm{\gamma}}^\mathsf{H}
\bm{\Lambda}^\mathsf{H}
\bm{P}_{\bm{\gamma}}^\perp
\bm{Y}\bm{Y}^\mathsf{H}
\bm{P}_{\bm{\gamma}}^\perp
\bm{\Lambda}
\bm{A}_{\bm{\gamma}})^\mathsf{T} ].
\end{multline*}
\end{lemma}

\begin{proof}
    Using the definition of the Jacobian of the residual in \eqref{eq:jacobian_gn_alg}, the $(j,l)$-entry of the Jacobian Gramian can be written as
\begin{align*}
[\bm{J}(\bm{\gamma})^\mathsf{H}\bm{J}(\bm{\gamma})]_{j,l}
&= \mathrm{Trace}\!\left(\left(\frac{\partial \bm{P}_{\bm{\gamma}}^\perp\bm{Y}}{\partial \gamma_j}\right)^{\!\mathsf{H}}\frac{\partial \bm{P}_{\bm{\gamma}}^\perp\bm{Y}}{\partial \gamma_l}\right) \nonumber\\
&= \mathrm{Trace}\!\left(\left(\frac{\partial \bm{P}_{\bm{\gamma}}\bm{Y}}{\partial \gamma_j}\right)^{\!\mathsf{H}}\frac{\partial \bm{P}_{\bm{\gamma}}\bm{Y}}{\partial \gamma_l}\right),
\end{align*}
where we used $\bm{P}_{\bm{\gamma}}^\perp = \bm{I}_N - \bm{P}_{\bm{\gamma}}$. Substituting the projection derivative from \eqref{eq:projection_deriv} and applying the cyclic property of the trace,
\begin{multline*}
[\bm{J}(\bm{\gamma})^\mathsf{H}\bm{J}(\bm{\gamma})]_{j,l}
= \mathrm{Trace}\!\left(\bm{Y}\bm{Y}^\mathsf{H}\!\left[\left(\bm{P}_{\bm{\gamma}}^\perp\bm{\Lambda}\bm{A}_{\bm{\gamma}}\bm{e}_j\bm{e}_j^\mathsf{T}\bm{A}_{\bm{\gamma}}^\dagger\right)^{\!\mathsf{H}} + \bm{P}_{\bm{\gamma}}^\perp\bm{\Lambda}\bm{A}_{\bm{\gamma}}\bm{e}_j\bm{e}_j^\mathsf{T}\bm{A}_{\bm{\gamma}}^\dagger\right] \right.\\
\qquad \times \left.\left[\bm{P}_{\bm{\gamma}}^\perp\bm{\Lambda}\bm{A}_{\bm{\gamma}}\bm{e}_l\bm{e}_l^\mathsf{T}\bm{A}_{\bm{\gamma}}^\dagger + \left(\bm{P}_{\bm{\gamma}}^\perp\bm{\Lambda}\bm{A}_{\bm{\gamma}}\bm{e}_l\bm{e}_l^\mathsf{T}\bm{A}_{\bm{\gamma}}^\dagger\right)^{\!\mathsf{H}}\right]\right).
\end{multline*}
Expanding the product into four trace terms,
\begin{align*}
[\bm{J}(\bm{\gamma})^\mathsf{H}\bm{J}(\bm{\gamma})]_{j,l}
&= \bm{e}_j^\mathsf{T}\bm{A}_{\bm{\gamma}}^\mathsf{H}\bm{\Lambda}^\mathsf{H}\bm{P}_{\bm{\gamma}}^\perp\bm{P}_{\bm{\gamma}}^\perp\bm{\Lambda}\bm{A}_{\bm{\gamma}}\bm{e}_l \cdot \bm{e}_l^\mathsf{T}\bm{A}_{\bm{\gamma}}^\dagger\bm{Y}\bm{Y}^\mathsf{H}\bm{A}_{\bm{\gamma}}^{\dagger\mathsf{H}}\bm{e}_j \\
&\quad+ \bm{e}_j^\mathsf{T}\bm{A}_{\bm{\gamma}}^\mathsf{H}\bm{\Lambda}^\mathsf{H}\bm{P}_{\bm{\gamma}}^\perp\bm{A}_{\bm{\gamma}}^{\dagger\mathsf{H}}\bm{e}_l \cdot \bm{e}_l^\mathsf{T}\bm{A}_{\bm{\gamma}}^\mathsf{H}\bm{\Lambda}^\mathsf{H}\bm{P}_{\bm{\gamma}}^\perp\bm{Y}\bm{Y}^\mathsf{H}\bm{A}_{\bm{\gamma}}^{\dagger\mathsf{H}}\bm{e}_j \\
&\quad+ \bm{e}_j^\mathsf{T}\bm{A}_{\bm{\gamma}}^\dagger\bm{P}_{\bm{\gamma}}^\perp\bm{\Lambda}\bm{A}_{\bm{\gamma}}\bm{e}_l \cdot \bm{e}_l^\mathsf{T}\bm{A}_{\bm{\gamma}}^\dagger\bm{Y}\bm{Y}^\mathsf{H}\bm{P}_{\bm{\gamma}}^\perp\bm{\Lambda}\bm{A}_{\bm{\gamma}}\bm{e}_j \\
&\quad+ \bm{e}_j^\mathsf{T}\bm{A}_{\bm{\gamma}}^\dagger\bm{A}_{\bm{\gamma}}^{\dagger\mathsf{H}}\bm{e}_l \cdot \bm{e}_l^\mathsf{T}\bm{A}_{\bm{\gamma}}^\mathsf{H}\bm{\Lambda}^\mathsf{H}\bm{P}_{\bm{\gamma}}^\perp\bm{Y}\bm{Y}^\mathsf{H}\bm{P}_{\bm{\gamma}}^\perp\bm{\Lambda}\bm{A}_{\bm{\gamma}}\bm{e}_j.
\end{align*}
Identifying each summand as an entry of a Hadamard product, and using $(\bm{P}_{\bm{\gamma}}^\perp)^2 = \bm{P}_{\bm{\gamma}}^\perp$,
\begin{align*}
[\bm{J}(\bm{\gamma})^\mathsf{H}\bm{J}(\bm{\gamma})]_{j,l}
&= \big[\bm{A}_{\bm{\gamma}}^\mathsf{H}\bm{\Lambda}^\mathsf{H}\bm{P}_{\bm{\gamma}}^\perp\bm{\Lambda}\bm{A}_{\bm{\gamma}} \odot (\bm{A}_{\bm{\gamma}}^\dagger\bm{Y}\bm{Y}^\mathsf{H}\bm{A}_{\bm{\gamma}}^{\dagger\mathsf{H}})^\mathsf{T}\big]_{j,l} \\
&\quad+ \big[\bm{A}_{\bm{\gamma}}^\mathsf{H}\bm{\Lambda}^\mathsf{H}\bm{P}_{\bm{\gamma}}^\perp\bm{A}_{\bm{\gamma}}^{\dagger\mathsf{H}} \odot (\bm{A}_{\bm{\gamma}}^\mathsf{H}\bm{\Lambda}^\mathsf{H}\bm{P}_{\bm{\gamma}}^\perp\bm{Y}\bm{Y}^\mathsf{H}\bm{A}_{\bm{\gamma}}^{\dagger\mathsf{H}})^\mathsf{T}\big]_{j,l} \\
&\quad+ \big[\bm{A}_{\bm{\gamma}}^\dagger\bm{P}_{\bm{\gamma}}^\perp\bm{\Lambda}\bm{A}_{\bm{\gamma}} \odot (\bm{A}_{\bm{\gamma}}^\dagger\bm{Y}\bm{Y}^\mathsf{H}\bm{P}_{\bm{\gamma}}^\perp\bm{\Lambda}\bm{A}_{\bm{\gamma}})^\mathsf{T}\big]_{j,l} \\
&\quad+ \big[\bm{A}_{\bm{\gamma}}^\dagger\bm{A}_{\bm{\gamma}}^{\dagger\mathsf{H}} \odot (\bm{A}_{\bm{\gamma}}^\mathsf{H}\bm{\Lambda}^\mathsf{H}\bm{P}_{\bm{\gamma}}^\perp\bm{Y}\bm{Y}^\mathsf{H}\bm{P}_{\bm{\gamma}}^\perp\bm{\Lambda}\bm{A}_{\bm{\gamma}})^\mathsf{T}\big]_{j,l}.
\end{align*}
The second and third terms vanish because $\bm{A}_{\bm{\gamma}}^\dagger\bm{P}_{\bm{\gamma}}^\perp = \bm{0}$. Collecting the remaining first and fourth terms over all $j,l$ gives the assertion of Lemma~\ref{lem:expression_for_jac_alg}, which concludes the proof.

\end{proof}

\subsection{Hessian expression for the VarProSD objective}
We now state and prove the Hessian decomposition for the VarProSD objective. We first index the entries of the residual $\bm{f}(\bm{\gamma}) = \mathrm{vec}(\bm{P}_{\bm{\gamma}}^\perp\bm{Y}) \in \mathbb{C}^{NL}$ by sample $n \in [N]$ and snapshot $\ell \in [L]$ such that
\begin{align}
\label{eq:r_nl_def}
r_{n,\ell}(\bm{\gamma}) = \bm{e}_n^\mathsf{T}\bm{P}_{\bm{\gamma}}^\perp\bm{y}_\ell, \quad n \in [N],\, \ell \in [L],
\end{align}
with $r_{n,\ell}^{\mathrm{Re}} = \operatorname{Re}(r_{n,\ell})$ and $r_{n,\ell}^{\mathrm{Im}} = \operatorname{Im}(r_{n,\ell})$.
\begin{lemma}
\label{lem:expression_for_hess_alg}
Suppose that all entries of $\bm{\gamma}$ are distinct and let $\bm{\Lambda}$ be defined as in \eqref{def:lambda}. Then the Hessian of $\ell(\bm{\gamma})$ admits the decomposition
\begin{align}
\label{eq:hess_decomp}
\nabla^2 \ell(\bm{\gamma}) = \frac{1}{L}\operatorname{Re}\!\left(\bm{J}(\bm{\gamma})^\mathsf{H}\bm{J}(\bm{\gamma})\right) + \frac{1}{L}\sum_{n=1}^N\sum_{\ell=1}^L\Big(r_{n,\ell}^{\mathrm{Re}}\,\nabla_{\bm{\gamma}}^2 r_{n,\ell}^{\mathrm{Re}} + r_{n,\ell}^{\mathrm{Im}}\,\nabla_{\bm{\gamma}}^2 r_{n,\ell}^{\mathrm{Im}}\Big),
\end{align}
where $\bm{J}(\bm{\gamma})^\mathsf{H}\bm{J}(\bm{\gamma})$ is the Jacobian Gramian whose closed-form expression is given in Lemma~\ref{lem:expression_for_jac_alg}. Furthermore, the residual curvature term in \eqref{eq:hess_decomp} admits the explicit expression
\begin{align}
\label{eq:residual_curv}
\sum_{n,\ell}\Big(r_{n,\ell}^{\mathrm{Re}}\,\nabla_{\bm{\gamma}}^2 r_{n,\ell}^{\mathrm{Re}} + r_{n,\ell}^{\mathrm{Im}}\,\nabla_{\bm{\gamma}}^2 r_{n,\ell}^{\mathrm{Im}}\Big) &= \operatorname{Re}\!\left[\bm{A}_{\bm{\gamma}}^\dagger\bm{\Lambda}\bm{A}_{\bm{\gamma}} \odot (\bm{A}_{\bm{\gamma}}^\dagger\bm{Y}\bm{Y}^\mathsf{H}\bm{P}_{\bm{\gamma}}^\perp\bm{\Lambda}\bm{A}_{\bm{\gamma}})^\mathsf{T}\right] \nonumber\\
&\quad - \operatorname{Re}\!\left[\bm{A}_{\bm{\gamma}}^\dagger\bm{Y}\bm{Y}^\mathsf{H}\bm{P}_{\bm{\gamma}}^\perp\bm{\Lambda}^2\bm{A}_{\bm{\gamma}} \odot \bm{I}_K\right] \nonumber\\
&\quad + \operatorname{Re}\!\left[(\bm{A}_{\bm{\gamma}}^\dagger\bm{\Lambda}\bm{A}_{\bm{\gamma}})^\mathsf{T} \odot \bm{A}_{\bm{\gamma}}^\dagger\bm{Y}\bm{Y}^\mathsf{H}\bm{P}_{\bm{\gamma}}^\perp\bm{\Lambda}\bm{A}_{\bm{\gamma}}\right] \nonumber\\
&\quad - 2\operatorname{Re}\!\left[\bm{A}_{\bm{\gamma}}^\dagger\bm{A}_{\bm{\gamma}}^{\dagger\mathsf{H}} \odot (\bm{A}_{\bm{\gamma}}^\mathsf{H}\bm{\Lambda}^\mathsf{H}\bm{P}_{\bm{\gamma}}^\perp\bm{Y}\bm{Y}^\mathsf{H}\bm{P}_{\bm{\gamma}}^\perp\bm{\Lambda}\bm{A}_{\bm{\gamma}})^\mathsf{T}\right].
\end{align}
\end{lemma}
\begin{proof}
The classical Hessian decomposition for nonlinear least squares applies to real-valued residuals \cite{dennis1996numerical}. Since the residual $\bm{f}(\bm{\gamma}) = \mathrm{vec}(\bm{P}_{\bm{\gamma}}^\perp\bm{Y}) \in \mathbb{C}^{NL}$ is complex-valued, we begin by constructing a real-equivalent reformulation that allows the classical identity to be applied.

We define the real-equivalent residual
\[
\tilde{\bm{f}}(\bm{\gamma}) = \begin{bmatrix}\operatorname{Re}(\bm{f}(\bm{\gamma})) \\ \operatorname{Im}(\bm{f}(\bm{\gamma}))\end{bmatrix} \in \mathbb{R}^{2NL},
\]
and the corresponding real-equivalent Jacobian
\[
\tilde{\bm{J}}(\bm{\gamma}) = \begin{bmatrix}\partial \operatorname{Re}(\bm{f}(\bm{\gamma}))/\partial \gamma_1 & \cdots & \partial \operatorname{Re}(\bm{f}(\bm{\gamma}))/\partial \gamma_K \\[3pt] \partial \operatorname{Im}(\bm{f}(\bm{\gamma}))/\partial \gamma_1 & \cdots & \partial \operatorname{Im}(\bm{f}(\bm{\gamma}))/\partial \gamma_K\end{bmatrix} \in \mathbb{R}^{2NL \times K}.
\]
Since $\ell(\bm{\gamma}) = \frac{1}{2L}\norm{\bm{f}(\bm{\gamma})}_2^2 = \frac{1}{2L}\|\tilde{\bm{f}}(\bm{\gamma})\|_2^2$, the standard nonlinear least-squares Hessian identity \cite[Ch.~10]{dennis1996numerical} applied to $\tilde{\bm{f}}$ yields
\begin{align}
\label{eq:nls_hessian_real}
\nabla^2 \ell(\bm{\gamma}) = \frac{1}{L}\,\tilde{\bm{J}}(\bm{\gamma})^\mathsf{T}\tilde{\bm{J}}(\bm{\gamma}) + \frac{1}{L}\sum_{i=1}^{2NL}\tilde{f}_i(\bm{\gamma})\,\nabla_{\bm{\gamma}}^2 \tilde{f}_i(\bm{\gamma}).
\end{align}
To connect this real-equivalent identity to the complex Jacobian $\bm{J}(\bm{\gamma})$, we observe that for any pair of complex vectors $\bm{u},\bm{v} \in \mathbb{C}^{NL}$,
\begin{align}
\label{eq:complex_inner_product_identity}
\langle \operatorname{Re}(\bm{u}),\operatorname{Re}(\bm{v})\rangle + \langle \operatorname{Im}(\bm{u}),\operatorname{Im}(\bm{v})\rangle = \operatorname{Re}\!\left(\langle \bm{u},\bm{v}\rangle\right).
\end{align}
Applying \eqref{eq:complex_inner_product_identity} to the columns of $\bm{J}(\bm{\gamma})$, the $(m,n)$-entry of the real-equivalent Jacobian Gramian satisfies
\[
\big[\tilde{\bm{J}}(\bm{\gamma})^\mathsf{T}\tilde{\bm{J}}(\bm{\gamma})\big]_{m,n} = \operatorname{Re}\!\left(\left\langle \frac{\partial \bm{f}(\bm{\gamma})}{\partial \gamma_m},\frac{\partial \bm{f}(\bm{\gamma})}{\partial \gamma_n}\right\rangle\right) = \big[\operatorname{Re}(\bm{J}(\bm{\gamma})^\mathsf{H}\bm{J}(\bm{\gamma}))\big]_{m,n},
\]
so that
\begin{align}
\label{eq:real_jacob_gramian}
\tilde{\bm{J}}(\bm{\gamma})^\mathsf{T}\tilde{\bm{J}}(\bm{\gamma}) = \operatorname{Re}\!\left(\bm{J}(\bm{\gamma})^\mathsf{H}\bm{J}(\bm{\gamma})\right).
\end{align}
Substituting \eqref{eq:real_jacob_gramian} into \eqref{eq:nls_hessian_real} and re-expressing the second summand in terms of the real and imaginary parts of $r_{n,\ell}$ established earlier in \eqref{eq:r_nl_def} gives
\[
\nabla^2 \ell(\bm{\gamma}) = \frac{1}{L}\operatorname{Re}\!\left(\bm{J}(\bm{\gamma})^\mathsf{H}\bm{J}(\bm{\gamma})\right) + \frac{1}{L}\sum_{n=1}^N\sum_{\ell=1}^L\Big(r_{n,\ell}^{\mathrm{Re}}\,\nabla_{\bm{\gamma}}^2 r_{n,\ell}^{\mathrm{Re}} + r_{n,\ell}^{\mathrm{Im}}\,\nabla_{\bm{\gamma}}^2 r_{n,\ell}^{\mathrm{Im}}\Big),
\]
which establishes the decomposition \eqref{eq:hess_decomp}. We now turn to deriving the explicit expression \eqref{eq:residual_curv} for the residual curvature sum.

We begin by computing the first partial derivative of $r_{n,\ell}$ with respect to $\gamma_k$. Using $\bm{P}_{\bm{\gamma}}^\perp = \bm{I}_N - \bm{P}_{\bm{\gamma}}$ and the projection derivative \eqref{eq:projection_deriv},
\begin{align}
\label{eq:first_deriv_r}
\frac{\partial r_{n,\ell}}{\partial \gamma_k} = -\bm{e}_n^\mathsf{T}\frac{\partial \bm{P}_{\bm{\gamma}}}{\partial \gamma_k}\bm{y}_\ell = -\bm{e}_n^\mathsf{T}\big(\bm{T}_k + \bm{T}_k^\mathsf{H}\big)\bm{y}_\ell,
\end{align}
where we have introduced the shorthand
\begin{align}
\label{eq:Tk_def}
\bm{T}_k = \bm{P}_{\bm{\gamma}}^\perp\bm{\Lambda}\bm{A}_{\bm{\gamma}}\bm{e}_k\bm{e}_k^\mathsf{T}\bm{A}_{\bm{\gamma}}^\dagger.
\end{align}
Differentiating \eqref{eq:first_deriv_r} once more with respect to $\gamma_j$, taking the real part,
\begin{align}
\frac{\partial^2 \operatorname{Re}(r_{n,\ell})}{\partial \gamma_j\partial \gamma_k} &= -\operatorname{Re}\!\left(\bm{e}_n^\mathsf{T}\!\left(\frac{\partial \bm{T}_k}{\partial \gamma_j} + \left(\frac{\partial \bm{T}_k}{\partial \gamma_j}\right)^\mathsf{H}\right)\!\bm{y}_\ell\right) \nonumber \\
&= -2\operatorname{Re}\!\left(\operatorname{Trace}\!\left(\bm{e}_n^\mathsf{T}\,\mathcal{H}\!\left(\frac{\partial \bm{T}_k}{\partial \gamma_j}\right)\bm{y}_\ell\right)\right),
\label{eq:Re_second_deriv_compact}
\end{align}
where the Hermitian projection operator is defined by
\begin{align*}
\mathcal{H}(\bm{M}) := \tfrac{1}{2}\big(\bm{M} + \bm{M}^\mathsf{H}\big).
\end{align*}
We now introduce the real-valued Frobenius inner product
\begin{align}
\label{eq:real_inner_product_def}
\langle \bm{A},\bm{B}\rangle_\mathbb{R} := \operatorname{Re}\!\left(\operatorname{Trace}(\bm{A}^\mathsf{H}\bm{B})\right),
\end{align}
under which the operator $\mathcal{H}$ is self-adjoint, i.e., $\langle \mathcal{H}(\bm{A}),\bm{B}\rangle_\mathbb{R} = \langle \bm{A},\mathcal{H}(\bm{B})\rangle_\mathbb{R}$ for all complex matrices $\bm{A},\bm{B}$ of compatible dimensions. Rewriting the trace in \eqref{eq:Re_second_deriv_compact} as an inner product of the form \eqref{eq:real_inner_product_def} and applying the self-adjointness of $\mathcal{H}$, we obtain
\begin{align}
\label{eq:real_curv_term}
\frac{\partial^2 \operatorname{Re}(r_{n,\ell})}{\partial \gamma_j\partial \gamma_k} = -2\left\langle \frac{\partial \bm{T}_k}{\partial \gamma_j},\,\mathcal{H}(\bm{y}_\ell\bm{e}_n^\mathsf{T})\right\rangle_\mathbb{R}.
\end{align}
The same argument applied to the imaginary part of $r_{n,\ell}$ yields
\begin{align}
\label{eq:imag_curv_term}
\frac{\partial^2 \operatorname{Im}(r_{n,\ell})}{\partial \gamma_j\partial \gamma_k} = -2\left\langle \frac{\partial \bm{T}_k}{\partial \gamma_j},\,\mathcal{H}(-i\,\bm{y}_\ell\bm{e}_n^\mathsf{T})\right\rangle_\mathbb{R}.
\end{align}

Substituting \eqref{eq:real_curv_term} and \eqref{eq:imag_curv_term} into the residual curvature sum on the left-hand side of \eqref{eq:residual_curv} yields
\begin{align}
\sum_{n,\ell}\Big(r_{n,\ell}^{\mathrm{Re}}\nabla_{\bm{\gamma}}^2 r_{n,\ell}^{\mathrm{Re}} + r_{n,\ell}^{\mathrm{Im}}\nabla_{\bm{\gamma}}^2 r_{n,\ell}^{\mathrm{Im}}\Big) = -2\left\langle \frac{\partial \bm{T}_k}{\partial \gamma_j},\,\sum_{n,\ell}\!\left[r_{n,\ell}^{\mathrm{Re}}\mathcal{H}(\bm{y}_\ell\bm{e}_n^\mathsf{T}) + r_{n,\ell}^{\mathrm{Im}}\mathcal{H}(-i\,\bm{y}_\ell\bm{e}_n^\mathsf{T})\right]\right\rangle_\mathbb{R}.
\label{eq:curv_sum_with_inner_product}
\end{align}
By linearity of $\mathcal{H}$ on the real field, the real and imaginary contributions inside the inner product on the right-hand side of \eqref{eq:curv_sum_with_inner_product} combine into a single Hermitian-projected term as
\begin{align}
\label{eq:combining_real_imag}
r_{n,\ell}^{\mathrm{Re}}\mathcal{H}(\bm{y}_\ell\bm{e}_n^\mathsf{T}) + r_{n,\ell}^{\mathrm{Im}}\mathcal{H}(-i\,\bm{y}_\ell\bm{e}_n^\mathsf{T}) = \mathcal{H}\!\left(\big(r_{n,\ell}^{\mathrm{Re}} - i\,r_{n,\ell}^{\mathrm{Im}}\big)\,\bm{y}_\ell\bm{e}_n^\mathsf{T}\right) = \mathcal{H}\!\left(r_{n,\ell}^\ast\,\bm{y}_\ell\bm{e}_n^\mathsf{T}\right),
\end{align}
where $r_{n,\ell}^\ast = r_{n,\ell}^{\mathrm{Re}} - i\,r_{n,\ell}^{\mathrm{Im}}$ denotes the complex conjugate of $r_{n,\ell}$. Substituting the definition $r_{n,\ell} = \bm{e}_n^\mathsf{T}\bm{P}_{\bm{\gamma}}^\perp\bm{y}_\ell$ into \eqref{eq:combining_real_imag} and summing over $n$ and $\ell$ gives
\begin{align}
\sum_{n,\ell}\mathcal{H}\!\left(r_{n,\ell}^\ast\,\bm{y}_\ell\bm{e}_n^\mathsf{T}\right) &= \sum_{n,\ell}\mathcal{H}\!\left(\bm{y}_\ell\bm{y}_\ell^\mathsf{H}\bm{P}_{\bm{\gamma}}^\perp\bm{e}_n\bm{e}_n^\mathsf{T}\right) \nonumber \\
&= \mathcal{H}\!\left(\!\left(\sum_\ell \bm{y}_\ell\bm{y}_\ell^\mathsf{H}\right)\bm{P}_{\bm{\gamma}}^\perp\!\left(\sum_n \bm{e}_n\bm{e}_n^\mathsf{T}\right)\!\right) = \mathcal{H}\!\left(\bm{Y}\bm{Y}^\mathsf{H}\bm{P}_{\bm{\gamma}}^\perp\right),
\label{eq:H_YYH_proj}
\end{align}
where we used $\sum_\ell \bm{y}_\ell\bm{y}_\ell^\mathsf{H} = \bm{Y}\bm{Y}^\mathsf{H}$ and $\sum_n \bm{e}_n\bm{e}_n^\mathsf{T} = \bm{I}_N$. Substituting \eqref{eq:H_YYH_proj} into \eqref{eq:curv_sum_with_inner_product}, applying the self-adjointness of $\mathcal{H}$ once more, and expanding the Hermitian projection,
\begin{align}
\sum_{n,\ell}\Big(r_{n,\ell}^{\mathrm{Re}}\nabla_{\bm{\gamma}}^2 r_{n,\ell}^{\mathrm{Re}} + r_{n,\ell}^{\mathrm{Im}}\nabla_{\bm{\gamma}}^2 r_{n,\ell}^{\mathrm{Im}}\Big) = -\left(\!\left\langle \frac{\partial \bm{T}_k}{\partial \gamma_j},\,\bm{Y}\bm{Y}^\mathsf{H}\bm{P}_{\bm{\gamma}}^\perp\right\rangle_\mathbb{R} + \left\langle \frac{\partial \bm{T}_k}{\partial \gamma_j},\,\bm{P}_{\bm{\gamma}}^\perp\bm{Y}\bm{Y}^\mathsf{H}\right\rangle_\mathbb{R}\!\right).
\label{eq:curv_term_two_inner_prods}
\end{align}

It remains to compute the partial derivative of $\bm{T}_k$ defined in \eqref{eq:Tk_def}. Applying the product rule,
\begin{align}
\frac{\partial \bm{T}_k}{\partial \gamma_j}
&= \!\left(\frac{\partial \bm{P}_{\bm{\gamma}}^\perp}{\partial \gamma_j}\right)\!\bm{\Lambda}\bm{A}_{\bm{\gamma}}\bm{e}_k\bm{e}_k^\mathsf{T}\bm{A}_{\bm{\gamma}}^\dagger
+ \bm{P}_{\bm{\gamma}}^\perp\bm{\Lambda}\!\left(\frac{\partial \bm{A}_{\bm{\gamma}}}{\partial \gamma_j}\right)\!\bm{e}_k\bm{e}_k^\mathsf{T}\bm{A}_{\bm{\gamma}}^\dagger + \bm{P}_{\bm{\gamma}}^\perp\bm{\Lambda}\bm{A}_{\bm{\gamma}}\bm{e}_k\bm{e}_k^\mathsf{T}\!\left(\frac{\partial \bm{A}_{\bm{\gamma}}^\dagger}{\partial \gamma_j}\right).
\label{eq:Tk_deriv_expand}
\end{align}
The first two derivatives in \eqref{eq:Tk_deriv_expand} were computed earlier in \eqref{eq:projection_deriv} and \eqref{eq:Phi_gamma_deriv}. For the pseudoinverse derivative, we extend the Fr{\'e}chet identity from \cite[Theorem~4.3]{golub1973differentiation} to the complex-valued case by replacing the transpose with the Hermitian transpose, which gives
\begin{align}
\label{eq:pinv_deriv_full}
\frac{\partial \bm{A}_{\bm{\gamma}}^\dagger}{\partial \gamma_j}
&= -\bm{A}_{\bm{\gamma}}^\dagger\,\frac{\partial \bm{A}_{\bm{\gamma}}}{\partial \gamma_j}\,\bm{A}_{\bm{\gamma}}^\dagger
+ \bm{A}_{\bm{\gamma}}^\dagger\bm{A}_{\bm{\gamma}}^{\dagger\mathsf{H}}\,\frac{\partial \bm{A}_{\bm{\gamma}}^\mathsf{H}}{\partial \gamma_j}\,\bm{P}_{\bm{\gamma}}^\perp + \big(\bm{I}_K - \bm{A}_{\bm{\gamma}}^\dagger\bm{A}_{\bm{\gamma}}\big)\,\frac{\partial \bm{A}_{\bm{\gamma}}^\mathsf{H}}{\partial \gamma_j}\,\bm{A}_{\bm{\gamma}}^{\dagger\mathsf{H}}\bm{A}_{\bm{\gamma}}^\dagger.
\end{align}
Since the entries of $\bm{\gamma}$ are distinct by assumption, $\bm{A}_{\bm{\gamma}}$ is full column rank and therefore $\bm{A}_{\bm{\gamma}}^\dagger\bm{A}_{\bm{\gamma}} = \bm{I}_K$, which causes the third term in \eqref{eq:pinv_deriv_full} to vanish. Substituting the resulting expression for $\partial \bm{A}_{\bm{\gamma}}^\dagger/\partial \gamma_j$ together with \eqref{eq:projection_deriv} and \eqref{eq:Phi_gamma_deriv} back into \eqref{eq:Tk_deriv_expand} yields the closed-form expansion
\begin{align}
\frac{\partial \bm{T}_k}{\partial \gamma_j}
&= -\bm{P}_{\bm{\gamma}}^\perp\bm{\Lambda}\bm{A}_{\bm{\gamma}}\bm{e}_j\bm{e}_j^\mathsf{T}\bm{A}_{\bm{\gamma}}^\dagger\bm{\Lambda}\bm{A}_{\bm{\gamma}}\bm{e}_k\bm{e}_k^\mathsf{T}\bm{A}_{\bm{\gamma}}^\dagger
- \bm{A}_{\bm{\gamma}}^{\dagger\mathsf{H}}\bm{e}_j\bm{e}_j^\mathsf{T}\bm{A}_{\bm{\gamma}}^\mathsf{H}\bm{\Lambda}^\mathsf{H}\bm{P}_{\bm{\gamma}}^\perp\bm{\Lambda}\bm{A}_{\bm{\gamma}}\bm{e}_k\bm{e}_k^\mathsf{T}\bm{A}_{\bm{\gamma}}^\dagger \nonumber \\
&\quad + \bm{P}_{\bm{\gamma}}^\perp\bm{\Lambda}^2\bm{A}_{\bm{\gamma}}\bm{e}_j\bm{e}_j^\mathsf{T}\bm{e}_k\bm{e}_k^\mathsf{T}\bm{A}_{\bm{\gamma}}^\dagger
- \bm{P}_{\bm{\gamma}}^\perp\bm{\Lambda}\bm{A}_{\bm{\gamma}}\bm{e}_k\bm{e}_k^\mathsf{T}\bm{A}_{\bm{\gamma}}^\dagger\bm{\Lambda}\bm{A}_{\bm{\gamma}}\bm{e}_j\bm{e}_j^\mathsf{T}\bm{A}_{\bm{\gamma}}^\dagger \nonumber \\
&\quad + \bm{P}_{\bm{\gamma}}^\perp\bm{\Lambda}\bm{A}_{\bm{\gamma}}\bm{e}_k\bm{e}_k^\mathsf{T}\bm{A}_{\bm{\gamma}}^\dagger\bm{A}_{\bm{\gamma}}^{\dagger\mathsf{H}}\bm{e}_j\bm{e}_j^\mathsf{T}\bm{A}_{\bm{\gamma}}^\mathsf{H}\bm{\Lambda}^\mathsf{H}\bm{P}_{\bm{\gamma}}^\perp.
\label{eq:Tk_deriv_explicit}
\end{align}

We now substitute the explicit expression \eqref{eq:Tk_deriv_explicit} into each of the two inner products in \eqref{eq:curv_term_two_inner_prods}. Consider first the inner product $\langle \partial \bm{T}_k/\partial \gamma_j,\,\bm{Y}\bm{Y}^\mathsf{H}\bm{P}_{\bm{\gamma}}^\perp\rangle_\mathbb{R}$. Each of the five terms in \eqref{eq:Tk_deriv_explicit} contributes a separate trace, which by the cyclic property of the trace can be identified as the $(j,k)$-entry of a Hadamard product, yielding
\begin{align}
\label{eq:first_inner_contribs}
\langle \tfrac{\partial \bm{T}_k}{\partial \gamma_j},\,\bm{Y}\bm{Y}^\mathsf{H}\bm{P}_{\bm{\gamma}}^\perp\rangle_\mathbb{R}
&= -\big[\bm{A}_{\bm{\gamma}}^\dagger\bm{\Lambda}\bm{A}_{\bm{\gamma}} \odot (\bm{A}_{\bm{\gamma}}^\dagger\bm{Y}\bm{Y}^\mathsf{H}\bm{P}_{\bm{\gamma}}^\perp\bm{\Lambda}\bm{A}_{\bm{\gamma}})^\mathsf{T}\big]_{j,k} \nonumber \\
&\quad - \big[\bm{A}_{\bm{\gamma}}^\mathsf{H}\bm{\Lambda}^\mathsf{H}\bm{P}_{\bm{\gamma}}^\perp\bm{A}_{\bm{\gamma}}^\dagger\bm{Y}\bm{Y}^\mathsf{H}\bm{P}_{\bm{\gamma}}^\perp\bm{A}_{\bm{\gamma}}^{\dagger\mathsf{H}} \odot \bm{A}_{\bm{\gamma}}^\mathsf{H}\bm{\Lambda}^\mathsf{H}\bm{P}_{\bm{\gamma}}^\perp\bm{\Lambda}\bm{A}_{\bm{\gamma}}\big]_{j,k} \nonumber \\
&\quad + \delta_{j,k}\,\big[\bm{A}_{\bm{\gamma}}^\dagger\bm{Y}\bm{Y}^\mathsf{H}\bm{P}_{\bm{\gamma}}^\perp\bm{\Lambda}^2\bm{A}_{\bm{\gamma}}\big]_{j,k} \nonumber \\
&\quad - \big[\bm{A}_{\bm{\gamma}}^\dagger\bm{\Lambda}\bm{A}_{\bm{\gamma}} \odot (\bm{A}_{\bm{\gamma}}^\dagger\bm{Y}\bm{Y}^\mathsf{H}\bm{P}_{\bm{\gamma}}^\perp\bm{\Lambda}\bm{A}_{\bm{\gamma}})^\mathsf{T}\big]_{k,j} \nonumber \\
&\quad + \big[\bm{A}_{\bm{\gamma}}^\dagger\bm{A}_{\bm{\gamma}}^{\dagger\mathsf{H}} \odot (\bm{A}_{\bm{\gamma}}^\mathsf{H}\bm{\Lambda}^\mathsf{H}\bm{P}_{\bm{\gamma}}^\perp\bm{Y}\bm{Y}^\mathsf{H}\bm{P}_{\bm{\gamma}}^\perp\bm{\Lambda}\bm{A}_{\bm{\gamma}})^\mathsf{T}\big]_{k,j}.
\end{align}
The second summand on the right-hand side of \eqref{eq:first_inner_contribs} vanishes because it contains a factor $\bm{A}_{\bm{\gamma}}^\dagger\bm{P}_{\bm{\gamma}}^\perp = \bm{0}$. Applying the same argument to the second inner product $\langle \partial \bm{T}_k/\partial \gamma_j,\,\bm{P}_{\bm{\gamma}}^\perp\bm{Y}\bm{Y}^\mathsf{H}\rangle_\mathbb{R}$ yields
\begin{align}
\label{eq:second_inner_contribs}
\langle \tfrac{\partial \bm{T}_k}{\partial \gamma_j},\,\bm{P}_{\bm{\gamma}}^\perp\bm{Y}\bm{Y}^\mathsf{H}\rangle_\mathbb{R}
&= \big[\bm{A}_{\bm{\gamma}}^\dagger\bm{A}_{\bm{\gamma}}^{\dagger\mathsf{H}} \odot (\bm{A}_{\bm{\gamma}}^\mathsf{H}\bm{\Lambda}^\mathsf{H}\bm{P}_{\bm{\gamma}}^\perp\bm{Y}\bm{Y}^\mathsf{H}\bm{P}_{\bm{\gamma}}^\perp\bm{\Lambda}\bm{A}_{\bm{\gamma}})^\mathsf{T}\big]_{k,j},
\end{align}
where the other four trace terms vanish since each contains a factor $\bm{A}_{\bm{\gamma}}^\dagger\bm{P}_{\bm{\gamma}}^\perp = \bm{0}$.

Substituting \eqref{eq:first_inner_contribs} and \eqref{eq:second_inner_contribs} into \eqref{eq:curv_term_two_inner_prods}, taking the real part, and collecting the resulting $(j,k)$-entries over all $j,k \in [K]$ into matrix form gives
\begin{align*}
\sum_{n,\ell}\Big(r_{n,\ell}^{\mathrm{Re}}\nabla_{\bm{\gamma}}^2 r_{n,\ell}^{\mathrm{Re}} + r_{n,\ell}^{\mathrm{Im}}\nabla_{\bm{\gamma}}^2 r_{n,\ell}^{\mathrm{Im}}\Big)
&= \operatorname{Re}\!\left[\bm{A}_{\bm{\gamma}}^\dagger\bm{\Lambda}\bm{A}_{\bm{\gamma}} \odot (\bm{A}_{\bm{\gamma}}^\dagger\bm{Y}\bm{Y}^\mathsf{H}\bm{P}_{\bm{\gamma}}^\perp\bm{\Lambda}\bm{A}_{\bm{\gamma}})^\mathsf{T}\right] \\
&\quad - \operatorname{Re}\!\left[\bm{A}_{\bm{\gamma}}^\dagger\bm{Y}\bm{Y}^\mathsf{H}\bm{P}_{\bm{\gamma}}^\perp\bm{\Lambda}^2\bm{A}_{\bm{\gamma}} \odot \bm{I}_K\right] \\
&\quad + \operatorname{Re}\!\left[(\bm{A}_{\bm{\gamma}}^\dagger\bm{\Lambda}\bm{A}_{\bm{\gamma}})^\mathsf{T} \odot \bm{A}_{\bm{\gamma}}^\dagger\bm{Y}\bm{Y}^\mathsf{H}\bm{P}_{\bm{\gamma}}^\perp\bm{\Lambda}\bm{A}_{\bm{\gamma}}\right] \\
&\quad - 2\operatorname{Re}\!\left[\bm{A}_{\bm{\gamma}}^\dagger\bm{A}_{\bm{\gamma}}^{\dagger\mathsf{H}} \odot (\bm{A}_{\bm{\gamma}}^\mathsf{H}\bm{\Lambda}^\mathsf{H}\bm{P}_{\bm{\gamma}}^\perp\bm{Y}\bm{Y}^\mathsf{H}\bm{P}_{\bm{\gamma}}^\perp\bm{\Lambda}\bm{A}_{\bm{\gamma}})^\mathsf{T}\right],
\end{align*}
which is exactly \eqref{eq:residual_curv}. This concludes the proof.
\end{proof}

\subsection{Proof of Lemma~\ref{lem:jacobian_iota}}
\label{subsecA3:crossderiv}

The gradient \eqref{eq:first_deriv_gamma} with respect to $\bm{\gamma}$ is a vector-valued function of $\bm{\gamma}$ and $\bm{Z}$. The goal of this proof is to compute its cross-derivative with respect to $\bm{Z}$. We proceed by first evaluating the partial derivatives of each of the $K$ scalar components of \eqref{eq:first_deriv_gamma} with respect to $\mathrm{vec}(\bm{Z})$, and then combining them into a single matrix expression.

Since $\bm{Y} = \bm{Y}_0 + \bm{Z}$, the dependence of $\ell(\bm{\gamma};\bm{Z})$ on $\bm{Z}$ enters only through $\bm{Y}$ via a simple shift. The cross-derivative with respect to $\bm{Z}$ can therefore be computed equivalently by first differentiating $\ell(\bm{\gamma};\bm{Y}-\bm{Y}_0)$ with respect to $\bm{Y}$, then substituting $\bm{Y} = \bm{Y}_0 + \bm{Z}$ at the end. We use the shorthand
\begin{align}
\label{eq:F_gamma_def}
\bm{F}_{\bm{\gamma}} = \bm{P}_{\bm{\gamma}}^\perp\bm{\Lambda}\bm{A}_{\bm{\gamma}}
\end{align}
throughout this proof. Writing $\bm{Y} = \bm{Y}_R + j\,\bm{Y}_I$, where $\bm{Y}_R$ and $\bm{Y}_I$ denote the real and imaginary parts of $\bm{Y}$ respectively, and substituting into \eqref{eq:first_deriv_gamma} together with \eqref{eq:F_gamma_def} gives
\begin{multline}
\label{eq:gradient_expanded_real_imag}
\frac{\partial \ell(\bm{\gamma};\bm{Y}-\bm{Y}_0)}{\partial \gamma_k}
= -\frac{1}{L}\operatorname{Re}\!\Big[
\underbrace{\bm{e}_k^\mathsf{T}\bm{A}_{\bm{\gamma}}^\dagger\bm{Y}_R\bm{Y}_R^\mathsf{T}\bm{F}_{\bm{\gamma}}\bm{e}_k}_{(\mathrm{a})}
+ \underbrace{\bm{e}_k^\mathsf{T}\bm{A}_{\bm{\gamma}}^\dagger\bm{Y}_I\bm{Y}_I^\mathsf{T}\bm{F}_{\bm{\gamma}}\bm{e}_k}_{(\mathrm{b})}
+ j\underbrace{\bm{e}_k^\mathsf{T}\bm{A}_{\bm{\gamma}}^\dagger\bm{Y}_I\bm{Y}_R^\mathsf{T}\bm{F}_{\bm{\gamma}}\bm{e}_k}_{(\mathrm{c})} \\
- j\underbrace{\bm{e}_k^\mathsf{T}\bm{A}_{\bm{\gamma}}^\dagger\bm{Y}_R\bm{Y}_I^\mathsf{T}\bm{F}_{\bm{\gamma}}\bm{e}_k}_{(\mathrm{d})}
\Big].
\end{multline}
The cross-derivative with respect to $\bm{Z}$ then reduces to differentiating each of the four scalar quantities $(\mathrm{a})$, $(\mathrm{b})$, $(\mathrm{c})$, $(\mathrm{d})$ with respect to $\mathrm{vec}(\bm{Y}_R)$ and $\mathrm{vec}(\bm{Y}_I)$. The terms $(\mathrm{a})$ and $(\mathrm{b})$ are quadratic in the real or imaginary part respectively, while $(\mathrm{c})$ and $(\mathrm{d})$ are bilinear in both parts.

We first recall two matrix-calculus identities used in computing the derivatives of $(\mathrm{a})$--$(\mathrm{d})$. For $\bm{W} \in \mathbb{R}^{M \times L}$, $\bm{M} \in \mathbb{C}^{L \times K}$, and $\bm{B} \in \mathbb{C}^{M \times K}$, the partial derivative of the bilinear form $\bm{e}_k^\mathsf{T}\bm{B}^\mathsf{T}\bm{W}\bm{M}\bm{e}_k$ with respect to $\mathrm{vec}(\bm{W})$ satisfies the bilinear identity
\begin{align}
\label{eq:bilinear_identity}
\frac{\partial \bm{e}_k^\mathsf{T}\bm{B}^\mathsf{T}\bm{W}\bm{M}\bm{e}_k}{\partial \mathrm{vec}(\bm{W})} = \bm{M}\bm{e}_k \otimes \bm{B}\bm{e}_k,
\end{align}
while for $\bm{W} \in \mathbb{R}^{M \times L}$, $\bm{M} \in \mathbb{C}^{M \times K}$, and $\bm{B} \in \mathbb{C}^{M \times K}$, the partial derivative of the quadratic form $\bm{e}_k^\mathsf{T}\bm{B}^\mathsf{T}\bm{W}\bm{W}^\mathsf{T}\bm{M}\bm{e}_k$ with respect to $\mathrm{vec}(\bm{W})$ satisfies the quadratic identity
\begin{align}
\label{eq:quadratic_identity}
\frac{\partial \bm{e}_k^\mathsf{T}\bm{B}^\mathsf{T}\bm{W}\bm{W}^\mathsf{T}\bm{M}\bm{e}_k}{\partial \mathrm{vec}(\bm{W})} = \bm{W}^\mathsf{T}\bm{B}\bm{e}_k \otimes \bm{M}\bm{e}_k + \bm{W}^\mathsf{T}\bm{M}\bm{e}_k \otimes \bm{B}\bm{e}_k.
\end{align}
Both identities follow from the standard vectorization rule $\bm{b}^\mathsf{T}\bm{W}\bm{a} = (\bm{a}\otimes\bm{b})^\mathsf{T}\mathrm{vec}(\bm{W})$.

Applying the bilinear identity \eqref{eq:bilinear_identity} to the mixed terms $(\mathrm{c})$ and $(\mathrm{d})$ with the appropriate choices of $\bm{W}$, $\bm{M}$, and $\bm{B}$, we obtain
\begin{align}
\label{eq:deriv_c_d}
\frac{\partial\,(\mathrm{c})}{\partial \mathrm{vec}(\bm{Y}_R)} &= (\bm{A}_{\bm{\gamma}}^\dagger\bm{Y}_I)^\mathsf{T}\bm{e}_k \otimes \bm{F}_{\bm{\gamma}}\bm{e}_k, &
\frac{\partial\,(\mathrm{d})}{\partial \mathrm{vec}(\bm{Y}_R)} &= \bm{Y}_I^\mathsf{T}\bm{F}_{\bm{\gamma}}\bm{e}_k \otimes \bm{A}_{\bm{\gamma}}^{\dagger\mathsf{T}}\bm{e}_k, \nonumber\\
\frac{\partial\,(\mathrm{c})}{\partial \mathrm{vec}(\bm{Y}_I)} &= \bm{Y}_R^\mathsf{T}\bm{F}_{\bm{\gamma}}\bm{e}_k \otimes \bm{A}_{\bm{\gamma}}^{\dagger\mathsf{T}}\bm{e}_k, &
\frac{\partial\,(\mathrm{d})}{\partial \mathrm{vec}(\bm{Y}_I)} &= (\bm{A}_{\bm{\gamma}}^\dagger\bm{Y}_R)^\mathsf{T}\bm{e}_k \otimes \bm{F}_{\bm{\gamma}}\bm{e}_k.
\end{align}
Similarly, applying the quadratic identity \eqref{eq:quadratic_identity} to the quadratic terms $(\mathrm{a})$ and $(\mathrm{b})$,
\begin{align}
\label{eq:deriv_a_b}
\frac{\partial\,(\mathrm{a})}{\partial \mathrm{vec}(\bm{Y}_R)} &= \bm{Y}_R^\mathsf{T}\bm{A}_{\bm{\gamma}}^{\dagger\mathsf{T}}\bm{e}_k \otimes \bm{F}_{\bm{\gamma}}\bm{e}_k + \bm{Y}_R^\mathsf{T}\bm{F}_{\bm{\gamma}}\bm{e}_k \otimes \bm{A}_{\bm{\gamma}}^{\dagger\mathsf{T}}\bm{e}_k, \nonumber\\
\frac{\partial\,(\mathrm{b})}{\partial \mathrm{vec}(\bm{Y}_I)} &= \bm{Y}_I^\mathsf{T}\bm{A}_{\bm{\gamma}}^{\dagger\mathsf{T}}\bm{e}_k \otimes \bm{F}_{\bm{\gamma}}\bm{e}_k + \bm{Y}_I^\mathsf{T}\bm{F}_{\bm{\gamma}}\bm{e}_k \otimes \bm{A}_{\bm{\gamma}}^{\dagger\mathsf{T}}\bm{e}_k.
\end{align}
The remaining derivatives $\partial(\mathrm{a})/\partial \mathrm{vec}(\bm{Y}_I)$ and $\partial(\mathrm{b})/\partial \mathrm{vec}(\bm{Y}_R)$ vanish since $(\mathrm{a})$ and $(\mathrm{b})$ depend only on $\bm{Y}_R$ and $\bm{Y}_I$ respectively.

Having computed the derivatives of all four scalar components in \eqref{eq:gradient_expanded_real_imag} with respect to both $\mathrm{vec}(\bm{Y}_R)$ and $\mathrm{vec}(\bm{Y}_I)$, we now assemble them into a single Jacobian matrix. We define the real-equivalent vectorization map
\begin{align}
\label{eq:def_iota}
\iota(\bm{Z}) = \begin{bmatrix}\operatorname{Re}(\mathrm{vec}(\bm{Z})) \\ \operatorname{Im}(\mathrm{vec}(\bm{Z}))\end{bmatrix} \in \mathbb{R}^{2NL},
\end{align}
which stacks the real and imaginary parts of $\mathrm{vec}(\bm{Z})$ on top of each other. Concatenating the partial derivatives in \eqref{eq:deriv_c_d} and \eqref{eq:deriv_a_b} over $k = 1,\ldots,K$ produces column-wise Kronecker products, which by definition coincide with Khatri-Rao products. Stacking the contributions with respect to $\mathrm{vec}(\bm{Y}_R)$ and $\mathrm{vec}(\bm{Y}_I)$ as the two row-blocks dictated by \eqref{eq:def_iota} yields
\begin{multline}
\label{eq:cross_deriv_step1}
\frac{\partial }{\partial \iota(\bm{Y}-\bm{Y}_0)} 
\operatorname{Re}\!\left[
\begin{matrix} \bm{e}_1^\mathsf{T} \bm{A}_{\bm{\gamma}}^\dagger \bm{Y}\bm{Y}^\mathsf{H} \bm{F}_{\bm{\gamma}} \bm{e}_1 & \cdots & \bm{e}_K^\mathsf{T} \bm{A}_{\bm{\gamma}}^\dagger \bm{Y}\bm{Y}^\mathsf{H} \bm{F}_{\bm{\gamma}} \bm{e}_K
\end{matrix} \right] \\
= \begin{bmatrix}
\operatorname{Re}\!\left( \bm{Y}_R^\mathsf{T} \bm{A}_{\bm{\gamma}}^{\dagger\mathsf{T}} \ast \bm{F}_{\bm{\gamma}} + \bm{Y}_R^\mathsf{T} \bm{F}_{\bm{\gamma}} \ast \bm{A}_{\bm{\gamma}}^{\dagger\mathsf{T}} + j\,\bm{Y}_I^\mathsf{T} \bm{A}_{\bm{\gamma}}^{\dagger\mathsf{T}} \ast \bm{F}_{\bm{\gamma}} - j\,\bm{Y}_I^\mathsf{T} \bm{F}_{\bm{\gamma}} \ast \bm{A}_{\bm{\gamma}}^{\dagger\mathsf{T}} \right) \\[3pt]
\operatorname{Re}\!\left( \bm{Y}_I^\mathsf{T} \bm{A}_{\bm{\gamma}}^{\dagger\mathsf{T}} \ast \bm{F}_{\bm{\gamma}} + \bm{Y}_I^\mathsf{T} \bm{F}_{\bm{\gamma}} \ast \bm{A}_{\bm{\gamma}}^{\dagger\mathsf{T}} + j\,\bm{Y}_R^\mathsf{T} \bm{F}_{\bm{\gamma}} \ast \bm{A}_{\bm{\gamma}}^{\dagger\mathsf{T}} - j\,\bm{Y}_R^\mathsf{T} \bm{A}_{\bm{\gamma}}^{\dagger\mathsf{T}} \ast \bm{F}_{\bm{\gamma}} \right)
\end{bmatrix}.
\end{multline}
We now simplify the two row-blocks of \eqref{eq:cross_deriv_step1} by grouping the terms sharing the same Khatri-Rao factor. In the first row, the terms involving $\bm{A}_{\bm{\gamma}}^{\dagger\mathsf{T}} \ast \bm{F}_{\bm{\gamma}}$ combine into $(\bm{Y}_R^\mathsf{T} + j\,\bm{Y}_I^\mathsf{T})\bm{A}_{\bm{\gamma}}^{\dagger\mathsf{T}} \ast \bm{F}_{\bm{\gamma}}$, while the terms involving $\bm{F}_{\bm{\gamma}} \ast \bm{A}_{\bm{\gamma}}^{\dagger\mathsf{T}}$ combine into $(\bm{Y}_R^\mathsf{T} - j\,\bm{Y}_I^\mathsf{T})\bm{F}_{\bm{\gamma}} \ast \bm{A}_{\bm{\gamma}}^{\dagger\mathsf{T}}$. In the second row, factoring out $-j$ from the $\bm{A}_{\bm{\gamma}}^{\dagger\mathsf{T}} \ast \bm{F}_{\bm{\gamma}}$ contributions and $j$ from the $\bm{F}_{\bm{\gamma}} \ast \bm{A}_{\bm{\gamma}}^{\dagger\mathsf{T}}$ contributions yields $-j(j\,\bm{Y}_I^\mathsf{T} + \bm{Y}_R^\mathsf{T})\bm{A}_{\bm{\gamma}}^{\dagger\mathsf{T}} \ast \bm{F}_{\bm{\gamma}}$ and $j(-j\,\bm{Y}_I^\mathsf{T} + \bm{Y}_R^\mathsf{T})\bm{F}_{\bm{\gamma}} \ast \bm{A}_{\bm{\gamma}}^{\dagger\mathsf{T}}$ respectively. After these regroupings,
\begin{multline}
\label{eq:cross_deriv_step2}
\frac{\partial }{\partial \iota(\bm{Y}-\bm{Y}_0)} 
\operatorname{Re}\!\left[
\begin{matrix} \bm{e}_1^\mathsf{T} \bm{A}_{\bm{\gamma}}^\dagger \bm{Y}\bm{Y}^\mathsf{H} \bm{F}_{\bm{\gamma}} \bm{e}_1 & \cdots & \bm{e}_K^\mathsf{T} \bm{A}_{\bm{\gamma}}^\dagger \bm{Y}\bm{Y}^\mathsf{H} \bm{F}_{\bm{\gamma}} \bm{e}_K
\end{matrix} \right] \\
= \begin{bmatrix}
\operatorname{Re}\!\left( (\bm{Y}_R^\mathsf{T} + j\,\bm{Y}_I^\mathsf{T}) \bm{A}_{\bm{\gamma}}^{\dagger\mathsf{T}} \ast \bm{F}_{\bm{\gamma}} + (\bm{Y}_R^\mathsf{T} - j\,\bm{Y}_I^\mathsf{T}) \bm{F}_{\bm{\gamma}} \ast \bm{A}_{\bm{\gamma}}^{\dagger\mathsf{T}} \right) \\[3pt]
\operatorname{Re}\!\left( -j(j\,\bm{Y}_I^\mathsf{T} + \bm{Y}_R^\mathsf{T}) \bm{A}_{\bm{\gamma}}^{\dagger\mathsf{T}} \ast \bm{F}_{\bm{\gamma}} + j(-j\,\bm{Y}_I^\mathsf{T} + \bm{Y}_R^\mathsf{T}) \bm{F}_{\bm{\gamma}} \ast \bm{A}_{\bm{\gamma}}^{\dagger\mathsf{T}} \right)
\end{bmatrix}.
\end{multline}
The combinations of $\bm{Y}_R$ and $\bm{Y}_I$ appearing in \eqref{eq:cross_deriv_step2} can be rewritten in terms of $\bm{Y}$ and its Hermitian transpose using $\bm{Y}^\mathsf{T} = \bm{Y}_R^\mathsf{T} + j\,\bm{Y}_I^\mathsf{T}$ and $\bm{Y}^\mathsf{H} = \bm{Y}_R^\mathsf{T} - j\,\bm{Y}_I^\mathsf{T}$. Substituting these expressions into the first row-block of \eqref{eq:cross_deriv_step2} gives $\bm{Y}^\mathsf{T} \bm{A}_{\bm{\gamma}}^{\dagger\mathsf{T}} \ast \bm{F}_{\bm{\gamma}} + \bm{Y}^\mathsf{H} \bm{F}_{\bm{\gamma}} \ast \bm{A}_{\bm{\gamma}}^{\dagger\mathsf{T}}$. For the second row-block, $-j(j\,\bm{Y}_I^\mathsf{T} + \bm{Y}_R^\mathsf{T}) = -j\bm{Y}^\mathsf{T}$ and $j(-j\,\bm{Y}_I^\mathsf{T} + \bm{Y}_R^\mathsf{T}) = j\bm{Y}^\mathsf{H}$, so applying $\operatorname{Re}(-jw) = \operatorname{Im}(w)$ for any complex $w$ yields
\begin{multline}
\label{eq:cross_deriv_final}
\frac{\partial }{\partial \iota(\bm{Y}-\bm{Y}_0)} 
\operatorname{Re}\!\left[
\begin{matrix} \bm{e}_1^\mathsf{T} \bm{A}_{\bm{\gamma}}^\dagger \bm{Y}\bm{Y}^\mathsf{H} \bm{F}_{\bm{\gamma}} \bm{e}_1 & \cdots & \bm{e}_K^\mathsf{T} \bm{A}_{\bm{\gamma}}^\dagger \bm{Y}\bm{Y}^\mathsf{H} \bm{F}_{\bm{\gamma}} \bm{e}_K
\end{matrix} \right] \\
= \begin{bmatrix}
\operatorname{Re}\!\left( \bm{Y}^\mathsf{T} \bm{A}_{\bm{\gamma}}^{\dagger\mathsf{T}} \ast \bm{F}_{\bm{\gamma}} + \bm{Y}^\mathsf{H} \bm{F}_{\bm{\gamma}} \ast \bm{A}_{\bm{\gamma}}^{\dagger\mathsf{T}} \right) \\[3pt]
\operatorname{Im}\!\left( \bm{Y}^\mathsf{T} \bm{A}_{\bm{\gamma}}^{\dagger\mathsf{T}} \ast \bm{F}_{\bm{\gamma}} - \bm{Y}^\mathsf{H} \bm{F}_{\bm{\gamma}} \ast \bm{A}_{\bm{\gamma}}^{\dagger\mathsf{T}} \right)
\end{bmatrix}.
\end{multline}
Substituting $\bm{Y} = \bm{Y}_0 + \bm{Z}$ in \eqref{eq:cross_deriv_final} to recover the dependence on the noise gives the assertion of Lemma~\ref{lem:jacobian_iota}.

Including the $-\frac{1}{L}$ prefactor from the gradient \eqref{eq:first_deriv_gamma}, the cross-derivative in real-equivalent form reads
\begin{align}
\label{eq:cross_deriv_with_factor}
\frac{\partial \nabla_{\bm{\gamma}}\ell(\bm{\gamma};\bm{Y}-\bm{Y}_0)}{\partial \iota(\bm{Y}-\bm{Y}_0)} 
&= -\frac{1}{L}\begin{bmatrix}
\operatorname{Re}\!\left( \bm{Y}^\mathsf{T} \bm{A}_{\bm{\gamma}}^{\dagger\mathsf{T}} \ast \bm{F}_{\bm{\gamma}} + \bm{Y}^\mathsf{H} \bm{F}_{\bm{\gamma}} \ast \bm{A}_{\bm{\gamma}}^{\dagger\mathsf{T}} \right) \\[3pt]
\operatorname{Im}\!\left( \bm{Y}^\mathsf{T} \bm{A}_{\bm{\gamma}}^{\dagger\mathsf{T}} \ast \bm{F}_{\bm{\gamma}} - \bm{Y}^\mathsf{H} \bm{F}_{\bm{\gamma}} \ast \bm{A}_{\bm{\gamma}}^{\dagger\mathsf{T}} \right)
\end{bmatrix}.
\end{align}
The complex cross-derivative $\nabla_{\bm{z}}\nabla_{\bm{\gamma}}\ell(\bm{\gamma};\bm{Z})$ is obtained by combining the top and bottom blocks of \eqref{eq:cross_deriv_with_factor} into the single complex expression as
\begin{align*}
\nabla_{\bm{z}}\nabla_{\bm{\gamma}}\ell(\bm{\gamma};\bm{Z}) = -\frac{1}{2L}\big[\bm{Y}^\mathsf{T}\bm{A}_{\bm{\gamma}}^{\dagger\mathsf{T}} \ast \bm{F}_{\bm{\gamma}} + \bm{Y}^\mathsf{H}\bm{F}_{\bm{\gamma}} \ast \bm{A}_{\bm{\gamma}}^{\dagger\mathsf{T}}\big]^\mathsf{H}.
\end{align*}
Substituting $\bm{Y} = \bm{Y}_0 + \bm{Z}$ recovers the dependence on the noise and yields the assertion of Lemma~\ref{lem:jacobian_iota}. This concludes the proof.

\section{Proof of Lemma~\ref{lem:local_geometry}}
\label{secA4}

This appendix proves the local geometric properties of the VarProSD objective inside the basin $\mathcal{N}(\bm{\tau},\varrho)$ stated in Lemma~\ref{lem:local_geometry}, specifically the lower bound on $\sigma_{\min}(\nabla^2 \ell(\bm{\gamma}))$ and the upper bound on $\sigma_{\max}(\nabla^2 \ell(\bm{\gamma}))$.

\begin{proof}
Let $\delta := \mathrm{d}_\infty(\bm{\tau},\bm{\gamma})$ throughout this proof. By assumption $\bm{\gamma} \in \mathcal{N}(\bm{\tau},\varrho)$ together with \eqref{eq:basin_radius_cond}, $\delta < \tfrac{1}{2}(\Delta - \tfrac{2}{3}\rho\kappa^2)$. We use the shorthand $\bm{F}_{\bm{\gamma}} = \bm{P}_{\bm{\gamma}}^\perp\bm{\Lambda}\bm{A}_{\bm{\gamma}}$ and $\bm{S}_{\bm{\gamma}} = \bm{A}_{\bm{\gamma}}^\mathsf{H}\bm{\Lambda}^\mathsf{H}\bm{F}_{\bm{\gamma}}$ from Lemma~\ref{lem:exsv_schur_comp} throughout. 
By Lemma~\ref{lem:expression_for_hess_alg}, the Hessian of the VarProSD objective admits the closed-form expression
\begin{align}
\label{eq:hess_D_full}
\nabla^2 \ell(\bm{\gamma})
&= \operatorname{Re}\!\left[\left(\bm{A}_{\bm{\gamma}}^\dagger\widehat{\bm{R}}_{\bm{Y}}\bm{A}_{\bm{\gamma}}^{\dagger\mathsf{H}}\right)^\mathsf{T} \odot \bm{S}_{\bm{\gamma}}\right]
- \operatorname{Re}\!\left[\bm{A}_{\bm{\gamma}}^\dagger\widehat{\bm{R}}_{\bm{Y}}\bm{P}_{\bm{\gamma}}^\perp\bm{\Lambda}^2\bm{A}_{\bm{\gamma}} \odot \bm{I}_K\right] \nonumber\\
&\quad + \operatorname{Re}\!\left[\left(\bm{A}_{\bm{\gamma}}^\dagger\widehat{\bm{R}}_{\bm{Y}}\bm{F}_{\bm{\gamma}}\right)^\mathsf{T} \odot \bm{A}_{\bm{\gamma}}^\dagger\bm{\Lambda}\bm{A}_{\bm{\gamma}}\right]
+ \operatorname{Re}\!\left[\left(\bm{A}_{\bm{\gamma}}^\dagger\bm{\Lambda}\bm{A}_{\bm{\gamma}}\right)^\mathsf{T} \odot \bm{A}_{\bm{\gamma}}^\dagger\widehat{\bm{R}}_{\bm{Y}}\bm{F}_{\bm{\gamma}}\right] \nonumber\\
&\quad - \operatorname{Re}\!\left[\left(\bm{A}_{\bm{\gamma}}^\dagger\bm{A}_{\bm{\gamma}}^{\dagger\mathsf{H}}\right)^\mathsf{T} \odot \bm{F}_{\bm{\gamma}}^\mathsf{H}\widehat{\bm{R}}_{\bm{Y}}\bm{F}_{\bm{\gamma}}\right],
\end{align}
where $\widehat{\bm{R}}_{\bm{Y}} = \tfrac{1}{L}\bm{Y}\bm{Y}^\mathsf{H}$ denotes the empirical covariance of $\bm{Y}$.

We next observe that substituting $\bm{Y}\bm{Y}^\mathsf{H} = \bm{I}_N$ into \eqref{eq:hess_D_full} reduces the Hessian to zero, i.e.
\begin{align}
\label{eq:hess_zero_at_identity}
\nabla^2 \ell(\bm{\gamma})\big|_{\bm{Y}\bm{Y}^\mathsf{H} = \bm{I}_N}
&= \operatorname{Re}\!\left[\left(\bm{A}_{\bm{\gamma}}^\dagger\bm{A}_{\bm{\gamma}}^{\dagger\mathsf{H}}\right)^\mathsf{T} \odot \bm{S}_{\bm{\gamma}}\right]
- \operatorname{Re}\!\left[\bm{A}_{\bm{\gamma}}^\dagger\bm{P}_{\bm{\gamma}}^\perp\bm{\Lambda}^2\bm{A}_{\bm{\gamma}} \odot \bm{I}_K\right] \nonumber\\
&\quad + \operatorname{Re}\!\left[\left(\bm{A}_{\bm{\gamma}}^\dagger\bm{P}_{\bm{\gamma}}^\perp\bm{\Lambda}\bm{A}_{\bm{\gamma}}\right)^\mathsf{T} \odot \bm{A}_{\bm{\gamma}}^\dagger\bm{\Lambda}\bm{A}_{\bm{\gamma}}\right]
+ \operatorname{Re}\!\left[\left(\bm{A}_{\bm{\gamma}}^\dagger\bm{\Lambda}\bm{A}_{\bm{\gamma}}\right)^\mathsf{T} \odot \bm{A}_{\bm{\gamma}}^\dagger\bm{P}_{\bm{\gamma}}^\perp\bm{\Lambda}\bm{A}_{\bm{\gamma}}\right] \nonumber\\
&\quad - \operatorname{Re}\!\left[\left(\bm{A}_{\bm{\gamma}}^\dagger\bm{A}_{\bm{\gamma}}^{\dagger\mathsf{H}}\right)^\mathsf{T} \odot \bm{F}_{\bm{\gamma}}^\mathsf{H}\bm{F}_{\bm{\gamma}}\right] = \bm{0},
\end{align}
where the second, third, and fourth summands on the right-hand side of \eqref{eq:hess_zero_at_identity} vanish since they each contain a factor $\bm{A}_{\bm{\gamma}}^\dagger\bm{P}_{\bm{\gamma}}^\perp = \bm{0}$, and the first and last summands cancel using $\bm{F}_{\bm{\gamma}}^\mathsf{H}\bm{F}_{\bm{\gamma}} = \bm{S}_{\bm{\gamma}}$. Since the Hessian \eqref{eq:hess_D_full} is linear in $\bm{Y}\bm{Y}^\mathsf{H}$ and vanishes when $\bm{Y}\bm{Y}^\mathsf{H} = \bm{I}_N$, the substitution $\bm{Y}\bm{Y}^\mathsf{H} \mapsto \bm{Y}\bm{Y}^\mathsf{H} - cL\bm{I}_N$ leaves the Hessian unchanged for any $c \in \mathbb{R}$. We therefore decompose
\begin{align}
\label{eq:YYH_decomp}
\bm{Y}\bm{Y}^\mathsf{H} = \bm{Y}_0\bm{Y}_0^\mathsf{H} + L\bm{R},
\end{align}
where $\bm{Y}_0 = \bm{A}_{\bm{\tau}}\bm{X}$ is the noiseless signal and $\bm{R}$ is the residual covariance matrix defined in Theorem~\ref{thm:local_geometry}. Substituting \eqref{eq:YYH_decomp} into \eqref{eq:hess_D_full} splits the Hessian into a signal contribution involving $\bm{Y}_0\bm{Y}_0^\mathsf{H}$ and a residual covariance contribution involving $\bm{R}$. Using the Hadamard norm bound from Lemma~\ref{lem:zhan_hadamard} together with the triangle inequality, the minimum singular value of the Hessian is lower bounded as
\begin{align}
\label{eq:sigma_min_decomp}
\sigma_{\min}(\nabla^2 \ell(\bm{\gamma}))
&\geq \sigma_{\min}\Big(\underbrace{\operatorname{Re}\!\left[\left(\bm{A}_{\bm{\gamma}}^\dagger\bm{Y}_0\bm{Y}_0^\mathsf{H}\bm{A}_{\bm{\gamma}}^{\dagger\mathsf{H}}\right)^\mathsf{T} \odot \bm{S}_{\bm{\gamma}}\right]}_{(\mathrm{a})}\Big) - \underbrace{\norm{\bm{A}_{\bm{\gamma}}^\dagger\bm{Y}_0}}_{(\mathrm{b})}\cdot \underbrace{\norm{\bm{Y}_0^\mathsf{H}\bm{P}_{\bm{\gamma}}^\perp}}_{(\mathrm{c})}\cdot \underbrace{\norm{\bm{\Lambda}^2\bm{A}_{\bm{\gamma}}}}_{(\mathrm{d})} \nonumber\\
&\quad - 2\,\norm{\bm{A}_{\bm{\gamma}}^\dagger\bm{Y}_0}\cdot \underbrace{\norm{\bm{Y}_0^\mathsf{H}\bm{F}_{\bm{\gamma}}}}_{(\mathrm{e})}\cdot \underbrace{\norm{\bm{A}_{\bm{\gamma}}^\dagger}}_{(\mathrm{f})}\cdot \underbrace{\norm{\bm{\Lambda}\bm{A}_{\bm{\gamma}}}}_{(\mathrm{g})} - \norm{\bm{Y}_0^\mathsf{H}\bm{F}_{\bm{\gamma}}}^2\cdot \norm{\bm{A}_{\bm{\gamma}}^\dagger}^2 \nonumber\\
&\quad - \norm{\bm{R}}\cdot\bigg(\norm{\bm{A}_{\bm{\gamma}}^\dagger}\cdot \underbrace{\norm{\bm{P}_{\bm{\gamma}}^\perp}}_{(\mathrm{h})}\cdot \norm{\bm{\Lambda}^2\bm{A}_{\bm{\gamma}}} + 2\norm{\bm{A}_{\bm{\gamma}}^\dagger}^2\cdot \underbrace{\norm{\bm{F}_{\bm{\gamma}}}}_{(\mathrm{i})}\cdot \norm{\bm{\Lambda}\bm{A}_{\bm{\gamma}}} \nonumber\\
&\qquad\qquad\quad + \norm{\bm{F}_{\bm{\gamma}}}^2\cdot \norm{\bm{A}_{\bm{\gamma}}^\dagger}^2 + \norm{\bm{A}_{\bm{\gamma}}^\dagger}^2\cdot \underbrace{\norm{\bm{S}_{\bm{\gamma}}}}_{(\mathrm{j})}\bigg).
\end{align}
We now bound each of the terms $(\mathrm{a})$--$(\mathrm{j})$ in turn.

We begin with term $(\mathrm{f})$. Applying \cite[Theorem~1]{ferreira2023conditionNumber} (paraphrased in Appendix~\ref{secA1} as Lemma~\ref{lem:ferreira_thm1}),
\begin{align}
\label{eq:bound_f}
\norm{\bm{A}_{\bm{\gamma}}^\dagger} = \norm{(\bm{G}\bm{\Phi}_{\bm{\gamma}})^\dagger} \leq \frac{1}{\sqrt{T E_g}\cdot \eta_-},
\end{align}
where $\eta_+,\eta_-$ are defined in \eqref{eq:def_eta}. By the basin condition \eqref{eq:basin_radius_cond} together with $\delta < \tfrac{1}{2}(\Delta - \tfrac{2}{3}\rho\kappa^2)$, these constants satisfy
\begin{align}
\label{eq:eta_bounds_D}
\eta_+ \leq \sqrt{\tfrac{4}{3}}, \qquad \eta_- \geq \sqrt{\tfrac{2}{3}}.
\end{align}
Next, we bound term $(\mathrm{b})$ by writing $\bm{A}_{\bm{\gamma}}^\dagger\bm{Y}_0 = \bm{A}_{\bm{\gamma}}^\dagger\bm{A}_{\bm{\tau}}\bm{X} = \big(\bm{I}_K + \bm{A}_{\bm{\gamma}}^\dagger(\bm{A}_{\bm{\tau}} - \bm{A}_{\bm{\gamma}})\big)\bm{X}$, where we used $\bm{A}_{\bm{\gamma}}^\dagger\bm{A}_{\bm{\gamma}} = \bm{I}_K$ since $\bm{A}_{\bm{\gamma}}$ is full column rank. The submultiplicative property of the spectral norm together with Lemma~\ref{cor:bounding_sing_val_phi_gamm_phi_tau} and Lemma~\ref{lem:ferreira_thm1} then yields
\begin{align}
\label{eq:bound_b}
\norm{\bm{A}_{\bm{\gamma}}^\dagger\bm{Y}_0} \leq \bigg(1 + \delta\cdot\sqrt{\tfrac{E_{g'}}{E_g}}\cdot \tfrac{\eta_+}{\eta_-}\bigg)\cdot\norm{\bm{X}}.
\end{align}
Applying Lemma~\ref{cor:bounding_sing_val_phi_gamm_phi_tau} to bound term $(\mathrm{c})$,
\begin{align}
\label{eq:bound_c}
\norm{\bm{Y}_0^\mathsf{H}\bm{P}_{\bm{\gamma}}^\perp} \leq \norm{\bm{P}_{\bm{\gamma}}^\perp\bm{A}_{\bm{\tau}}}\cdot\norm{\bm{X}} \leq \delta\cdot\sqrt{T E_{g'}}\cdot \eta_+\cdot\norm{\bm{X}}.
\end{align}
Terms $(\mathrm{d})$ and $(\mathrm{g})$ are bounded using Lemma~\ref{lem:cor_of_ferreira_thm1},
\begin{align}
\label{eq:bound_dg}
\norm{\bm{\Lambda}\bm{A}_{\bm{\gamma}}} \leq \sqrt{T E_{g'}}\cdot\eta_+, \qquad \norm{\bm{\Lambda}^2\bm{A}_{\bm{\gamma}}} \leq \sqrt{T E_{g''}}\cdot\eta_+.
\end{align}
Term $(\mathrm{e})$ is bounded by writing $\bm{Y}_0^\mathsf{H}\bm{F}_{\bm{\gamma}} = \bm{X}^\mathsf{H}\bm{A}_{\bm{\tau}}^\mathsf{H}\bm{P}_{\bm{\gamma}}^\perp\bm{\Lambda}\bm{A}_{\bm{\gamma}}$ and applying submultiplicativity together with Lemma~\ref{cor:bounding_sing_val_phi_gamm_phi_tau} and \eqref{eq:bound_dg},
\begin{align}
\label{eq:bound_e}
\norm{\bm{Y}_0^\mathsf{H}\bm{F}_{\bm{\gamma}}} \leq \norm{\bm{\Lambda}\bm{A}_{\bm{\gamma}}}\cdot\norm{\bm{P}_{\bm{\gamma}}^\perp\bm{A}_{\bm{\tau}}}\cdot\norm{\bm{X}} \leq T E_{g'}\cdot\eta_+^2\cdot\delta\cdot\norm{\bm{X}}.
\end{align}
Term $(\mathrm{h})$ satisfies $\norm{\bm{P}_{\bm{\gamma}}^\perp} \leq 1$ since $\bm{P}_{\bm{\gamma}}^\perp$ is an orthogonal projection. Term $(\mathrm{i})$ is bounded using $\norm{\bm{P}_{\bm{\gamma}}^\perp} \leq 1$ together with \eqref{eq:bound_dg},
\begin{align}
\label{eq:bound_i}
\norm{\bm{F}_{\bm{\gamma}}} \leq \norm{\bm{P}_{\bm{\gamma}}^\perp}\cdot\norm{\bm{\Lambda}\bm{A}_{\bm{\gamma}}} \leq \sqrt{T E_{g'}}\cdot\eta_+.
\end{align}
Term $(\mathrm{j})$ is bounded directly using Lemma~\ref{lem:exsv_schur_comp},
\begin{align}
\label{eq:bound_j}
\norm{\bm{S}_{\bm{\gamma}}} \leq T E_{g'}\cdot\eta_+^2.
\end{align}

It remains to lower bound term $(\mathrm{a})$. Decomposing $\bm{S}_{\bm{\gamma}} = T E_{g'}\bm{I}_K + (\bm{S}_{\bm{\gamma}} - T E_{g'}\bm{I}_K)$,
\begin{multline*}
\operatorname{Re}\!\left[\left(\bm{A}_{\bm{\gamma}}^\dagger\bm{Y}_0\bm{Y}_0^\mathsf{H}\bm{A}_{\bm{\gamma}}^{\dagger\mathsf{H}}\right)^\mathsf{T} \odot \bm{S}_{\bm{\gamma}}\right]
= T E_{g'}\cdot\left(\bm{A}_{\bm{\gamma}}^\dagger\bm{Y}_0\bm{Y}_0^\mathsf{H}\bm{A}_{\bm{\gamma}}^{\dagger\mathsf{H}}\right)^\mathsf{T} \odot \bm{I}_K \\
 + \operatorname{Re}\!\left[\left(\bm{A}_{\bm{\gamma}}^\dagger\bm{Y}_0\bm{Y}_0^\mathsf{H}\bm{A}_{\bm{\gamma}}^{\dagger\mathsf{H}}\right)^\mathsf{T} \odot \left(\bm{S}_{\bm{\gamma}} - T E_{g'}\bm{I}_K\right)\right],
\end{multline*}
and applying the triangle inequality together with Lemma~\ref{lem:zhan_hadamard},
\begin{align}
\label{eq:sigma_min_a}
\sigma_{\min}\!\left(\operatorname{Re}\!\left[\left(\bm{A}_{\bm{\gamma}}^\dagger\bm{Y}_0\bm{Y}_0^\mathsf{H}\bm{A}_{\bm{\gamma}}^{\dagger\mathsf{H}}\right)^\mathsf{T} \odot \bm{S}_{\bm{\gamma}}\right]\right)
&\geq T E_{g'}\cdot \underbrace{\min_{j \in [K]}\norm{\bm{e}_j^\mathsf{T}\bm{A}_{\bm{\gamma}}^\dagger\bm{Y}_0}^2}_{(\mathrm{k})} - \underbrace{\norm{\bm{S}_{\bm{\gamma}} - T E_{g'}\bm{I}_K}}_{(\mathrm{l})}\cdot\norm{\bm{A}_{\bm{\gamma}}^\dagger\bm{Y}_0}^2.
\end{align}
Term $(\mathrm{l})$ is bounded using Lemma~\ref{lem:exsv_schur_comp},
\begin{align}
\label{eq:bound_l}
\norm{\bm{S}_{\bm{\gamma}} - T E_{g'}\bm{I}_K} \leq \tfrac{2}{3}T E_{g'}\rho(\Delta-2\delta)^{-1}.
\end{align}
Term $(\mathrm{k})$ is lower bounded by writing $\bm{e}_j^\mathsf{T}\bm{A}_{\bm{\gamma}}^\dagger\bm{Y}_0 = \bm{e}_j^\mathsf{T}\big(\bm{I}_K + \bm{A}_{\bm{\gamma}}^\dagger(\bm{A}_{\bm{\tau}} - \bm{A}_{\bm{\gamma}})\big)\bm{X}$ and applying the reverse triangle inequality,
\begin{align}
\min_{j \in [K]}\norm{\bm{e}_j^\mathsf{T}\bm{A}_{\bm{\gamma}}^\dagger\bm{Y}_0}_2
&\geq \min_{j \in [K]}\norm{\bm{e}_j^\mathsf{T}\bm{X}}_2 - \norm{\bm{A}_{\bm{\gamma}}^\dagger(\bm{A}_{\bm{\tau}} - \bm{A}_{\bm{\gamma}})\bm{X}} \nonumber\\
\label{eq:bound_k}
&\geq r_{\min}(\bm{X}) - \delta\cdot\sqrt{\tfrac{E_{g'}}{E_g}}\cdot \tfrac{\eta_+}{\eta_-}\cdot\norm{\bm{X}},
\end{align}
where the second inequality follows from $\min_{j}\norm{\bm{e}_j^\mathsf{T}\bm{X}}_2 = r_{\min}(\bm{X})$ and the bound on the second term in \eqref{eq:bound_k} (derived using the same argument as in \eqref{eq:bound_b}).

Substituting the bounds \eqref{eq:bound_f}--\eqref{eq:bound_j} along with \eqref{eq:sigma_min_a}--\eqref{eq:bound_k} into \eqref{eq:sigma_min_decomp} yields
\begin{align*}
\sigma_{\min}(\nabla^2\ell(\bm{\gamma})) 
&\geq T E_{g'}\cdot\bigg(r_{\min}(\bm{X}) - \delta\cdot\sqrt{\tfrac{E_{g'}}{E_g}}\cdot\tfrac{\eta_+}{\eta_-}\cdot\norm{\bm{X}}\bigg)^2 \nonumber\\
&\quad - \tfrac{2}{3}T E_{g'}\rho(\Delta-2\delta)^{-1}\cdot\bigg(1 + \delta\cdot\sqrt{\tfrac{E_{g'}}{E_g}}\cdot\tfrac{\eta_+}{\eta_-}\bigg)^2\cdot\norm{\bm{X}}^2 \nonumber\\
&\quad - \bigg(1 + \delta\cdot\sqrt{\tfrac{E_{g'}}{E_g}}\cdot\tfrac{\eta_+}{\eta_-}\bigg)\cdot T\cdot\norm{\bm{X}}^2\cdot\delta\cdot\sqrt{E_{g'}\cdot E_{g''}}\cdot\eta_+^2 \nonumber\\
&\quad - 2\bigg(1 + \delta\cdot\sqrt{\tfrac{E_{g'}}{E_g}}\cdot\tfrac{\eta_+}{\eta_-}\bigg)\cdot T\cdot\norm{\bm{X}}^2\cdot\delta\cdot\sqrt{\tfrac{E_{g'}}{E_g}}\cdot\tfrac{\eta_+}{\eta_-}\cdot E_{g'}\cdot\eta_+^2 \nonumber\\
&\quad - \delta^2\cdot\tfrac{E_{g'}}{E_g}\cdot\tfrac{\eta_+^2}{\eta_-^2}\cdot E_{g'}\cdot\eta_+^2\cdot T\cdot\norm{\bm{X}}^2 \nonumber\\
&\quad - \norm{\bm{R}}\cdot\bigg(\sqrt{\tfrac{E_{g''}}{E_g}}\cdot\tfrac{\eta_+}{\eta_-} + 4\cdot\tfrac{E_{g'}}{E_g}\cdot\tfrac{\eta_+^2}{\eta_-^2}\bigg).
\end{align*}
Applying the bounds in \eqref{eq:eta_bounds_D} together with the basin condition \eqref{eq:basin_radius_cond} on $\delta$ and the residual covariance condition \eqref{eq:cond_on_R1} on $\norm{\bm{R}}$ yields the lower bound
\begin{align*}
\sigma_{\min}(\nabla^2\ell(\bm{\gamma})) \geq \tfrac{1}{3}T E_{g'}\cdot r_{\min}^2(\bm{X}),
\end{align*}
which establishes \eqref{eq:min_sing_value_lemma}.

The upper bound on $\sigma_{\max}(\nabla^2\ell(\bm{\gamma}))$ is obtained by an analogous argument, with the triangle inequality applied in the reverse direction so that all signs in \eqref{eq:sigma_min_decomp} become positive. Specifically, using the Hadamard norm bound from Lemma~\ref{lem:zhan_hadamard} together with the triangle inequality, the maximum singular value of the Hessian is upper bounded as
\begin{align}
\label{eq:sigma_max_decomp}
\sigma_{\max}(\nabla^2 \ell(\bm{\gamma}))
&\leq \sigma_{\max}\Big(\underbrace{\operatorname{Re}\!\left[\left(\bm{A}_{\bm{\gamma}}^\dagger\bm{Y}_0\bm{Y}_0^\mathsf{H}\bm{A}_{\bm{\gamma}}^{\dagger\mathsf{H}}\right)^\mathsf{T} \odot \bm{S}_{\bm{\gamma}}\right]}_{(\mathrm{a})}\Big) + \norm{\bm{A}_{\bm{\gamma}}^\dagger\bm{Y}_0}\cdot \norm{\bm{Y}_0^\mathsf{H}\bm{P}_{\bm{\gamma}}^\perp}\cdot \norm{\bm{\Lambda}^2\bm{A}_{\bm{\gamma}}} \nonumber\\
&\quad + 2\,\norm{\bm{A}_{\bm{\gamma}}^\dagger\bm{Y}_0}\cdot \norm{\bm{Y}_0^\mathsf{H}\bm{F}_{\bm{\gamma}}}\cdot \norm{\bm{A}_{\bm{\gamma}}^\dagger}\cdot \norm{\bm{\Lambda}\bm{A}_{\bm{\gamma}}} + \norm{\bm{Y}_0^\mathsf{H}\bm{F}_{\bm{\gamma}}}^2\cdot \norm{\bm{A}_{\bm{\gamma}}^\dagger}^2 \nonumber\\
&\quad + \norm{\bm{R}}\cdot\bigg(\norm{\bm{A}_{\bm{\gamma}}^\dagger}\cdot \norm{\bm{P}_{\bm{\gamma}}^\perp}\cdot \norm{\bm{\Lambda}^2\bm{A}_{\bm{\gamma}}} + 2\norm{\bm{A}_{\bm{\gamma}}^\dagger}^2\cdot \norm{\bm{F}_{\bm{\gamma}}}\cdot \norm{\bm{\Lambda}\bm{A}_{\bm{\gamma}}} \nonumber\\
&\qquad\qquad\quad + \norm{\bm{F}_{\bm{\gamma}}}^2\cdot \norm{\bm{A}_{\bm{\gamma}}^\dagger}^2 + \norm{\bm{A}_{\bm{\gamma}}^\dagger}^2\cdot \norm{\bm{S}_{\bm{\gamma}}}\bigg).
\end{align}
Most terms in \eqref{eq:sigma_max_decomp} are bounded exactly as in \eqref{eq:bound_f}--\eqref{eq:bound_j}. The only term whose treatment differs from the lower bound argument is the signal contribution $(\mathrm{a})$, which is upper bounded by decomposing $\bm{S}_{\bm{\gamma}} = T E_{g'}\bm{I}_K + (\bm{S}_{\bm{\gamma}} - T E_{g'}\bm{I}_K)$ and applying Lemma~\ref{lem:zhan_hadamard} together with the triangle inequality to obtain
\begin{align*}
\sigma_{\max}\!\left(\operatorname{Re}\!\left[\left(\bm{A}_{\bm{\gamma}}^\dagger\bm{Y}_0\bm{Y}_0^\mathsf{H}\bm{A}_{\bm{\gamma}}^{\dagger\mathsf{H}}\right)^\mathsf{T} \odot \bm{S}_{\bm{\gamma}}\right]\right) 
\leq T E_{g'}\cdot \underbrace{\max_{j \in [K]}\norm{\bm{e}_j^\mathsf{T}\bm{A}_{\bm{\gamma}}^\dagger\bm{Y}_0}^2}_{(\mathrm{m})} + \norm{\bm{S}_{\bm{\gamma}} - T E_{g'}\bm{I}_K}\cdot\norm{\bm{A}_{\bm{\gamma}}^\dagger\bm{Y}_0}^2,
\end{align*}
where term $(\mathrm{m})$ is upper bounded by writing $\bm{e}_j^\mathsf{T}\bm{A}_{\bm{\gamma}}^\dagger\bm{Y}_0 = \bm{e}_j^\mathsf{T}\big(\bm{I}_K + \bm{A}_{\bm{\gamma}}^\dagger(\bm{A}_{\bm{\tau}} - \bm{A}_{\bm{\gamma}})\big)\bm{X}$ and applying the triangle inequality together with Lemma~\ref{cor:bounding_sing_val_phi_gamm_phi_tau} and Lemma~\ref{lem:ferreira_thm1},
\begin{align*}
\max_{j \in [K]}\norm{\bm{e}_j^\mathsf{T}\bm{A}_{\bm{\gamma}}^\dagger\bm{Y}_0}_2 \leq r_{\max}(\bm{X}) + \delta\cdot\sqrt{\tfrac{E_{g'}}{E_g}}\cdot\tfrac{\eta_+}{\eta_-}\cdot\norm{\bm{X}},
\end{align*}
where $r_{\max}(\bm{X}) = \max_{j \in [K]}\norm{\bm{e}_j^\mathsf{T}\bm{X}}_2$. Substituting these bounds into \eqref{eq:sigma_max_decomp} and applying \eqref{eq:eta_bounds_D}, the basin condition \eqref{eq:basin_radius_cond} on $\delta$, and the residual covariance condition \eqref{eq:cond_on_R1} on $\norm{\bm{R}}$ yields the upper bound
\begin{align*}
\sigma_{\max}(\nabla^2\ell(\bm{\gamma})) \leq T E_{g'}\cdot r_{\max}^2(\bm{X}),
\end{align*}
which establishes \eqref{eq:max_sing_value_lemma}. This concludes the proof.
\end{proof}

\section{Proof of Lemma~\ref{lem:spectral_norm_grad}}
\label{secA5}

This appendix establishes the spectral norm bound on the cross-derivative $\nabla_{\bm{z}}\nabla_{\bm{\gamma}}\ell(\bm{\gamma}_\star;\bm{Z})$ stated in Lemma~\ref{lem:spectral_norm_grad}, which underpins the adversarial noise analysis. The proof leverages the closed-form expression of the cross-derivative from Lemma~\ref{lem:jacobian_iota} together with the conditioning bounds collected in Appendix~\ref{secA4} for the proof of Lemma~\ref{lem:local_geometry}.

\begin{proof}
Let $\delta := \mathrm{d}_\infty(\bm{\tau},\bm{\gamma}_\star)$ throughout this proof. By assumption $\bm{\gamma}_\star \in \mathcal{N}(\bm{\tau},\varrho)$ together with \eqref{eq:basin_radius_cond}, $\delta < \tfrac{1}{2}(\Delta - \tfrac{2}{3}\rho\kappa^2)$, and the constants $\eta_+,\eta_-$ defined in \eqref{eq:def_eta} satisfy the bounds \eqref{eq:eta_bounds_D}. We use the shorthand $\bm{F}_{\bm{\gamma}_\star} = \bm{P}_{\bm{\gamma}_\star}^\perp\bm{\Lambda}\bm{A}_{\bm{\gamma}_\star}$ throughout. By Lemma~\ref{lem:jacobian_iota}, the cross-derivative admits the closed-form expression
\begin{align}
\label{eq:cross_deriv_E}
\nabla_{\bm{z}}\nabla_{\bm{\gamma}}\ell(\bm{\gamma}_\star;\bm{Z}) = -\frac{1}{2L}\big[\bm{Y}^\mathsf{T}\bm{A}_{\bm{\gamma}_\star}^{\dagger\mathsf{T}} \ast \bm{F}_{\bm{\gamma}_\star} + \bm{Y}^\mathsf{H}\bm{F}_{\bm{\gamma}_\star} \ast \bm{A}_{\bm{\gamma}_\star}^{\dagger\mathsf{T}}\big]^\mathsf{H}.
\end{align}
Applying the spectral norm bound on Khatri-Rao products from Lemma~\ref{lem:khatri_spectral_norm} together with the triangle inequality, the spectral norm of \eqref{eq:cross_deriv_E} is upper bounded as
\begin{align}
\label{eq:spectral_norm_split}
\norm{\nabla_{\bm{z}}\nabla_{\bm{\gamma}}\ell(\bm{\gamma}_\star;\bm{Z})}
\leq \frac{1}{2L}\bigg(\underbrace{\norm{\bm{A}_{\bm{\gamma}_\star}^\dagger\bm{Y}}}_{(\mathrm{a})}\cdot \norm{\bm{F}_{\bm{\gamma}_\star}} + \underbrace{\norm{\bm{Y}^\mathsf{H}\bm{F}_{\bm{\gamma}_\star}}}_{(\mathrm{b})}\cdot \norm{\bm{A}_{\bm{\gamma}_\star}^\dagger}\bigg).
\end{align}
The bounds on $\norm{\bm{A}_{\bm{\gamma}_\star}^\dagger}$ and $\norm{\bm{F}_{\bm{\gamma}_\star}}$ are obtained from \eqref{eq:bound_f} and \eqref{eq:bound_i} respectively, and it remains to bound the terms $(\mathrm{a})$ and $(\mathrm{b})$.

We bound term $(\mathrm{a})$ by writing $\bm{Y} = \bm{Y}_0 + \bm{Z}$ and applying the triangle inequality,
\begin{align}
\norm{\bm{A}_{\bm{\gamma}_\star}^\dagger\bm{Y}}
&\leq \norm{\bm{A}_{\bm{\gamma}_\star}^\dagger\bm{Y}_0} + \norm{\bm{A}_{\bm{\gamma}_\star}^\dagger}\cdot\norm{\bm{Z}} \nonumber\\
\label{eq:bound_a_E}
&\leq \bigg(1 + \delta\cdot\sqrt{\tfrac{E_{g'}}{E_g}}\cdot\tfrac{\eta_+}{\eta_-}\bigg)\cdot\norm{\bm{X}} + \tfrac{1}{\sqrt{T E_g}\cdot\eta_-}\cdot\norm{\bm{Z}},
\end{align}
where the bound on $\norm{\bm{A}_{\bm{\gamma}_\star}^\dagger\bm{Y}_0}$ follows from \eqref{eq:bound_b} and the bound on $\norm{\bm{A}_{\bm{\gamma}_\star}^\dagger}$ from \eqref{eq:bound_f}. Similarly, we bound term $(\mathrm{b})$ as
\begin{align}
\norm{\bm{Y}^\mathsf{H}\bm{F}_{\bm{\gamma}_\star}}
&\leq \norm{\bm{Y}_0^\mathsf{H}\bm{F}_{\bm{\gamma}_\star}} + \norm{\bm{Z}}\cdot\norm{\bm{F}_{\bm{\gamma}_\star}} \nonumber\\
\label{eq:bound_b_E}
&\leq T E_{g'}\cdot\eta_+^2\cdot\delta\cdot\norm{\bm{X}} + \sqrt{T E_{g'}}\cdot\eta_+\cdot\norm{\bm{Z}},
\end{align}
where the bound on $\norm{\bm{Y}_0^\mathsf{H}\bm{F}_{\bm{\gamma}_\star}}$ follows from \eqref{eq:bound_e} and the bound on $\norm{\bm{F}_{\bm{\gamma}_\star}}$ from \eqref{eq:bound_i}.

Substituting \eqref{eq:bound_a_E} and \eqref{eq:bound_b_E} together with the bounds \eqref{eq:bound_f} and \eqref{eq:bound_i} into \eqref{eq:spectral_norm_split} yields
\begin{multline*}
\norm{\nabla_{\bm{z}}\nabla_{\bm{\gamma}}\ell(\bm{\gamma}_\star;\bm{Z})}
\leq \frac{1}{2L}\bigg(\sqrt{T E_{g'}}\cdot\eta_+\cdot\norm{\bm{X}} + 2\delta\cdot\sqrt{T E_{g'}}\cdot\eta_+\cdot\sqrt{\tfrac{E_{g'}}{E_g}}\cdot\tfrac{\eta_+}{\eta_-}\cdot\norm{\bm{X}} \\
\qquad\qquad + 2\sqrt{\tfrac{E_{g'}}{E_g}}\cdot\tfrac{\eta_+}{\eta_-}\cdot\norm{\bm{Z}}\bigg).
\end{multline*}
Applying the bounds in \eqref{eq:eta_bounds_D} to absorb the $\eta_+, \eta_-$ factors into a constant $c$ yields
\begin{align*}
\norm{\nabla_{\bm{z}}\nabla_{\bm{\gamma}}\ell(\bm{\gamma}_\star;\bm{Z})}
&\leq c\,\bigg(\sqrt{T E_{g'}}\cdot\norm{\bm{X}} + \delta\cdot\sqrt{T E_{g'}}\cdot\sqrt{\tfrac{E_{g'}}{E_g}}\cdot\norm{\bm{X}} + \sqrt{\tfrac{E_{g'}}{E_g}}\cdot\norm{\bm{Z}}\bigg),
\end{align*}
which establishes \eqref{eq:lemma_spectral_norm_jac_z} and concludes the proof.
\end{proof}


\clearpage

\renewcommand*{\bibfont}{\footnotesize}
\printbibliography[heading=bibintoc]

\end{document}